\documentclass[lettersize,journal]{IEEEtran}

\usepackage{cite}
\usepackage{amsmath,amssymb,amsfonts,bm,amsthm}
\usepackage{algorithmic}
\usepackage{tikz}
\usepackage{tabularx}
\usepackage{graphicx}
\usepackage{dblfloatfix}
\usepackage{subcaption}
\usepackage{textcomp}
\usepackage{xcolor}
\usepackage{bbold}
\usepackage{mwe}
\newcommand{\E}{\mathrm{E}}

\newcommand{\eqc}{\overset{c}{=}}

\DeclareMathOperator{\Tr}{Tr}
\usepackage[ruled,vlined]{algorithm2e}
\usepackage{array}
\usepackage{booktabs}
\usepackage{multirow}

\newtheorem{assumption}{Assumption}
\newtheorem{definition}{Definition}
 
\newtheorem{theorem}{Theorem}[section]
\newtheorem{corollary}{Corollary}[theorem]
\newtheorem{lemma}[theorem]{Lemma}
\newtheorem{proposition}[theorem]{Proposition}
\theoremstyle{remark}
\newtheorem{remark}{Remark}[theorem]
\newtheorem*{remark*}{Remark}

\def\BibTeX{{\rm B\kern-.05em{\sc i\kern-.025em b}\kern-.08em
    T\kern-.1667em\lower.7ex\hbox{E}\kern-.125emX}}

\newcommand{\overbar}[1]{\mkern 1.5mu\overline{\mkern-1.5mu#1\mkern-1.5mu}\mkern 1.5mu}

\begin{document}

\title{PiVoT: A Variational Solution for Real-time Large-scale Multi-object Detection and Tracking under Heavy Clutter}
\author{Runze~Gan,~\IEEEmembership{Member,~IEEE}, Qing~Li,~\IEEEmembership{Member,~IEEE}, 
Simon~J.~Godsill,~\IEEEmembership{Fellow,~IEEE}, 
Mike~E.~Davies,~\IEEEmembership{Fellow,~IEEE},
and James~R.~Hopgood,~\IEEEmembership{Senior Member,~IEEE}
\thanks{
This work was supported through the SIGNeTS project (W911NF-20-2-0225) and the Engineering and Physical Sciences Research Council (EPSRC) under grant EP/X025365/1, in collaboration with Leonardo, as part of the ‘Smart Products Made Smarter’ project.
}
\thanks{
R. Gan is with the Institute for Imaging, Data and Communications (IDCOM), University of Edinburgh, Edinburgh EH9 3FG, U.K., and also with the Department of Engineering, University of Cambridge, Cambridge CB2 1PZ, U.K. (e-mail: rgan@ed.ac.uk; rg605@cam.ac.uk).
Q. Li is with the School of Mathematics, University of Edinburgh, Edinburgh EH9 3FD, U.K., and also with the Department of Engineering, University of Cambridge, Cambridge CB2 1PZ, U.K. (e-mail: qli10@ed.ac.uk; ql289@cam.ac.uk).
S. J. Godsill is with the Department of Engineering, University of Cambridge, Cambridge CB2 1PZ, U.K. (e-mail: sjg30@cam.ac.uk).
M. E. Davies, and J. R. Hopgood are with IDCOM, University of Edinburgh, Edinburgh EH9 3FG, U.K. (e-mail: \{mike.davies, james.hopgood\}@ed.ac.uk).
}%
}

\maketitle
\begingroup\renewcommand\thefootnote{\textsection}

\begin{abstract}
Multi-object detection and tracking from noisy point clouds remain challenging in many data-scarce radar applications. Current Bayesian trackers based on Poisson measurement models offer a training-free solution but struggle to achieve accuracy and efficiency under severe clutter, large object populations, and full-resolution Doppler point clouds. We address this with PiVoT, a fast, clutter-resilient multi-object tracker
for both positional and Doppler measurements.
PiVoT performs end-to-end detection and tracking of a large and time-varying number of objects without external clustering or detectors, through joint inference of object states, shapes, existence probabilities, data association, and measurement rates. Its efficiency is driven by several variational inference innovations, such as theoretically justified birth pruning, quadratic-to-linear complexity reductions for exact updates, and a computationally efficient Doppler Poisson model. 
Experiments show that PiVoT substantially outperforms existing Bayesian trackers in challenging scenes,
while also demonstrating exceptional scalability to a thousand objects, robustness to clutter visually inseparable from objects, and real-time operation on full-scale modern automotive radar datasets,
where it attains performance comparable to a deep-learning detection benchmark as a training-free joint detector and tracker.

\end{abstract}


\section{Introduction}

\label{sec:intro}
\IEEEPARstart{D}{etecting} and tracking multiple objects directly from noisy point-cloud-like measurements \cite{wang2024sequential,granstrom2022tutorial,scheiner2021object,cavagna2019sparta} is increasingly relevant for modern sensing systems. A prominent example is high-resolution radar that provides many spatial returns with Doppler velocity over long detection ranges (see Fig. \ref{fig: FrontDemo} for an illustration), even under poor weather conditions \cite{kari2023evolutionary}. While recent advances in point-cloud deep learning have enabled effective detection and subsequent tracking in data-rich automotive settings \cite{wang2024sequential,scheiner2021object}, labelled data scarcity remains a fundamental challenge in many radar-dependent applications \cite{ahmad2023review,scheiner2021new}, and is especially restrictive in surveillance radar, where controlled data collection and annotation require costly and time-consuming field trials.

\begin{figure*}[t!] 
\centerline{\includegraphics[width=18.2cm]{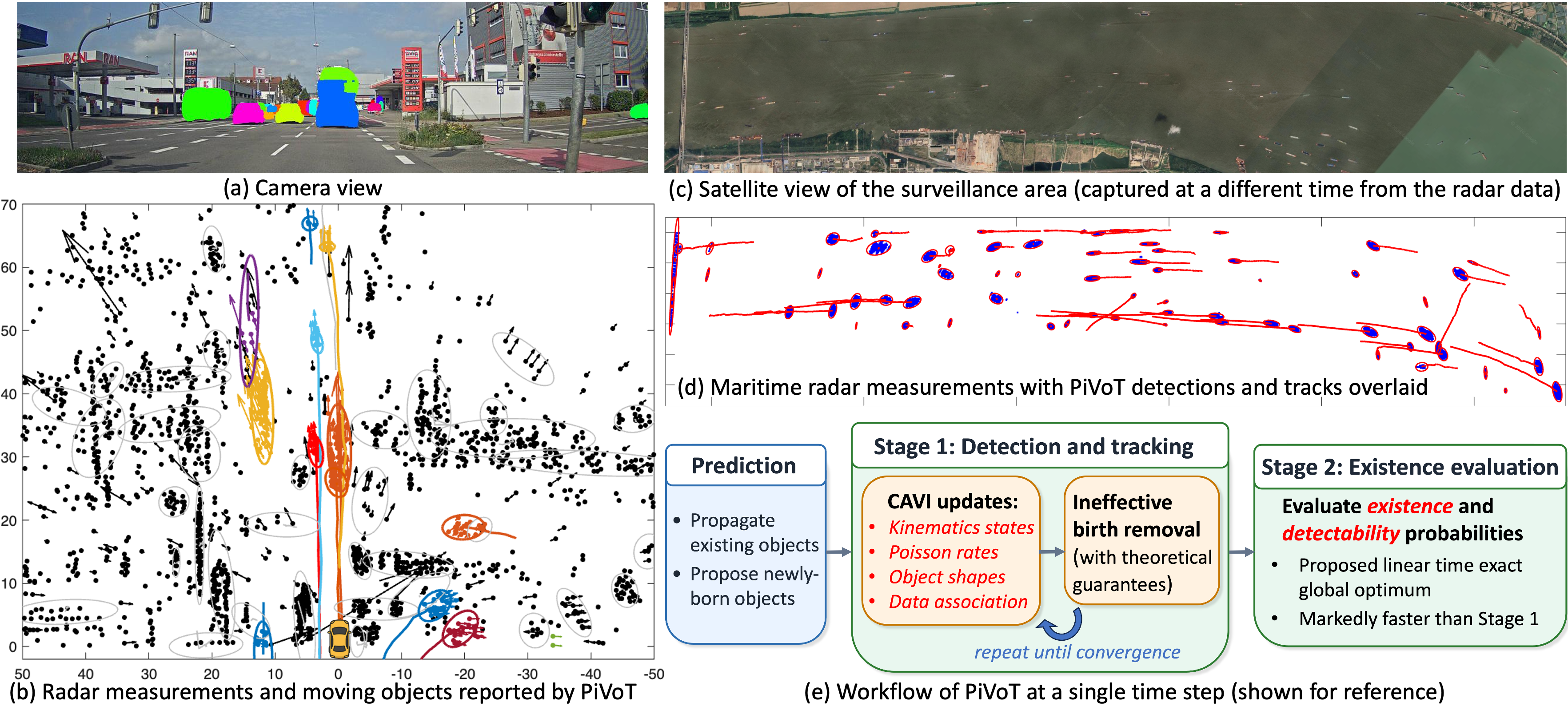}}
% {figure/FrontDemo3.png}}
    \caption{PiVoT joint detection and tracking on real radar data, with associated workflow.
    (a) Camera image from the RadarScenes dataset \cite{schumann2021radarscenes} (not used by PiVoT). (b) Corresponding radar point cloud with Doppler shown by arrows, where coloured points and arrows mark ground-truth moving-object returns. PiVoT estimates are overlaid as ellipses (shapes) and curves (tracks), with coloured estimates for moving objects and light-grey estimates for stationary objects or clutter, reliably recovering the true movers under heavy clutter.
    (c) Satellite view of the maritime surveillance area.
    (d) Maritime radar measurements (blue point clusters), overlaid by PiVoT estimates (red ellipses: shapes; curves: tracks). (e) Workflow of PiVoT at each time step.
    \label{fig: FrontDemo}
    }
    \vspace{-0.9em}
\end{figure*}

In such a data-constrained setting, multi-object tracking has historically been formulated within a model-based Bayesian framework \cite{bar1995multitarget,cox1996efficient}, where object kinematics and measurement generation are modelled probabilistically, and detection and tracking are performed as approximate Bayesian inference without the need for training data.

To account for the point-cloud-like measurements from modern high-resolution sensors, the non-homogeneous Poisson process (NHPP) measurement model was introduced in \cite{gilholm2005poisson}, characterising a generation process where each object produces a number of measurements drawn from a Poisson distribution.
Many Bayesian multi-object trackers have been actively developed based on this NHPP measurement model, commonly referred to as extended-object trackers \cite{granstrom2022tutorial}. Examples include random-finite-set-based \cite{granstrom2022tutorial,granstrom2019poisson,xia2023trajectory} and vector-based methods \cite{meyer2021scalable,li2023scalable,li2023adaptive}, employing inference techniques such as multiple-hypothesis tracking \cite{granstrom2019poisson}, belief propagation \cite{xia2023trajectory,meyer2021scalable}, and Monte Carlo methods \cite{li2023scalable,li2023adaptive}.

While effective in ideal scenarios, these trackers often struggle with accuracy and efficiency in challenging real-world settings with heavy clutter and a large number of objects. Such clutter may arise from noisy backgrounds, as commonly observed in automotive radar data (see Fig. \ref{fig: FrontDemo}(b)), and from sea or rain clutter \cite{BOLE2014139}. Large-scale tracking further involves many objects and massive data, for example when tracking drone swarms or groups of insects or animals. 

A further limitation of existing Bayesian extended-object trackers is that they often neglect or inefficiently exploit valuable Doppler point cloud information, which is routinely provided by many automotive and surveillance radars \cite{kari2023evolutionary,schumann2021radarscenes}.

To this end, we introduce PiVoT (\textbf{P}o\textbf{i}sson Measurements-based \textbf{V}ariational Multi-\textbf{o}bject Detection and \textbf{T}racking), an efficient variational inference framework for detecting and tracking a large, time-varying number of objects directly from point-cloud measurements, robust to dense clutter and supporting optional incorporation of Doppler information. PiVoT does not require training data and jointly estimates object kinematic states, shapes, Poisson rates (the mean measurement counts), existence and detectability probabilities within a unified probabilistic model. Its performance on challenging real automotive and maritime radar data\footnote{Video demonstrations of PiVoT on real automotive and maritime radar data, together with simulated heavy-clutter and large-scale tracking scenes, are provided in \cite{pivotweb}.} is illustrated in Fig.~\ref{fig: FrontDemo}.

\vspace{-0.5EM}
\subsection{Related work}
The efficiency of PiVoT arises from a novel variational inference routine tailored to the NHPP measurement model with uncertain object existence, admitting closed form, parallelisable coordinate ascent variational inference (CAVI) updates \cite{blei2017variational,bishop:2006:PRML} and a linear-time global optimum, see Fig.~\ref{fig: FrontDemo}(e).

Although CAVI has been applied to NHPP-based tracking in \cite{gan2022variational,gan2024variational}, these methods assume a fixed number of objects, and extending them to a time-varying cardinality introduces substantial technical difficulties (see Section~\ref{sec: challenges}). Previous works \cite{turner2014complete,lau2016structured,davey2007integrated} have also applied CAVI and expectation-maximisation to joint detection and tracking; however, they are under the point object model rather than the NHPP, and \cite{turner2014complete,lau2016structured} further require iterative belief propagation at each CAVI iteration. Extending these methods to NHPP is also problematic, \cite{lau2016structured} loses closed-form updates, while \cite{davey2007integrated} retains the original issue of exponential scaling with the number of objects. In contrast to \cite{turner2014complete,lau2016structured,davey2007integrated}, PiVoT preserves closed-form structure and scalability via a two-stage variational inference that targets distinct marginal posteriors to exploit the model’s tractable structure. It also addresses the computational bottlenecks in existing NHPP-based trackers \cite{granstrom2019poisson,xia2023trajectory,meyer2021scalable,li2023scalable,li2023adaptive}.

To handle a varying number of objects, PiVoT uses existence indicators and birth models, as in\cite{granstrom2019poisson,meyer2021scalable,xia2023trajectory,li2023scalable,turner2014complete,lau2016structured,davey2007integrated}. 
A key innovation to improve efficiency is adaptive birth removal (Fig. \ref{fig: FrontDemo}(e), Stage 1), which progressively removes ineffective births during CAVI and is formally supported by Theorem~\ref{main theorem}.

PiVoT extends naturally to Doppler point clouds through our Doppler-augmented NHPP model, which, unlike existing nonlinear Doppler models~\cite{knill2016direct,thormann2021incorporating,scheel2018tracking}, retains a linear Gaussian per-object likelihood and thus preserves the efficient closed-form update structure of the positional NHPP model~\cite{gilholm2005poisson}.

PiVoT directly processes all point cloud measurements without an external detector or clustering. Instead, its end to end joint inference inherently provides clutter robust, parallelisable clustering to detect an unknown number of objects while incorporating previous tracking results. This distinguishes PiVoT from existing NHPP-based trackers, e.g. \cite{granstrom2019poisson,meyer2021scalable,xia2023trajectory} that benefit from external clustering for data association, track initiation, or measurement censoring, and from automotive radar trackers, e.g. \cite{scheel2018tracking,tilly2020detection} that require clustering or pretrained networks as detectors. 
For Doppler point clouds, PiVoT does not prune near-zero Doppler measurements, as they may arise from tangential motion. Unlike conventional clustering methods \cite{scheiner2019multi, malzer2021constraint} that group measurements by similar Doppler values, PiVoT accounts for geometry and kinematics, enabling detection of objects that may generate measurements with quite different Doppler values. See Section~\ref{sec: Property of Doppler PiVoT} for details.

\subsection{Contributions}
Our key contribution is to present PiVoT, a fast, training-free, clutter resilient method for large-scale multi-object detection and tracking, applicable to both positional and Doppler point cloud measurements. Compared to existing NHPP-based trackers, PiVoT delivers substantial performance gains in challenging settings, as shown in Section \ref{sec: results}. It also demonstrates the following capabilities (illustrated in supplementary videos \cite{pivotweb}) beyond existing Bayesian trackers: 1) robustness in dense clutter, even when objects are difficult to distinguish by human inspection; 2) detecting and tracking a thousand objects in under a second on a standard laptop without gating; 3) to our knowledge, being the first fully training-free\footnote{No supervised learning is used in either detection or tracking.} detector and tracker to run efficiently on full-scale modern automotive radar datasets, with performance comparable to a deep-learning detection benchmark
as shown in Section \ref{sec: results}. 
At the same time, PiVoT additionally provides estimates of tracks, shapes, data association, detectability and existence probabilities.

The central methodological contribution of PiVoT is a novel two-stage variational inference framework (Section~\ref{sec: two stage overview}) that targets distinct marginal posteriors to exploit the NHPP model's tractable structure. This design addresses several inherent challenges of applying standard variational inference directly to the full posterior, as detailed in Section~\ref{sec: challenges}. 

Within this two-stage framework (see Fig.~\ref{fig: FrontDemo}(e)), three further methodological innovations enable PiVoT's efficiency and strong empirical performance:
1) a procedure for early removal of ineffective birth objects that would otherwise converge to having no associated measurements, with theoretical guarantees developed in Theorem~\ref{main theorem} in Section~\ref{sec: theoretical analysis}, thereby greatly accelerating Stage~1 inference; 2) an exact global optimiser for object existence and detectability probabilities in Stage~2, which reduces the standard quadratic complexity to linear time, as detailed in Section~\ref{sec: stage 2}; and 3) a Doppler-augmented NHPP model (Section~\ref{sec: Doppler model}) that, to our knowledge, is the first NHPP model to incorporate Doppler information while preserving a linear Gaussian per-object likelihood, providing a convenient building block for efficient handling of Doppler point clouds in both PiVoT and future trackers.

A preliminary version of PiVoT was presented in \cite{gan2025pivot}. The current paper substantially extends it with broader experiments and technical developments, including deeper CAVI theory and proofs, a numerically stable and efficient existence evaluation, a principled treatment of mean-field limitations, and a Doppler-augmented NHPP model (Doppler PiVoT) validated on automotive radar point clouds from RadarScenes~\cite{schumann2021radarscenes}.

\subsection{Paper layout}
The rest of the paper is organised as follows. Section~\ref{sec: Prob Formulation} formulates the problem as an inference task under a general NHPP model. Section~\ref{sec: PiVoT challenges solution} discusses the main inference challenges and introduces the proposed two-stage variational inference framework, with detailed Stages 1 and 2 inference steps presented in Sections~\ref{sec: stage 1} and \ref{sec: stage 2}. Section~\ref{sec: Doppler PiVoT} presents the Doppler-augmented NHPP model and Doppler PiVoT. Section~\ref{sec: Algorithm structure} describes the implementation details. Section~\ref{sec: results} presents the experimental results, while the theoretical analysis is deferred to Section~\ref{sec: theoretical analysis}. Section~\ref{sec: conclusion} concludes the paper.

Tables~\ref{tb: variable notation}--\ref{tab: notation_functions} in Appendix~\ref{apx: notation} summarise the main notation used throughout the paper. All appendices cited in the main paper are provided in the supplementary material.

\section{Problem formulation} \label{sec: Prob Formulation}
We assume that existing objects may cease to exist (die) and new objects may emerge (birth) at each time step, as in \cite{granstrom2019poisson,meyer2021scalable,xia2023trajectory,li2023scalable,turner2014complete}. At time step $n$, consider that there are $K_n=K_{n-1}+K_n^b$ potentially existing objects. These $K_n$ objects include $K_{n-1}$ legacy objects from the previous time step $n-1$, and $K_n^b$ newly born objects at the current time step $n$. We assign  each of the $K_n$ objects an index $k \in \{1,...,K_n\}$. Legacy objects from time $n-1$ are indexed as $k=1,...,K_{n-1}$, while newly born objects at time $n$ are indexed as $k=K_{n-1}+1,...,K_n$.

For each object $k=1,...,K_n$, we define the following variables. The kinematic state $X_{n,k}$ is a vector containing position, velocity, and any other motion parameters. The scalar Poisson rate $\Lambda_{n,k}>0$ equals the expected number of measurements generated by object $k$ at time $n$. The precision matrix $P_{n,k}$ has an inverse that describes the elliptical shape of object $k$. The binary indicators $D_{n,k}\in\{0,1\}$ and $E_{n,k}\in\{0,1\}$ represent detectability and existence respectively. More details about these features will be introduced in Section \ref{sec: measuremnt model}. The clutter Poisson rate (the expected number of clutter measurements) is denoted by $\Lambda_{n,0}$. Finally, we denote $X_n$, $\Lambda_{n}$, $P_{n}$, $D_{n}$ and $E_{n}$ as the respective collections of these variables for all objects $k=1,...,K_n$, for example, $\Lambda_n=[\Lambda_{n,1},...,\Lambda_{n,K_n}]$.

The measurements at time step $n$ are denoted by $Y_n=[Y_{n,1}, Y_{n,2},..., Y_{n,M_n}]$, where $M_n$ is the total number of measurements. We define the measurement-oriented association $\theta_n=[\theta_{n,1},\theta_{n,2},...,\theta_{n,M_n}]$, where each $\theta_{n,j}\in\{0,1,...,K_n\}$ for $j=1,2,...,M_n$ indicates the origin of the measurement $Y_{n,j}$. Specifically, $\theta_{n,j}=k$ indicates that $Y_{n,j}$ was generated by object $k$ if $k=1, ...,K_n$, or by clutter if $k=0$.

\subsection{Model for measurements and association} \label{sec: measuremnt model}
We assume that existing objects produce measurements following an NHPP measurement model \cite{gilholm2005poisson,gan2024variational} with a detection probability $p_{n,k}^d\!\in\!(0,1]$ \cite{granstrom2019poisson}. Specifically, an object $k=1,\ldots,K_n$ at time $n$ is detectable only if it exists, and given existence $E_{n,k}=1$ it is detectable, $D_{n,k}=1$, with probability $p_{n,k}^d$. An object can generate measurements only when $D_{n,k}=1$, so the detectability indicator is useful to model occasional occlusion. Mathematically,
\vspace{-0.3em}
\begin{align} \label{eq: detectability model}
    p(D_n|E_n)=\prod\nolimits_{k=1}^{K_n} p(D_{n,k}|E_{n,k}), \\[-2em]\notag
\end{align}
where $p(D_{n,k}|E_{n,k}\!=\!1)$ is $p_{n,k}^d$ if $D_{n,k}\!=\!1$ and $1\!-\!p_{n,k}^d$ otherwise, while $p(D_{n,k}|E_{n,k}\!=\!0)\!=\!\delta[D_{n,k}\!=\!0]$, where $\delta[\cdot]$ is Kronecker delta function. Finally, we assume $p_{n,k}^d=1$ for all newly born objects $k=K_{n-1}\!+\!1,...,K_n$ at current time $n$.

Conditional on $D_{n,k}=1$, object $k$ generates measurements according to an NHPP. Specifically, the number of measurements is Poisson distributed with rate $\Lambda_{n,k}$, and each measurement is independent and identically distributed with density $\ell(\cdot \mid X_{n,k},P_{n,k})$.
When $D_{n,k}=0$, object $k$ produces no measurements. 
Consequently, conditional on $D_n$, the full set of measurements $Y_n$ forms an NHPP obtained by superposing conditionally independent NHPPs from clutter and the $K_n$ objects with rates $\Lambda_{n,0}$ and $D_{n,k}\Lambda_{n,k}$ respectively. The likelihood is given by \cite{gilholm2005poisson}:
\vspace{-0.5em}
\begin{align}\notag
        &p(Y_{n},M_n|D_n,{X}_{n},\Lambda_n,P_n)=\frac{\exp(-\Lambda_{n,0}\!-\!\sum_{k=1}^{K_n}D_{n,k}\Lambda_{n,k})}{M_n !}\\[-0.5em] \label{likelihood poisson existence}
        & \!\times\prod_{j=1}^{M_n}\!\Big(\Lambda_{n,0} \ell_0(Y_{n,j})\!+\!\sum_{k=1}^{K_n}\!\Lambda_{n,k} D_{n,k} \ell(Y_{n,j}|X_{n,k},P_{n,k})\!\Big),\\[-2em]\notag
\end{align}
where $\ell_0(Y_{n,j})$ and $\ell(Y_{n,j}|X_{n,k},P_{n,k})$ denote the clutter and single-object measurement likelihoods, respectively.

Depending on the sensor, each measurement $Y_{n,j}$ may contain only position or position together with Doppler velocity. We first consider the classical positional-only model, where $Y_{n,j}$ reduces to a positional measurement $y_{n,j} \in \mathbb{R}^{d_Y}$ with $d_Y = 2$ for planar and $d_Y = 3$ for spatial tracking. For this positional-only model, i.e. $Y_{n,j}=y_{n,j}$, the clutter likelihood $\ell_0$ in \eqref{likelihood poisson existence} may be uniform over the surveillance area or follow a predefined spatial map, while the object measurement likelihood $\ell$ is assumed to follow a linear Gaussian form:
\vspace{-0.3em}
\begin{equation} 
   \ell(Y_{n,j}|X_{n,k},P_{n,k})= 
    \mathcal{N}(y_{n,j}; H X_{n,k},P_{n,k}^{-1}) \\[-0.3em]
\label{measurement model}
\end{equation}
where $H$ is the observation matrix mapping the object state to the positional measurement space, and $P_{n,k}^{-1}$ is the covariance matrix (with $P_{n,k}$ being its precision) that captures the object’s elliptical shape. Example measurements from this positional-only NHPP model, together with the corresponding $\Lambda_{n,k},P_{n,k},D_{n,k}$ and $E_{n,k}$ are illustrated in Fig. \ref{fig: NHPP Env}. 

\begin{figure}[t!]
    \centering
    \includegraphics[width=1\linewidth]{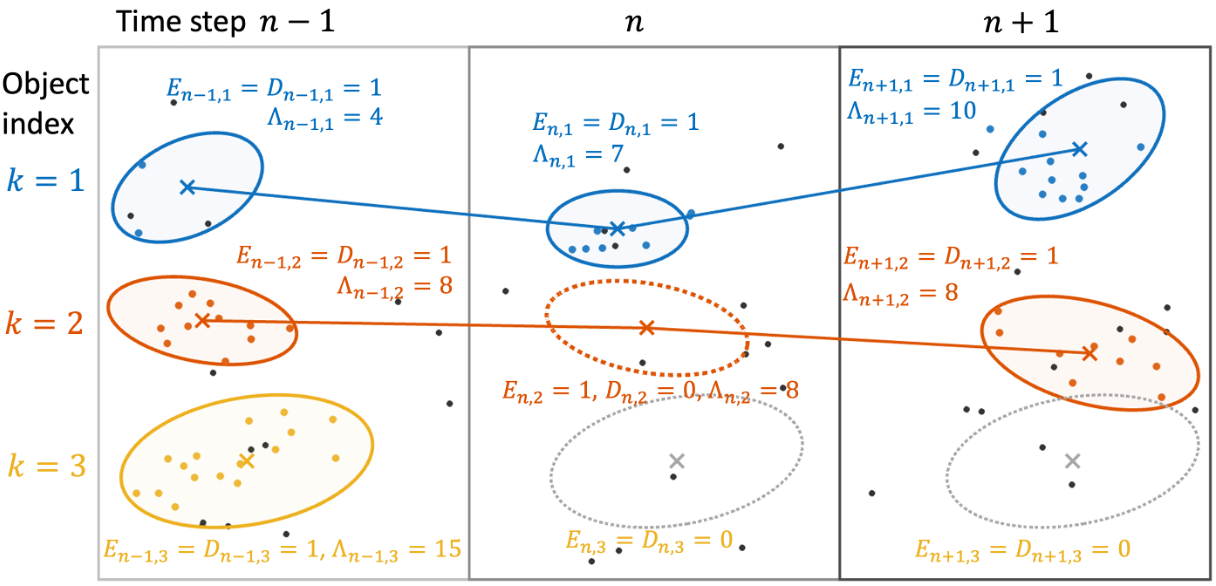}
    \caption{Simulated NHPP measurements from three objects over three time steps. Black dots denote clutter, and coloured dots are object-generated measurements with Poisson rate $\Lambda_{n,k}$ and shape $P_{n,k}$ (shown as ellipses), generated only when $D_{n,k}=1$ (which implies $E_{n,k}=1$). In particular, object $k=3$ no longer exists from time $n$ onwards, while object $k=2$ still exists at time $n$ but is not detectable (e.g. due to occlusion).}
    \label{fig: NHPP Env}
    \vspace{-0.5em}
\end{figure}

The Doppler NHPP model will be presented in Section~\ref{sec: Doppler PiVoT}.
The PiVoT inference framework in Sections~\ref{sec: PiVoT challenges solution}-\ref{sec: stage 2} applies to both settings, whereas the final update expressions in Section~\ref{sec: stage 1} are specific to the positional model introduced here; the corresponding Doppler case is given in Section~\ref{sec: Doppler PiVoT}.

By introducing the data association $\theta_n$, the joint likelihood of $Y_{n}$, $M_n$, $\theta_{n}$, $D_n$ can be expressed as follows, which can be verified using \eqref{likelihood poisson existence} following a similar procedure in \cite{gan2024variational}:
\begin{equation}
\label{eq: likelihood joint} 
\begin{split}
        &p(Y_{n},M_n,\theta_{n},D_n|{X}_{n},\Lambda_n,P_n,E_n)\\ 
        &\qquad =p(Y_{n}|\theta_{n},P_n,X_{n})p(\theta_n,M_n|\Lambda_n,D_n)p(D_n|E_n),
\end{split} 
\end{equation}
The association-conditioned likelihood in \eqref{eq: likelihood joint} is given by 
\vspace{-0.5em}
\begin{align} \label{eq: individual likelihood}p(Y_{n}|\theta_{n},P_n,X_{n})=\prod\nolimits_{j=1}^{M_n}p(Y_{n,j}|\theta_{n,j},P_n,X_{n}) \\[-1.7em]\notag
\end{align}
\begin{equation} \notag
   p(Y_{n,j}|\theta_{n,j},P_n,X_{n})=\begin{cases} \ell(Y_{n,j}|X_{n,\theta_{n,j}},P_{n,\theta_{n,j}}), & \!\!\text{$\theta_{n,j}\neq 0 $}\\
     \ell_0(Y_{n,j}), & \!\!\text{$\theta_{n,j}= 0 $}
\end{cases}
\end{equation}
The joint prior for $\theta_n,M_n$ in \eqref{eq: likelihood joint} can be verified to be \cite{gan2024variational}
\begin{align} \notag
    % &p(\theta_n,M_n|\Lambda_n,\!D_n)\!=\!\frac{e^{-\Lambda_{n\!,0}-\sum_{k=1}^K\!D_{n\!,k}\Lambda_{n\!,k}}}{M_n !}h(\Lambda_n,\!D_n,\theta_n) \\
    &p(\theta_n,M_n|\Lambda_n,D_n)= \tfrac{1}{M_n !}\exp(-\Lambda_{n,0}-\sum\nolimits_{k=1}^{K_n}D_{n,k}\Lambda_{n,k})\\ \label{eq: association prior}
     &\qquad\qquad\qquad\qquad\qquad \times h(\Lambda_n,D_n,\theta_n)
    \\ 
    \notag
    &h(\Lambda_n,D_n,\theta_n)=\Lambda_{n,0}^{\sum_{j=1}^{M_n} \delta[\theta_{n,j}=0]}\prod_{k=1}^{K_n} (\Lambda_{n,k}D_{n,k})^{\sum_{j=1}^{M_n} \delta[\theta_{n,j}=k]}\\[-1.5em] \label{eq: association prior form 1}
    \\[-0.5em] \label{eq: association prior form 2}
    &\ \  =\prod_{j=1}^{M_n} \Big(\Lambda_{n,0} \delta[\theta_{n,j}=0]+\sum_{k=1}^{K_n}\Lambda_{n,k}D_{n,k}\delta[\theta_{n,j}=k]\Big),
\end{align}
where \eqref{eq: association prior form 1} and \eqref{eq: association prior form 2} provide alternative but equivalent expressions, each useful in different contexts encountered later.
\subsection{Model for object dynamics and other features} \label{sec: object model}
\subsubsection{Model for legacy objects} \label{sec: legacy object model}
At each time $n$, we assume that transitions are independent across legacy objects $k=1,2,...,K_{n-1}$ and also independent among each object's features $X_{n,k}$, $\Lambda_{n,k}$, $P_{n,k}$, $E_{n,k}$. Specifically, each legacy object survives from time $n-1$ to $n$ with a survival probability $p_{n,k}^s\in(0,1]$, i.e. $p(E_{n,k}|E_{n-1,k}=1)=p_{n,k}^s$ if $E_{n,k}=1$ and $1-p_{n,k}^s$ if $E_{n,k}=0$. A legacy object that does not exist ($E_{n-1,k}=0$) remains nonexistent in future time steps ($E_{n,k}=0$). We assume linear Gaussian state transition with transition matrix $F_n$ and covariance $Q_n$ for each legacy object:
\vspace{-0.5em}
\begin{equation} \label{eq: linear Gaussian transition}
    p(X_{n,k}|X_{n-1,k})=\mathcal{N}(X_{n,k};F_nX_{n-1,k}, Q_n).\\[-0.5em]
\end{equation}
Moreover, we adopt the commonly assumed heuristic transition  (e.g. in \cite{granstrom2019poisson,granstrom2022tutorial,xia2023trajectory}) for rate $\Lambda_{n,k}$ and shape $P_{n,k}$, which preserves the Gamma and Wishart distribution forms of $\Lambda_{n,k}$ and $P_{n,k}$ in the prediction step while slightly increasing the uncertainty of $\Lambda_{n,k}$ and $P_{n,k}^{-1}$, as detailed in Section \ref{sec: Algorithm structure}.
\subsubsection{Model for newly born objects} \label{sec: birth model}
At each time $n$, we assume a birth model with a maximum of $K_n^b$ new births. For each newly born object $k=K_{n-1}+1,...,K_n$, its birth prior is assumed to be independent across features $X_{n,k}$, $\Lambda_{n,k}$, $P_{n,k}$, $E_{n,k}$, as well as independent of all other newly born and legacy objects' features. Specifically, 
the birth priors for existence $E_{n,k}$, state $X_{n,k}$, rate $\Lambda_{n,k}$, and shape (precision matrix) $P_{n,k}$ follow Bernoulli, Gaussian, Gamma, Wishart distributions, respectively:
\begin{align} \label{eq: birth prior}
\begin{aligned}
    &p(E_{n,k})\!=\!\text{Ber}( p_{n,k}^b), \ \quad \ \ \ \ p(X_{n,k})\!=\!\mathcal{N}(\mu_{n,k}^b,\Sigma_{n,k}^b), \\  
    &p(\Lambda_{n,k})\!=\!\mathcal{G}(\eta_{n,k}^b,\rho_{n,k}^b), \  \ \ \
    p(P_{n,k})\!=\!\mathcal{W}(\Phi_{n,k}^b,\phi_{n,k}^b),
\end{aligned}
\end{align}
where the Bernoulli parameter $p_{n,k}^b\in(0,1)$ denotes the probability of birth, i.e. $E_{n,k}=1$. $\mu_{n,k}^b,\Sigma_{n,k}^b$ denote the mean and covariance, $\eta_{n,k}^b,\rho_{n,k}^b$ are the shape and scale parameters, $\Phi_{n,k}^b$ is the scale matrix, and $\phi_{n,k}^b$ is the degrees of freedom.

For reliable detection of all potentially born objects in a cluttered environment, the maximum birth count $K_n^b$ should be sufficiently large, often far exceeding the actual number of births. However, PiVoT can adaptively reduce $K_n^b$, removing ineffective birth objects during inference to improve efficiency and refine the birth model, see Section \ref{sec: ineffective birth} and \ref{sec: initialisation, clustering demo}. For efficient detection/clustering over large areas, PiVoT typically uses identical non-informative birth priors $p(X_{n,k})$ (see Section \ref{sec: initialisation, clustering demo}), unlike the usual multi-Bernoulli birth setting.

\subsection{Inference goal: filtering posterior and predictive prior}

The goal of inference at time step $n$ is to obtain an accurate and efficient approximation of the following filtering posterior:
\vspace{-1.5em}
\begin{align}\notag
    \hat{p}_n(X_n,\Lambda_n,P_n &, \theta_n,E_n,D_n|Y_n)\propto\hat{p}_{n}(X_n,\Lambda_n,P_n,E_n)\\ \label{eq: filtering posterior definition}
    &\times p(Y_{n},M_n,\theta_{n},D_n|{X}_{n},\Lambda_n,P_n,E_n),\\[-2em]\notag
\end{align}
where the likelihood term (second line) follows the exact NHPP joint likelihood given in \eqref{eq: likelihood joint}-\eqref{eq: association prior} and \eqref{eq: detectability model}. The prior term $\hat{p}_{n}(X_n,\Lambda_n,P_n,E_n)$ consists of the exact birth prior (Section \ref{sec: birth model}) and the predictive prior for legacy objects, constructed using the transition model (Section \ref{sec: legacy object model}) and the approximated posterior from the previous time step. In PiVoT, this predictive prior retains the same independent factorisation and distribution forms as the birth prior, leading to:
\vspace{-0.5em}
\begin{align} \notag
    &\!\!\!\! \hat{p}_{n}(X_n,\Lambda_n ,P_n,E_n)\!=\!\!\prod_{k=1}^{K_n} \!\hat{p}_{n}(X_{n,k})\hat{p}_{n}(\Lambda_{n,k})\hat{p}_{n}(P_{n,k}\!)\hat{p}_{n}(E_{n,k})\\ \notag
    % &\hat{p}_{n}(E_{n,k})\!=\!\text{Ber}(p_{n,k}^{e \ \prime}), \ \hat{p}_{n}(X_{n,k})\!=\!\mathcal{N}(X_{n,k};\mu_{n,k}',\Sigma_{n,k}') \!\!\\ \notag
    % &\hat{p}_{n}(\Lambda_{n,k})\!=\!\mathcal{G}(\Lambda_k;\eta'_{n,k},\rho'_{n,k}), \  \hat{p}_{n}(P_{n,k})\!=\!\mathcal{W}(P_{n,k};\Phi_{n,k}',\phi_{n,k}')
    &\hat{p}_{n}(X_{n,k})=\mathcal{N}(\mu_{n,k}',\Sigma_{n,k}'), \ \quad \hat{p}_{n}(\Lambda_{n,k})=\mathcal{G}(\eta'_{n,k},\rho'_{n,k}), \\ \label{eq: prior parameters}
    &\hat{p}_{n}(P_{n,k})=\mathcal{W}(\Phi_{n,k}',\phi_{n,k}'), \ \quad \hat{p}_{n}(E_{n,k})=\text{Ber}(p_{n,k}^{e\prime}),\\[-2.0em]\notag
\end{align}
For newly born objects $k\!=\!K_{n-1}+1,...,K_n$, the prior parameters in \eqref{eq: prior parameters} are exactly identical to those in the birth prior \eqref{eq: birth prior}, i.e. $[\mu_{n,k}',\Sigma_{n,k}',...,\phi_{n,k}',p_{n,k}^{e \prime}]=[\mu_{n,k}^b,\Sigma_{n,k}^b,...,\phi_{n,k}^b,p_{n,k}^{b}]$. For legacy objects $k=1,...,K_{n-1}$, the prior is the predictive prior, with parameters computed from the prediction step, as detailed in \eqref{eq: prediction} in Section \ref{sec: Algorithm structure}.

In the rest of this paper, we omit the subscript $n$ in the notation for clarity, unless necessary to avoid confusion.

\section{PiVoT: Inference challenges and solution} \label{sec: PiVoT challenges solution}
This section first outlines the challenges of approximating the filtering posterior in \eqref{eq: filtering posterior definition}. We then present our solution, a two-stage variational inference framework that serves as the backbone of PiVoT and summarises the method in a nutshell.

\subsection{Inference challenges and motivation} \label{sec: challenges}
The key inference routine of PiVoT is CAVI, which demonstrated superior efficiency and reliability for tracking a fixed number of objects under the NHPP model \cite{gan2024variational,gan2022variational}. However, extending it to our models in Section \ref{sec: Prob Formulation} for a varying number of objects is challenging. A standard mean-field assumption:
\vspace{-0.5em}
\begin{align} \label{eq: naive mean-field}   q(\theta,X,\Lambda,P,E,D)=q(\theta)q(X)q(\Lambda)q(P)q(E,D),\\[-2.1em]\notag
\end{align}
fails when approximating $\hat{p}$ in \eqref{eq: filtering posterior definition} by minimising the Kullback-Leibler divergence (KLD) KL$(q\|\hat{p})$.
The reason is as follows. Since $\hat{p}$ is zero whenever a measurement is associated to a non-existent object, $q$ cannot place probability mass there. This forces variational inference (VI) into one of two failures: 1) forcing $q(E_k \!=\! 0) \!=\! 0$, meaning that object $k$ must exist with no uncertainty, or 2) preventing object $k$ from associating with measurements, making updates impossible. Either outcome renders the naive mean-field assumption in \eqref{eq: naive mean-field} fundamentally flawed. Detailed derivations are given in Appendix \ref{apx:limitation MF}.

Possible remedies include modifying the model to allow non-existent objects to generate measurements \cite{turner2014complete}, but inference performance is highly sensitive to the extra Poisson rates introduced for non-existent objects. Another approach replaces $q(\theta)$ in \eqref{eq: naive mean-field} with $q(\theta | D)$, which preserves detectability dependence but introduces computationally prohibitive evaluations of $q(\theta | D)$ over $2^{K_n}$ configurations as in \cite{davey2007integrated}. While general inference methods like Monte Carlo and gradient-based VI could be used, they compromise efficiency. Instead, PiVoT introduces alternative objectives and suitable approximations to retain efficient closed-form CAVI updates, as shown below.

\subsection{PiVoT: Two-stage variational inference framework} \label{sec: two stage overview}
PiVoT follows a mean-field approximation similar to \eqref{eq: naive mean-field}, but employs two separate inference stages with distinct, well motivated objectives to address inference challenges. Specifically, it approximates $\hat{p}$ in \eqref{eq: filtering posterior definition} as
\vspace{-0.5em}
\begin{align} \notag
q(X,\Lambda,P,\theta,E,D)\!=\!q_1(X)q_1(\Lambda)q_1(P)q_1(\theta)\prod_{k=1}^{K_n} q_2(E_k,D_k),\\[-1.3em]\label{eq: objective filtering posterior} \\[-2.2em]\notag
\end{align}
where $q_1$ and $q_2$ are evaluated in the first and second stages, respectively, as introduced below and summarised in Fig.~\ref{fig: FrontDemo}(e).

\subsubsection{Detection and tracking stage} \label{sec: first stage overview}
Stage 1 aims to use VI to approximate the marginal posterior $\hat{p}(X, \Lambda, P, \theta | Y)$ from \eqref{eq: filtering posterior definition} with $q_1(X)q_1(\Lambda)q_1(P)q_1(\theta)$:
\vspace{-0.5em}
\begin{align} \label{eq: stage 1 approximation goal}
    &q_1(X)q_1(\Lambda)q_1(P)q_1(\theta)\approx \sum_{E,D}\hat{p}(X, \Lambda, P, \theta, E, D | Y)\\[-0.4em] \label{eq: stage 1 unnormalised goal}
    &\ \propto\hat{p}(X,\Lambda,P) p(Y|\theta,P,X) \sum\nolimits_{D}p(\theta,M|\Lambda,D)\hat{p}(D),
\end{align}
Here, \eqref{eq: stage 1 unnormalised goal} follows from 
\eqref{eq: likelihood joint} and \eqref{eq: filtering posterior definition}. 
The factors $\hat{p}(X,\Lambda,P)$, $p(Y|\theta,P,X)$ and 
$p(\theta,M|\Lambda,D)$ are specified in 
\eqref{eq: prior parameters}, \eqref{eq: individual likelihood} and 
\eqref{eq: association prior}, respectively. The shorthand 
$\hat{p}(D)$ is defined as
\begin{align}
\begin{aligned} \label{eq: predictive Dk}
    &\hat{p}(D):=\sum\nolimits_E p(D|E)\hat{p}(E)= \prod\nolimits_{k=1}^{K_n}\hat{p}(D_k), \\ 
    &\hat{p}(D_k)=\sum\nolimits_{E_k}p(D_k|E_k)\hat{p}(E_k)=\text{Ber}(p_k^{e \prime}p_k^d),
\end{aligned}
\end{align}
where the simplified factorised form follows from \eqref{eq: detectability model} 
and \eqref{eq: prior parameters}.

Since $E,D$ are marginalised out, the resulting posterior is not necessarily zero for all possible associations $\theta$, avoiding the mean-field issue discussed in Section \ref{sec: challenges}. Parallelisable CAVI then iteratively updates each variational distribution $q_1$.  

This stage forms the core of PiVoT, jointly detecting new objects, tracking legacy objects, and estimating rates and shapes. Notably, the design in Section \ref{sec: stage 1} enables automatic clutter-robust clustering for birth detections while improving efficiency through early identification and removal of ineffective births, as supported by Theorem \ref{main theorem} in Section \ref{sec: ineffective birth}.

\subsubsection{Existence evaluation stage} \label{sec: second stage overview}
In Stage 2, the main objective is to independently evaluate $q_2(E_k,D_k)$ in \eqref{eq: objective filtering posterior} for each object $k=1,2,...,K_n$, by targeting the following object-wise marginal approximation with VI: 
\vspace{-0.5em}
\begin{align} \notag 
q_2(E_k,&D_k)q_2(\theta|E_k,D_k)q_1(X)q_1(\Lambda)q_1(P)\\ \label{eq: stage 2 marginal posterior}
    \approx &\sum\nolimits_{E_{k-},D_{k-}}\hat{p}(X, \Lambda, P, \theta, E, D | Y)\\ \notag
    & \propto  \ \hat{p}(X,\Lambda,P)\hat{p}(E_k)p(D_k|E_k)p(Y|\theta,P,X)\\ \label{eq: stage 2 unnormalised goal}
    & \qquad \times\sum\nolimits_{D_{k-}}p(\theta,M|\Lambda,D)\hat{p}(D_{k-}),\\[-2.1em]\notag
\end{align}
where $q_1(X)q_1(\Lambda)q_1(P)$ is obtained from the first stage and remains fixed in the second stage. The subscript $k-$ denotes all indexed variables except the $k$-th one, e.g. $E_{k-}$ includes all components of $E$ except $E_k$. \eqref{eq: stage 2 unnormalised goal} is derived similarly to \eqref{eq: stage 1 unnormalised goal}, with $\hat{p}(D_{k-})$ obtained directly from \eqref{eq: predictive Dk}. 

The rationale behind this stage is that $q_1(X) q_1(\Lambda) q_1(P)$ from the first stage should already approximate \( \hat{p}(X, \Lambda, P | Y) \) well, so keeping it fixed allows efficient evaluation of $q_2(E_k,D_k) q_2(\theta | E_k, D_k)$ to approximate $\hat{p}(E_k,D_k,\theta|Y)$, though only $q_2(E_k,D_k)$ is required in \eqref{eq: objective filtering posterior}.

Here $q_2(\theta | E_k, D_k)$ depends on $E_k, D_k$ to avoid the mean-field issue, and to avoid the  $2^{K_n}$ configurations when conditioning on all $E$, both discussed in Section \ref{sec: challenges}. Notably, Section \ref{sec: stage 2} presents a method that directly computes the globally optimal $q_2(E_k,D_k)$ for all $k=1,..., K_n$  without evaluating $q_2(\theta | E_k, D_k)$. This makes the second stage surprisingly fast, with the total cost of evaluating all $q_2(E_k,D_k)$ comparable to just two iterations of $q_1(\theta)$ updates in the first stage; thus Stage 2 takes much less time than Stage 1. 

For each object $k$, Stage 2 can be extended with an optional refinement step to improve the accuracy of inferred object features, as detailed in Appendix~\ref{apx: refinement}.

\subsection{Approximation for efficient two-stage inference} \label{sec: final approximation}
The two-stage VI requires the unnormalised marginal posteriors in \eqref{eq: stage 1 unnormalised goal} and \eqref{eq: stage 2 unnormalised goal}%.
, wherein the marginal associations in each are respectively given as follows using \eqref{eq: association prior}, \eqref{eq: association prior form 1}:
\vspace{-0.5em}
\begin{align} \notag
    &\sum_{D}p(\theta,M|\Lambda,D)\hat{p}(D)=\frac{e^{-\Lambda_{0}}}{M !}\Lambda_{0}^{\sum_{j=1}^{M} \delta[\theta_j=0]}\prod_{k=1}^{K_n}\Xi_k(\Lambda_k,\theta),\\ \label{eq: partial marginal assoc}
    &\sum_{D_{k-}}p(\theta,M|\Lambda,D)\hat{p}(D_{k-}\!)=\frac{e^{-\Lambda_{0}-D_k\Lambda_k}}{M !}\Lambda_{0}^{\sum_{j=1}^{M} \delta[\theta_j=0]}\\[-0.5em]  \notag
    &\qquad\qquad \quad \times(\Lambda_{k}D_k)^{\sum_{j=1}^{M} \delta[\theta_j=k]} \prod\nolimits_{k-}\Xi_k(\Lambda_k,\theta), \\[-2.0em]\notag
\end{align}
% where
\begin{align} \label{eq: original expectation}
    \Xi_k(\Lambda_k,\theta)&:=\E_{\hat{p}(D_k)}\left[e^{-D_k\Lambda_k} (\Lambda_kD_k)^{\sum_{j=1}^{M_n} \delta[\theta_{j}=k]}\right].\\[-2.0em]\notag
\end{align}

Although $\Xi_k(\Lambda_k,\theta)$ can be expressed analytically, their product complicates efficient optimisation in both stages. To enable efficient two-stage inference in PiVoT, we introduce one approximation: 
\vspace{-0.5em}
\begin{align} \label{eq: final approx for Xi}
    \Xi_k(\Lambda_k,\theta)\approx e^{-p_k^{e \prime}p_k^d\Lambda_k} (p_k^{e \prime}p_k^d\Lambda_k)^{\sum_{j=1}^{M_n} \delta[\theta_{j}=k]}.\\[-2.0em]\notag
\end{align}
This approximation replaces $D_k$ in $\Xi_k$ with its mean $\E_{\hat{p}(D_k)}[D_k] = p_k^{e \prime}p_k^d$ using \eqref{eq: predictive Dk},
corresponding to a first-order Taylor approximation of the expectation $\Xi_k$ \cite{papoulis2002probability}.

An alternative and perhaps more principled interpretation of \eqref{eq: final approx for Xi} is that it replaces a marginalised point process within the superposed measurement process by an NHPP with intensity $p_k^{e\prime}p_k^d \Lambda_k \ell(Y_j|X_k,P_k)$. This approximation is detailed and justified in Appendix~\ref{apx: alternative approx interpret}, where it is shown to be the unique NHPP minimising the KLD from the original marginalised process. 

Furthermore, Appendix \ref{apx: justification of approx for far away} analyses the impact of the approximation in \eqref{eq: final approx for Xi} on inference accuracy. 
A key takeaway is that, for an object sufficiently far from others, the approximation has minimal impact on accuracy, and residual error can be removed through the additional Stage 2 refinement steps.

\section{Stage 1: Detection and tracking} \label{sec: stage 1}
This section first presents the Stage 1 evidence lower bound (ELBO) \cite{blei2017variational}, then the corresponding CAVI updates and an early identification and removal procedure for ineffective birth objects, both under the Stage 1 block in Fig.~\ref{fig: FrontDemo}(e). This removal step, for which we develop a formal guarantee, enables efficient discovery of birth objects over large surveillance areas.

\subsection{Stage 1 ELBO} \label{sec: stage 1 ELBO}
Recall from Section \ref{sec: first stage overview} that Stage 1 aims to approximate the marginal posterior in \eqref{eq: stage 1 approximation goal} with variational inference.
This amounts to optimising $q_1$ to maximise the following ELBO:
\begin{align} \notag
\mathcal{F}_1=E_{q_1(X)q_1(\Lambda)q_1(P)q_1(\theta)}\log \frac{\bar{p}(X, \Lambda, P, \theta , Y)}{q_1(X)q_1(\Lambda)q_1(P)q_1(\theta)},
\end{align}
where the joint density $\bar{p}(X, \Lambda, P, \theta , Y)$ is defined as 
\vspace{-0.5em}
\begin{align} \notag
    &\bar{p}(X, \Lambda, P, \theta , Y) := \hat{p}(X,\Lambda,P) p(Y|\theta,P,X)\\ \label{eq: stage 1 unnormalised}
    &\!\!\!\times \tfrac{1}{M!} e^{-\Lambda_{0}-\sum_{k=1}^K\!p_k^{e \prime}p_k^d\Lambda_k} h (\Lambda,D\!=\![p_1^{e \prime}p_1^d,...,p_K^{e \prime}p_K^d],\theta),
\end{align}
where $h(\cdot)$ is defined in \eqref{eq: association prior form 1}. The last line in \eqref{eq: stage 1 unnormalised} approximates the marginal $\sum\nolimits_{D}p(\theta,M|\Lambda,D)\hat{p}(D)$ in \eqref{eq: partial marginal assoc} 
using \eqref{eq: final approx for Xi}, so that the resulting CAVI updates admit closed forms.

This ELBO $\mathcal{F}_1$ enables efficient CAVI to achieve the Stage 1's inference goal in \eqref{eq: stage 1 approximation goal}. Specifically, CAVI minimises KL$(q_1(X)q_1(\Lambda)q_1(P)q_1(\theta)\|\bar{p}(X, \Lambda, P, \theta | Y))$ by iteratively updating each of $q_1(X),q_1(\Lambda),q_1(P),q_1(\theta)$ to its optimal value while keeping others fixed. Each update ensures a non-increasing KLD, and subsequently, convergence to a stationary point is guaranteed.

\subsection{Coordinate ascent updates}  \label{sec: ca updates}
This section presents the updates for each $q_1(X)$, $q_1(\Lambda)$, $q_1(P)$, $q_1(\theta)$, using standard CAVI update formula \cite{bishop:2006:PRML,blei2017variational}: 
\begin{align} \label{eq: CAVI update formula}
    q_1(z_i) \propto \exp\left(\E_{q_1(z_{i-})} \log \bar{p}(X, \Lambda, P, \theta, Y) \right),
\end{align}
where $z_i$ represents each variable in $\{X,\Lambda,P,\theta\}$, $z_{i-}$ denotes all remaining variables, and $\bar{p}$ is given in \eqref{eq: stage 1 unnormalised}. 
The final update expressions given below are for the positional-only likelihood in \eqref{measurement model}; the corresponding Doppler-augmented updates are presented in Section \ref{sec: Doppler PiVoT}. 

The specific structure of $\bar{p}$ in \eqref{eq: stage 1 unnormalised} implies that the updated $q_1$ obtained via the CAVI formula in \eqref{eq: CAVI update formula} factorises as
\vspace{-0.5em}
\begin{align} \notag
&q_1(X)q_1(\Lambda)q_1(P)q_1(\theta)\!=\!\prod_{k=1}^{K_n} \!q_1(X_k)q_1(\Lambda_k)q_1(P_k)\!\prod_{j=1}^M \!q_1(\theta_j),
\\ \label{eq: all updates form}
    &q_1(X_k)\!=\!\mathcal{N}(\mu_k,\Sigma_k\!), \ q_1(\Lambda_k)\!=\!\mathcal{G}(\eta_k,\rho_k\!), \ q_1(P_k)\!=\!\mathcal{W}(\Phi_k,\phi_k\!)\\[-2.0em]\notag
\end{align}
where $(\mu_k,\Sigma_k)$, $(\eta_k,\rho_k)$, $(\Phi_k,\phi_k)$ are the corresponding variational parameter pairs for the Gaussian, Gamma, and Wishart distributions. Each parameter pair is updated in turn while the others are held fixed and used in the required expectations.

The following updates will apply to all legacy and newly-born objects, differing only in the respective model parameters.\\
\textbf{State update}: Setting $z_i=X$ in \eqref{eq: CAVI update formula} and substituting \eqref{eq: all updates form}, \eqref{eq: stage 1 unnormalised}, \eqref{eq: prior parameters}, \eqref{eq: individual likelihood}, the CAVI update for $q_1(X)$ reduces to updating each $q_1(X_k)$ for $k=1,...,K_n$ independently according to
\begin{align} \notag
    q_1(X_k) \!\propto \! \hat{p}(X_k) \!\prod\nolimits_{j=1}^M\!\!\exp(q_1(\theta_j\!=\!k) \E_{q_1(P_k)}\! \log \ell (Y_j|X_k,P_k))\\[-0.7em] \label{eq: general state update}
\end{align}
Using \eqref{eq: all updates form} and the positional-only likelihood $\ell$ in \eqref{measurement model}, each $\mu_k,\Sigma_k$ in  \eqref{eq: all updates form} can be obtained independently from a Kalman filter with prior $\hat{p}(X_k)$ in \eqref{eq: prior parameters} and a linear Gaussian likelihood:
\vspace{-0.5em}
\begin{align} \label{eq: state update}
    & \qquad\qquad q_1(X_k)\propto \hat{p}(X_k)\mathcal{N}\left(\overbar{y}_k;HX_{k},\overbar{R}_k\right), \\ \notag
    &\overbar{R}_k=\frac{\Phi_k^{-1}}{\phi_k\sum_{j=1}^{M}q_1(\theta_{j}=k)}, \quad 
    \overbar{y}_k=\frac{\sum_{j=1}^{M}y_{j} q_1(\theta_{j}=k)}{\sum_{j=1}^{M}q_1(\theta_{j}=k)}.\\[-2.0em]\notag
\end{align}
\textbf{Rate update}: Similarly, setting $z_i=\Lambda$ in \eqref{eq: CAVI update formula} and substituting \eqref{eq: all updates form}, \eqref{eq: stage 1 unnormalised}, \eqref{eq: prior parameters}, the update of $q_1(\Lambda)$ reduces to independently evaluating each $q_1(\Lambda_k)$ for $k=1,...,K_n$ as
\begin{align} \notag
    q_1(\Lambda_k)=\mathcal{G}(\eta_k,\rho_k\!)\propto \hat{p}(\Lambda_k)\exp(-p_k^{e \prime}p_k^d\Lambda_k)\Lambda_k^{\sum_{j=1}^M q_1(\theta_j=k)},\\[-2.5em]\notag
\end{align}
\begin{align} \label{eq: rate update}
    \eta_k=\eta'_{k}\!+\!\sum_{j=1}^M q_1(\theta_j=k), \quad \ \rho_k=\rho'_{k}/(p_k^{e \prime}p_k^d\rho'_{k}+1).\\[-2em]\notag
\end{align}
\textbf{Shape update}: For the CAVI update of $q_1(P)$, using \eqref{eq: CAVI update formula} with \eqref{eq: all updates form}, \eqref{eq: stage 1 unnormalised}, \eqref{eq: prior parameters}, \eqref{eq: individual likelihood}, it becomes evaluating each $q_1(P_k)$ via
\begin{align}\notag
    q_1(P_k) \!\propto\! \hat{p}(P_k)\!\prod\nolimits_{j=1}^M\!\!\exp(q_1(\theta_j\!=\!k) \E_{q_1(X_k)}\! \log \ell (Y_j|X_k,P_k)).
\end{align}
With the prior $\hat{p}(P_{k})=\mathcal{W}(\Phi_{k}',\phi_{k}')$ in \eqref{eq: prior parameters} and the positional-only likelihood $\ell$ in \eqref{measurement model}, this gives $q_1(P_k)=\mathcal{W}(\Phi_k,\phi_k)$ with parameters updated independently for each $k=1,...,K_n$ as:
\vspace{-0.5em}
\begin{align}\notag
    &\Phi_{k}^{-1}\!=\!{\Phi_{k}'}^{\!-1}\!\!+\!\!\sum_{j=1}^M q_1(\theta_j\!=\!k)\!\Big[\!(y_j\!-\!\!H\mu_k)(y_j\!-\!\!H\mu_k)^{\!\top} \!\!+\!H\Sigma_kH^{\!\top}\!\Big]\\[-0.5em] \label{eq: shape update}
    & \qquad\qquad\qquad \phi_k= \phi_{k}'+\sum_{j=1}^M q_1(\theta_j=k).\\[-2.0em]\notag
\end{align}
\textbf{Association update}: Finally, for the update of $q_1(\theta)$, using \eqref{eq: CAVI update formula} with \eqref{eq: all updates form}, \eqref{eq: stage 1 unnormalised}, \eqref{eq: prior parameters}, \eqref{eq: individual likelihood}, the association for each data $Y_j$ (i.e. $q_1(\theta_j)$) is independently updated for $j=1, ...,M$ as 
\vspace{-0.5em}
\begin{align} \label{eq: association update}
    &q_1(\theta_j)\propto \Lambda_0\ell_0(Y_j)\delta[\theta_{j}=0]+\sum_{k=1}^{K_n} p_k^{e \prime}p_k^d S_j^k\delta[\theta_{j}=k],\\\notag
    &S_j^k\!:= \! \exp( \E_{q_1(\Lambda_k)}\!\log\Lambda_k\!+\!\E_{q_1(X_k)q_1(P_k)}\!\log\ell(Y_j|X_k,P_k)).
\end{align}
Substituting the positional-only likelihood $\ell$ in \eqref{measurement model}, we have 
\begin{align} \label{eq: Sjk positional only l}
    & \qquad\qquad\qquad\qquad S_j^k = \exp(T_j^k),\\ \notag
    \!\!\!T_j^k&:= \!-\tfrac{1}{2}\Big[(y_j\!-\!H\mu_k)^{\!\top} \!\phi_k\Phi_k (y_j\!-\!H\mu_k)\!+\!\Tr(\phi_k\Phi_kH\Sigma_kH^\top\!)\Big] \\ \label{eq: Tjk definition}
    & +\psi(\eta_k)\!+\!\log \rho_k \!+\! \tfrac{1}{2}\big[\psi_{d_Y}\!(\tfrac{\phi_k}{2})\!+\!\log|\Phi_k|\!-\!d_Y\!\log(\pi)\big]
\end{align}
where $d_Y$ is the dimension of $y_j$, $|\cdot|$ denotes the determinant, and $\Tr(\cdot)$ denotes the trace. $\psi(\cdot)$ and $\psi_{d_Y}\!(\cdot)$ are the digamma and multivariate digamma functions, respectively.

\begin{figure*}[tp!]
\centerline{\includegraphics[width=18cm]{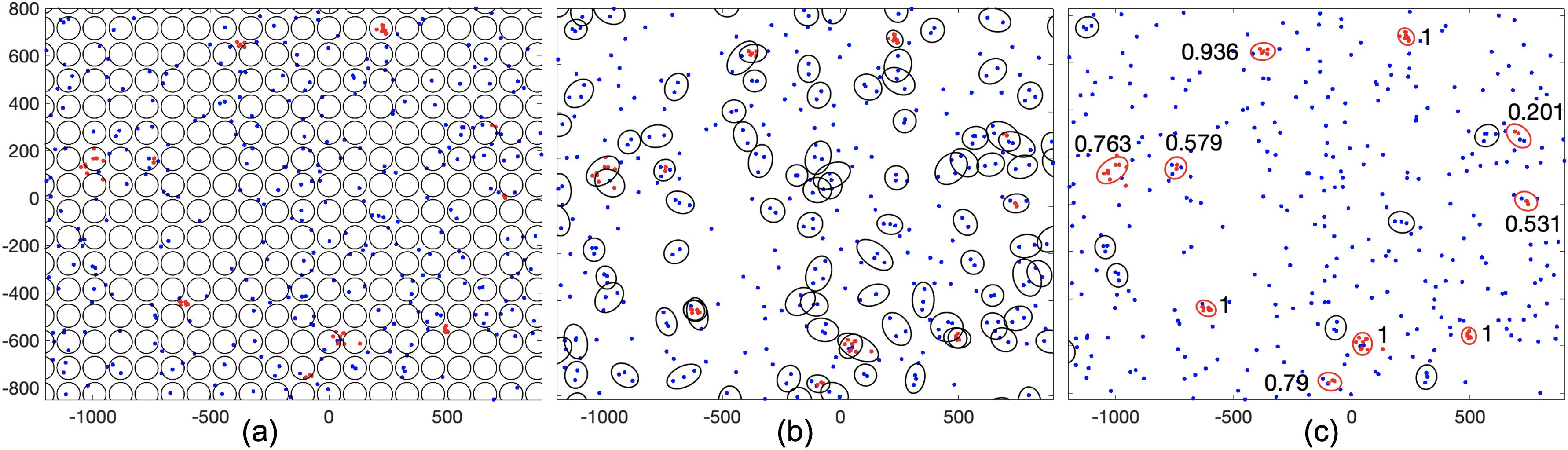}}
    \caption{Clutter-robust detection using PiVoT Stage 1. The data is from Dataset 3 (first time step) in Section \ref{sec: comparison simulation}, showing a zoomed-in region of the surveillance area. Blue and red dots represent clutter and true object measurements. Black ellipses denote the positions in $q_1(X_k)$ and shapes $q_1(P_k)$ in that iteration. (a) Initial CAVI iteration with 1312 births across the full surveillance area. (b) 3rd iteration with 392 remaining births after ineffective birth removal (Section \ref{sec: ineffective birth}). (c) 24th iteration (converged) with 47 remaining births. Red-highlighted ellipses mark objects with existence probability above 0.1, computed in Stage 2 (Section \ref{sec: stage 2}) and displayed next to them. Notably, these red ellipses precisely correspond to true objects.
    \label{fig:clustering}
    }
    \vspace{-0.9em}
\end{figure*}

\subsection{Early identification of ineffective births} \label{sec: ineffective birth}
Exhaustive birth detection under heavy clutter requires initially setting a birth count $K_n^b$ far above the true number, increasing feature updates and association dimensionality. However, many ineffective births that ultimately fail to associate with measurements upon CAVI convergence are naturally irrelevant to new object detection. Identifying and removing them early, by eliminating them from the prior and reducing the association dimension, greatly improves CAVI efficiency.

PiVoT implements this by monitoring the total association mass, $\sum_{j=1}^M q_1(\theta_{j}=k)$, in each CAVI iteration. If this count falls below a predefined threshold $L$ for newly born objects $k=K_{n-1}+1,...,K_n$, those births are deemed ineffective and removed from the model, and subsequent CAVI iterations update only the remaining objects' $X,\Lambda,P$, continuing from their current variational distribution (i.e. without restarting the optimisation, as illustrated in Fig. \ref{fig: FrontDemo}(e)).

This procedure is formally justified by Theorem~\ref{main theorem} in Section~\ref{sec: theoretical analysis}, informally stated below.
In short, any removed birth is provably redundant at CAVI convergence.

\noindent\textbf{Theorem~\ref{main theorem} (informal).}
\textit{Under mild modelling conditions, one can compute a sufficient threshold $L$ such that, if $\varsigma_k\!:=\!\sum_{j=1}^M \!q_1(\theta_{j}\!=\!k) \!<\! L$ at any CAVI iteration for any newly born object $k$, then subsequent CAVI updates drive $\varsigma_k$ monotonically to $0$.
Equivalently, at convergence, that birth object assigns zero probability of associating with any measurement.
}

The formal theorem and full analysis are deferred to Section~\ref{sec: theoretical analysis}. To our knowledge, existing CAVI analyses do not provide a computable sufficient condition certifying that a component will be driven to a zero-association fixed point. We derive such a condition by bounding how the total association mass propagates through subsequent coordinate updates, and by using the Lambert-W characterisation in Corollary~\ref{lemma: lambert W corr} in Appendix~\ref{sec: theoretical lemma and proof} to obtain an explicit sufficient threshold.

As Theorem~\ref{main theorem} gives only a sufficient condition, in practice we set $L$ heuristically, for example $L=0.5$ for speed or $L=0.05$ for more conservative pruning, noting that an overly large $L$ may discard valid births.

The modelling conditions required by Theorem~\ref{main theorem} are given in Assumption \ref{assump} in Section~\ref{sec: theoretical analysis}. Notably, they are satisfied by the efficient Model Scenario in the following Section \ref{sec: initialisation, clustering demo}, together with an extra uniform clutter assumption.

\subsection{Efficient clutter-robust birth detection: parameter suggestions and  demo} \label{sec: initialisation, clustering demo}
Denote by $X_k^p$ and $X_k^v$ the positional and remaining components of the kinematics state $X_k$. We recommend the following model scenario for efficient PiVoT implementation:

\noindent\textbf{Model Scenario:} 
\textit{For each newly-born object $k$, assume the birth prior in \eqref{eq: birth prior} is factorisable as $p(X_k)=p(X_k^p)p(X_k^v)$, and that $p(X_k^p)$ is a flat Gaussian approximating a uniform positional birth prior.}

In this model scenario, the uninformative positional birth prior greatly simplifies birth state update: \eqref{eq: state update} reduces to $q_1(X_k)=\mathcal{N}(X_k^{p};\overbar{y}_k,\overbar{R}_k)p(X_k^v)$; see Remark \ref{birth prior remark} under Assumption \ref{assump} in Section~\ref{sec: theoretical analysis}. Hence, the updated birth state is obtained directly without using a Kalman filter in any CAVI iteration, which reduces the per birth computational cost.

For efficient clutter-robust birth detection, we suggest:

\textit{(i)} The above Model Scenario holds, i.e. a factorised prior with a flat Gaussian $p(X_k^p)$ is used for all newly-born objects.

\textit{(ii)} Constrain each newly born object's initial positional variational distribution $q_1(X_k^p)$ to a distinct small region (crucial for updated $q_1(X_k)$ to localise the true object within it under dense clutter \cite{gan2024variational}). This can be achieved by setting a small covariance $\Sigma_k^p$ with a distinct mean $\mu_k^p$, where $\mu_k^p$ and $\Sigma_k^p$ denote the corresponding positional subvector and submatrix of the variational parameters $\mu_k$ and $\Sigma_k$ in \eqref{eq: all updates form}.

\textit{(iii)} Tile the entire surveillance area with these small regions for exhaustive detection, which implicitly determines an often very large $K_n^b$. However, ineffective birth removal in PiVoT Stage 1 inference (Section \ref{sec: ineffective birth}) then adaptively reduces $K_n^b$, improving efficiency and refining the model.

An example of such an initial $q_1(X_k)$ is shown in Fig. \ref{fig:clustering}(a), where each initial $q_1(X_k)$ covers a distinct small region, and collectively spans the full surveillance area (extending beyond the figure). This results in an initial $K_n^b$ of $1312$. Fig. \ref{fig:clustering}(b) and \ref{fig:clustering}(c) further illustrate the reliable clustering performance of PiVoT Stage 1 under dense clutter, using the same parameter settings as in Section \ref{sec: comparison simulation}. The proposed early ineffective birth removal progressively reduces $K_n^b$ to $952$, $631$, $392$ in the first three iterations, ultimately converging at $47$ by iteration $24$, greatly improving efficiency. Each iteration refines cluster positions while gradually eliminating those containing only sparse measurements. Notably, the relatively high existence probabilities (computed in PiVoT Stage 2) accurately distinguish true births from clutter-induced clusters in Fig. \ref{fig:clustering}(c).

\section{Stage 2: Existence evaluation} 
This section formulates the Stage 2 ELBO and its optimal $q_2(E_k,D_k)$ over existence and detectability (see Fig.~\ref{fig: FrontDemo}(e), Stage 2). While direct evaluation scales as $\mathcal{O}(K_n^2M)$, we derive an exact closed-form expression that reduces the cost to $\mathcal{O}(K_nM)$. We also provide a numerically stable implementation that preserves the $\mathcal{O}(K_nM)$ complexity,
making Stage 2 comparable in cost to a few updates of $q_1(\theta)$ in \eqref{eq: association update}.
\label{sec: stage 2}
\subsection{Stage 2 ELBO and the optimal variational solution}
\subsubsection{Stage 2 ELBO} Recall from Section \ref{sec: second stage overview} that Stage 2 inference targets \eqref{eq: stage 2 marginal posterior}. Specifically, holding Stage 1 solutions
$q_1(X)q_1(\Lambda)q_1(P)$ fixed, Stage 2 seeks $q_2(\theta,E_k,D_k)$ by maximising, for each object $k=1,...,K_n$, the ELBO
\begin{align}\notag
    \mathcal{F}_2^k\!=\!E_{q_2(\theta,E_k,D_k\!)q_1\!(\!X\!)q_1\!(\!\Lambda\!)q_1\!(\!P\!)}\!\log\!\frac{\bar{p}(X, \Lambda, P, \theta , E_k, D_k, Y)}{q_2(\theta,E_k,D_k\!) q_1\!(\!X)q_1\!(\Lambda)q_1\!(\!P)}
\end{align}
where the structured variational distribution $q_2(\theta,E_k,D_k)=q_2(E_k,D_k)q_2(\theta|E_k,D_k)$ is used to fully preserve conditional dependence, and the joint density $\bar{p}$ is defined as 
\begin{align} \notag
    \bar{p}(X, \Lambda, P, \theta , E_k, D_k, & Y) = \hat{p}(X,\Lambda,P)\hat{p}(E_k)p(D_k|E_k)\\ \label{eq: stage 2 unnormalised}
    &\times p(Y|\theta,P,X)\bar{p}(\theta,M|\Lambda,D_k),\\[-2.0em]\notag
\end{align}
\begin{align} \label{eq: stage 2 marginal assoc}
    &\bar{p}(\theta,M|\Lambda,D_k):=\tfrac{1}{M !}e^{-\Lambda_{0}-D_k\Lambda_k-\sum\nolimits_{s=1,s\neq k}^{K_n} p_s^{e \prime}p_s^d\Lambda_s }\\\notag
    & \!\times \! h (\Lambda,D\!=\![p_1^{e \prime}p_1^d,...,p_{k-1}^{e \prime}p_{k-1}^d,D_k,p_{k+1}^{e \prime}p_{k+1}^d,...,p_K^{e \prime}p_K^d],\theta),
\end{align}
where $h$ is given in \eqref{eq: association prior form 2}. The only approximation in $\mathcal{F}_2^k$, relative to the ELBO under the exact Stage~2 unnormalised posterior in \eqref{eq: stage 2 unnormalised goal}, is the use of $\bar{p}(\theta,M|\Lambda,D_k)$ in \eqref{eq: stage 2 marginal assoc} to approximate the marginal
$\sum_{D_{k-}} p(\theta,M\mid\Lambda,D)\hat{p}(D_{k-})$ in \eqref{eq: stage 2 unnormalised goal} and \eqref{eq: partial marginal assoc}. This approximation follows directly from \eqref{eq: final approx for Xi}, in the same manner as the Stage~1 ELBO formulation in Section~\ref{sec: stage 1 ELBO}.

\subsubsection{Optimal variational distribution} 
Maximising the Stage 2 ELBO $\mathcal{F}_2^k$ with Stage 1 solution $q_1(X)q_1(\Lambda)q_1(P)$ fixed admits a closed-form global optimum $q_2^*(E_k,D_k)q_2^*(\theta|E_k,D_k)$ for each object $k=1,2,...,K_n$.
Using the optimal update formula for conditional variational distributions \cite{Gan2022,bishop:2006:PRML}, this optimum is given by
\vspace{-0.5em}
\begin{align} \notag
    &\log q_2^*(\theta|E_k,D_k)= \E_{q_1(X,\Lambda,P)}\log \bar{p}(X, \Lambda, P, \theta , E_k, D_k,  Y)\\ 
    & \label{eq: conditional association} \quad\quad\quad\quad\quad\quad\quad\quad + c_1\\\label{eq: direct conditional update}
    & \log q_2^*(E_k,D_k) = c_2-\E_{q_2^*(\theta|E_k,D_k)}\log q_2^*(\theta|E_k,D_k)\\ \notag
    & \  + \E_{q_1(X,\Lambda,P)q_2^*(\theta|E_k,D_k)}\log \bar{p}(X, \Lambda, P, \theta , E_k, D_k,  Y),\\[-2.0em]\notag
\end{align}
where $c_1,c_2$ are constants. Recall that only $q_2^*(E_k,D_k)$ is needed as the filtering posterior in \eqref{eq: objective filtering posterior}. However, from \eqref{eq: direct conditional update}, direct computation requires first evaluating $q_2^*(\theta|E_k,D_k)$, which alone entails $M$ evaluations of the categorical distribution $q_2^*(\theta_j|E_k,D_k)$ of size $K_n+1$, leading to a total of $\mathcal{O}(K_n^2 M)$ exponential/logarithm operations for all $k=1,...,K_n$. This makes direct computation infeasible for large $K_n$ and $M$. Next, we introduce a surprisingly efficient method to compute all $q_2^*(E_k,D_k)$ directly.

\subsection{Efficient evaluation of the optimal $q_2^*(E_k,D_k)$} \label{sec: efficient optimal}
Here we derive an exact expression for $q_2^*(E_k,D_k)$ with overall complexity $\mathcal{O}(K_n M)$ for all $k=1,\ldots,K_n$.
The key idea is to collapse the dependence on $\theta$ by substituting the normalised form of $\log q_2^*(\theta| E_k,D_k)$ into \eqref{eq: direct conditional update}, and then rearranging the resulting expression until the summation over $\theta$ becomes tractable.
The final form reveals that the expensive part ($S_j^\text{sum}$ in \eqref{eq: Sj sum}) for computing $q_2^*(E_k,D_k)$ is shared across different $k$ and therefore needs to be computed only once. This shared structure is not evident from the direct update in \eqref{eq: direct conditional update}, yet is crucial for efficient evaluation.
We defer the detailed derivation to Appendix~\ref{apx: Stage2 derivation} and present the final form below:
\vspace{-0.5em}
\begin{align}\label{eq: existance update}
    q_2^*(E_k,D_k)\propto\hat{p}(E_k)p(D_k|E_k)\exp(g(D_k)),\\[-2.5em]\notag
\end{align}
\begin{align} \label{eq: simplified g}
g(D_k)&=\!-\eta_k\rho_kD_k\!+\!\sum_{j=1}^M\log\! \Big(\!S_j^\text{sum}+(D_k-p_k^{e \prime}p_k^d)S_j^k\Big)\\[-3em]\notag
\end{align}
\begin{align}\label{eq: Sj sum}
S_j^\text{sum}&=\Lambda_0\ell_0(Y_j)+\sum_{k=1}^{K_n} p_k^{e \prime}p_k^d S_j^k,\\[-2.0em]\notag
\end{align}
where $S_j^k$ is defined in \eqref{eq: association update}, and $S_j^\text{sum}$ is computed once for all $k$. If $q_1(\theta)$ is the last updated distribution in Stage 1, all $S_j^k$ have already been computed in \eqref{eq: association update} along with $S_j^\text{sum}$ as the normalisation constant. The other quantities $\hat{p}(E_k)$ and $p_k^{e\prime}$ are defined in \eqref{eq: prior parameters}, $p(D_k|E_k)$ and $p_k^d$ in \eqref{eq: detectability model}, and $\eta_k$, $\rho_k$ in \eqref{eq: all updates form}.

The globally optimal $q_2^*(E_k,D_k)$ in \eqref{eq: existance update} and \eqref{eq: simplified g} is now in a highly simplified form, requiring only $\mathcal{O}(K_nM)$ exponential/logarithm operations (the main computational cost) for updating all $K_n$ objects, compared to at least $\mathcal{O}(K_n^2 M)$ via direct computation in \eqref{eq: direct conditional update}. This is much more efficient and matches the complexity of updating $q_1(\theta)$ in \eqref{eq: association update}.

Finally, the object $k$'s existence probability is 
\vspace{-0.3em}
\begin{align} \label{eq: marginal existance calc}
    &q_2^*(E_k)=\text{Ber}(p_k^e)\propto \hat{p}(E_k)\sum_{D_k}p(D_k|E_k)\exp (g(D_k)) \\\notag
    &p_k^e=l_e/(l_e+l_n), \qquad\quad l_n=(1-p_k^{e\prime})\exp (g(D_k=0)),\\\notag
    &l_e=p_k^{e\prime} \big[p_k^d\exp (g(D_k=1))+(1-p_k^d)\exp (g(D_k=0))\big],
\end{align}
%}
\subsection{Numerically stable computation of $g(D_k)$} \label{sec: numerical stable gD}
The numerically stable evaluation of $g(D_k)$ using \eqref{eq: simplified g} requires care, especially for $g(D_k=0)$, which involves leave-one-out log-sums. A naive stable evaluation would recompute these sums using log-sum-exp for each $k$, reintroducing quadratic complexity. We avoid this by reusing stable global sums where possible, with a separate treatment only when the removed term is the largest one in the sum. Appendix~\ref{apx: numerical stable gD} gives the detailed numerically stable computation of $g(D_k)$, specified in \eqref{eq: gD1 computation}--\eqref{eq: numerical stable g0 k = Ij}, which preserves the $\mathcal{O}(K_nM)$ complexity.

\section{PiVoT with Doppler Velocity Measurements} 
\label{sec: Doppler PiVoT}

PiVoT can be readily extended to incorporate Doppler velocity measurements, which supply crucial radial motion information commonly exploited in radar systems to greatly improve tracking accuracy. This section first introduces a new Doppler-incorporated Poisson measurement model that retains the efficient linear Gaussian single-object likelihood. We then derive the corresponding PiVoT updates, where most of the existing updates remain unchanged. 
\subsection{Doppler-augmented Poisson measurement model} \label{sec: Doppler model}

\begin{figure*}[tp!]
\centerline{\includegraphics[width=16cm]{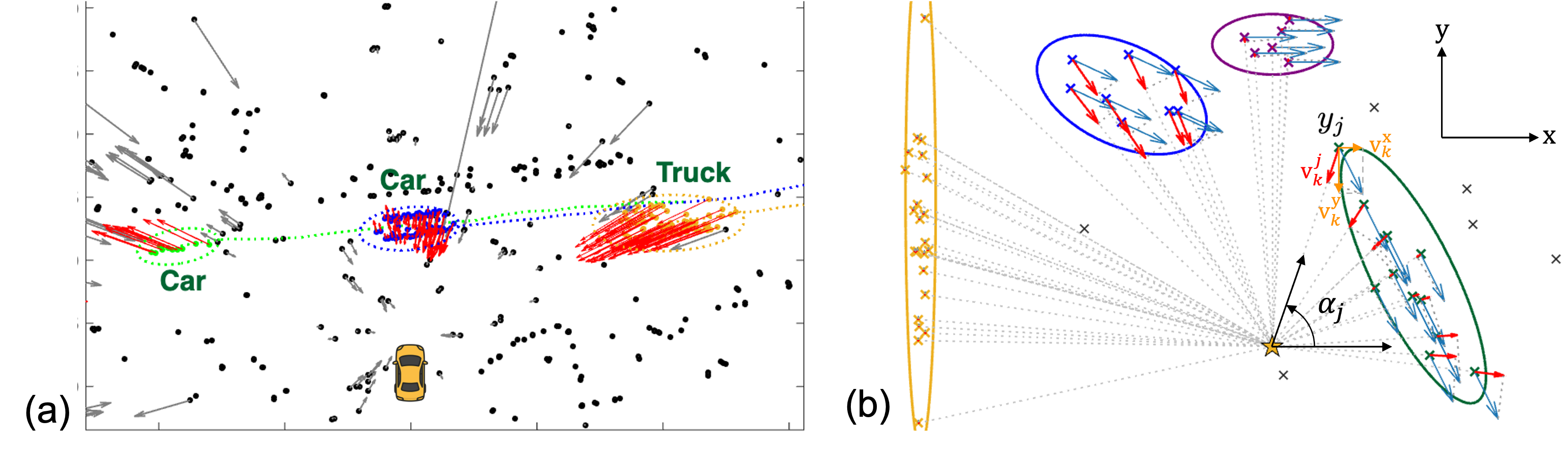}}
    \caption{(a) Automotive radar data from \cite{schumann2021radarscenes}, with PiVoT estimates of moving objects overlaid (ellipses: shape, curves: tracks). Coloured points and arrows denote positional measurements and Doppler velocities from ground truth moving objects, whereas black points and arrows denote static background and clutter. (b) Simulated measurements under the Doppler NHPP model. Ellipses show object shapes $P_k$. Crosses denote positional measurements $y_j$ from clutter (black), stationary objects (yellow) and moving objects (other colours). Measurements from the same object share the same true velocity, shown by blue arrows, while their radial velocities relative to the sensor (yellow star) are indicated by red arrows. One example detection $y_j$ is annotated with its bearing angle $\alpha_j$, true velocity components $\mathrm{v}_k^\mathrm{x}, \mathrm{v}_k^\mathrm{y}$ and radial velocity $\mathrm{v}_k^j$ ($k$ being the true generating object's index).
    }
    \label{fig: Dop Illust Combined}
    \vspace{-0.9em}
\end{figure*}

Doppler velocity measures the relative radial velocity of a detection point, namely the component of its actual velocity along the line of sight to the sensor.
Fig.~\ref{fig: Dop Illust Combined} illustrates this geometric effect.

To make the model concrete, we consider a 2D Cartesian sensing setting. Each object state $X_k$ contains planar velocity components $\mathrm{v}_k^\mathrm{x},\mathrm{v}_k^\mathrm{y}$ for the x and y axis, and we let $G$ extract them so that $G X_k = [\mathrm{v}_k^\mathrm{x},\mathrm{v}_k^\mathrm{y}]^\top$. For any positional measurement $y_j\in \mathbb{R}^2$, let $\alpha_j$ be its bearing angle relative to the sensor and define the line-of-sight unit vector $e_j:=[\cos\alpha_j,\sin\alpha_j]$. The true radial velocity of object $k$ at this bearing is therefore $\mathrm{v}_k^j= e_j G X_k$. A visual illustration of $\mathrm{v}_k^\mathrm{x}$, $\mathrm{v}_k^\mathrm{y}$, $\mathrm{v}_k^j$, $\alpha_j$ for a detection point $y_j$ is provided in Fig. \ref{fig: Dop Illust Combined}(b).

Since Doppler is naturally attached to each positional detection, we model each measurement $Y_j=(y_j,v_j)$ in a marked NHPP, where each positional measurement $y_j$ is equipped with a mark $v_j\in\mathbb{R}$. This retains the NHPP likelihood in \eqref{likelihood poisson existence}, with $\ell_0$ and $\ell$ defined below. Assuming the observed Doppler velocity $v_j$ is a Gaussian measurement of the true radial velocity $\mathrm{v}_k^j$, the single-object likelihood still admits a convenient linear Gaussian form:
\begin{align} \notag
\ell(Y_j\mid X_k,P_k)
&= \mathcal{N}(v_j; e_j G X_k,\sigma_v^2)\,
  \mathcal{N}(y_j; HX_k, P_k^{-1}),\\ \label{eq: Doppler likelihood}
  e_j&:=[\cos\alpha_j,\sin\alpha_j],
\end{align}
where $\sigma_v^2$ is the Doppler noise variance. As defined above, $\alpha_j$ is the bearing angle of $y_j$ (see Fig. \ref{fig: Dop Illust Combined}(b)), and $G$ and $H$ are the velocity and positional observation matrices, respectively. For instance, when $X_k= [\mathrm{p}_k^\mathrm{x}, \mathrm{v}_k^\mathrm{x},\mathrm{p}_k^\mathrm{y},\mathrm{v}_k^\mathrm{y}]^\top$ represents the position and velocity in the x and y axes, we have
\begin{align} \label{eq: G and H}
    G=\begin{bmatrix}
        0& 1 & 0 &0\\
        0 & 0 & 0 & 1
    \end{bmatrix}, 
    \qquad 
    H=\begin{bmatrix}
        1& 0& 0& 0\\
        0& 0& 1& 0
    \end{bmatrix}.
\end{align}
The bearing angle $\alpha_j$ in \eqref{eq: Doppler likelihood} can be computed from the known sensor position that receives $y_j$. For example, if the positional measurement $y_j=[y_j^\mathrm{x},y_j^\mathrm{y}]^\top$ and the corresponding sensor is located at $[s_j^\mathrm{x},s_j^\mathrm{y}]^\top$, then
$\alpha_j = \operatorname{atan2}(y_j^\mathrm{y}-s_j^\mathrm{y},\; y_j^\mathrm{x}-s_j^\mathrm{x})$. 

The clutter likelihood $\ell_0(Y_j)$ in \eqref{likelihood poisson existence} may be defined using a prelearned spatial–Doppler map or any suitable distribution over part or all of $Y_j=(y_j,v_j)$, analogous to Section~\ref{sec: measuremnt model}. For the Doppler component $v_j$, a common choice is a mixture of a uniform distribution over a prescribed velocity range and a normal distribution centred at zero, placing additional mass on stationary clutter. This completes the formulation of the Doppler-augmented Poisson measurement model. All joint priors and likelihood terms in \eqref{eq: likelihood joint}–\eqref{eq: association prior form 2} remain valid, with the single-object likelihood $\ell$ replaced by \eqref{eq: Doppler likelihood}.

Finally, we note that the convenient linear Gaussian single-object likelihood \eqref{eq: Doppler likelihood} in the Doppler-augmented Poisson measurement model relies on the assumption that all detection points from the same object share the same true velocity (see Fig. \ref{fig: Dop Illust Combined}(b)).
This is valid for a rigid object moving approximately in a straight line (including accelerated motion), but not for turning motion that induces velocity variation across the rigid body, or for nonrigid objects exhibiting micro Doppler effects. In such cases, the Doppler noise in \eqref{eq: Doppler likelihood} is expected to absorb part of the resulting model mismatch.
More accurate Doppler models that incorporate yaw rate as part of the object state exist \cite{knill2016direct,thormann2021incorporating} and can be extended to the NHPP setting, but these lose the convenient linear Gaussian form of \eqref{eq: Doppler likelihood}.
\subsection{Doppler-augmented PiVoT updates}
Since the two-stage inference framework and the corresponding updates in Sections \ref{sec: PiVoT challenges solution} to \ref{sec: stage 2} are formulated for the general NHPP likelihood in \eqref{likelihood poisson existence}, which includes the Doppler-augmented model in Section \ref{sec: Doppler model}, the derived update rules remain valid with only minor changes to the final expressions.

\subsubsection{Stage 1 updates} 
Stage 1 performs CAVI to iteratively update $q_1(X)$, $q_1(\Lambda)$, $q_1(P)$, and $q_1(\theta)$, which retain the same factorised structure and distributional forms as in \eqref{eq: all updates form}. 

The state update still follows the rule in \eqref{eq: general state update}, but with single-object likelihood $\ell$ now given by \eqref{eq: Doppler likelihood}. This yields the following independent state updates for each object $k$
\begin{align} \label{eq: Doppler state update}
q_1(X_k)\propto \hat{p}(X_k)\,
\mathcal{N}\!\left(\overbar{y}_k;HX_k,\overbar{R}_k\right)\,
\mathcal{N}\!\left(\overbar{v}_k;GX_k,\overbar{U}_k\right),
\end{align}
where $\hat{p}(X_k)$ is given in \eqref{eq: prior parameters}, and $\overbar{y}_k$ and $\overbar{R}_k$ are as defined in \eqref{eq: state update}. The matrices $H$ and $G$ are exemplified in \eqref{eq: G and H}. The Doppler related terms $\overbar{U}_k$ and $\overbar{v}_k$ take the form
\begin{align} \label{eq: pseudo measurements Dopp}
\overbar{U}_k
&=\sigma_v^2\Big(\sum\nolimits_{j=1}^M q_1(\theta_j=k)\,e_j^\top e_j\Big)^{-1},\\ \notag
\overbar{v}_k
&=\Big(\sum\nolimits_{j=1}^M q_1(\theta_j=k)\,e_j^\top e_j\Big)^{-1}
\sum\nolimits_{j=1}^M q_1(\theta_j=k)\,e_j^\top v_j,
\end{align}
where $\sigma_v^2, e_j$ are given in \eqref{eq: Doppler likelihood}. This update in \eqref{eq: Doppler state update} can be implemented by applying two standard Kalman filter updates in sequence, one to incorporate $\overbar{y}_k$ as in the positional update in \eqref{eq: state update}, and one to incorporate $\overbar{v}_k$ for the Doppler measurements.

The rate and shape updates remain unchanged: each $q_1(\Lambda_k)$ still follows \eqref{eq: rate update} and each $q_1(P_k)$ follows \eqref{eq: shape update}, as these do not involve the additional Doppler term $\mathcal{N}(v_j; e_j G X_k,\sigma_v^2)$ that augments the single-object likelihood $\ell$ in \eqref{eq: Doppler likelihood} relative to \eqref{measurement model}.

The data association still follows the independent updates in \eqref{eq: association update} for each $q_1(\theta_j)$. The clutter likelihood $\ell_0(Y_j)$ in \eqref{eq: association update} now incorporates the Doppler measurement. With the Doppler-augmented likelihood $\ell$ in \eqref{eq: Doppler likelihood}, the $S_j^k$ in \eqref{eq: association update} becomes
\begin{align} \label{eq: Sjk Doppler}
    S_j^k=\exp(T_j^k- \tfrac{1}{2\sigma_v^2} e_jG\Sigma_kG^\top e_j^\top)\mathcal{N}(v_j; e_j G \mu_k,\sigma_v^2),
\end{align}
where $T_j^k$ is defined in \eqref{eq: Tjk definition}; $v_j,\sigma_v^2,e_j,G$ are as in the Doppler likelihood \eqref{eq: Doppler likelihood}, and $\mu_k,\Sigma_k$ are variational parameters in \eqref{eq: all updates form}.

\subsubsection{Stage 2 updates} 
Recall that Stage 2 evaluates the existence and detectability distribution $q_2(E_k,D_k)$ for each object $k$, using the $q_1(X,\Lambda,P,\theta)$ obtained in Stage 1. All Stage 2 results in Sections \ref{sec: efficient optimal}, \ref{sec: numerical stable gD} depend on the single-object likelihood $\ell$ only through $S_j^k$. Therefore, the optimal $q_2^*(E_k,D_k), q_2^*(E_k)$ in \eqref{eq: existance update}, \eqref{eq: marginal existance calc}, the expression for $g(D_k)$ in \eqref{eq: simplified g}, \eqref{eq: Sj sum}, and its numerically stable and efficient evaluations in Appendix~\ref{apx: numerical stable gD}
all remain unchanged. We simply note that $S_j^k$ is now evaluated using its Doppler-augmented form in \eqref{eq: Sjk Doppler}.

\subsection{Properties of Doppler-augmented PiVoT} \label{sec: Property of Doppler PiVoT}
The Doppler-augmented PiVoT enables efficient detection and tracking using all received Doppler measurements through principled model-based inference. We highlight two advantages over common Doppler pruning and clustering heuristics.

First, PiVoT associates Doppler measurements using object shape, kinematics, and sensor geometry jointly through inference, rather than grouping similar Doppler values as in distance-based clustering methods \cite{scheiner2019multi, malzer2021constraint}. In practice, detections from the same object can exhibit markedly different Doppler velocities due to their relative geometry to the sensor, as seen from the arrows on the blue points in Fig.~\ref{fig: Dop Illust Combined}(a) and the green crosses in Fig.~\ref{fig: Dop Illust Combined}(b). Distance-based clustering may therefore incorrectly separate them, whereas PiVoT accounts for the geometric model and maintains a consistent association.

Second, for detection of moving objects, measurements with Doppler velocity close to zero are still useful and are fully exploited in PiVoT. As illustrated in Fig.~\ref{fig: Dop Illust Combined}(b), the purple points have almost zero radial velocity but arise from a moving object, and these measurements provide evidence for the existence of objects moving approximately perpendicular to the sensor’s line of sight. In contrast, a simple preprocessing step in radar pipelines is to prune such near-zero Doppler returns by treating them as stationary clutter.

\section{Algorithmic Realisation of PiVoT}
\label{sec: Algorithm structure}

\setlength{\textfloatsep}{5pt} % Reduce space before this algorithm
\begin{algorithm} [tp!]
% \SetAlgoLined
\algsetup{linenosize=\tiny}
  \scriptsize
\textbf{\textit{Prediction}}: Evaluate the predictive prior $\hat{p}_n$ in \eqref{eq: prior parameters} using the prediction step \eqref{eq: prediction}, with birth prior \eqref{eq: birth prior} set as in points~(i) and~(iii) of Section~\ref{sec: initialisation, clustering demo}.\\
\textbf{\textit{Stage 1 Initialisation}}: Initialise $q_1$ according to Appendix \ref{sec: CAVI initialisation}.\\
% \textit{Initialization for variational distribution:}\\
\textbf{\textit{Stage 1 CAVI updates}} \textit{(Section \ref{sec: ca updates})}:\\
\While {not converged}  
{

   For $k=1,2,...,K_n$, update kinematic states $q_1(X_{n,k})$ using \eqref{eq: general state update}.\\
   For $k=1,2,...,K_n$, update measurement rates $q_1(\Lambda_{n,k})$ using \eqref{eq: rate update}.\\
   For $k=1,2,...,K_n$, update object shapes $q_1(P_{n,k})$ using \eqref{eq: shape update}.\\
   For $j=1,2,...,M_n$, update data association $q_1(\theta_{n,j})$ using \eqref{eq: association update}.\\
   \textit{Ineffective birth removal (Section \ref{sec: ineffective birth})}:\\
   \If {$\sum_{j=1}^{M_n} q_1(\theta_{n,j}=k)<L$ for any  $k=K_{n-1}+1,...,K_n$} 
   {Remove components $\{X,\Lambda,P,E,D\}_{n,k}$ for all such $k$ from $\hat{p}_n$ and $q_1$, restrict each $\theta_{n,j}$ to the remaining indices, then update $K_n$ and relabel indices accordingly such that $k\in\{1,...,K_n\}$.}
}
\textbf{\textit{Stage 2}} \textit{(Section \ref{sec: stage 2})}: For $k=1,2,\ldots,K_n$, compute the globally optimal object existence and detectability distributions $q_2^*\!(E_{n,k},D_{n,k})$, $q_2^*\!(E_{n,k})$ using \eqref{eq: existance update}, \eqref{eq: marginal existance calc}, with $g(D_k)$ evaluated from \eqref{eq: gD1 computation}--\eqref{eq: numerical stable g0 k = Ij} in Appendix~\ref{apx: numerical stable gD}.\\
\vspace{0.3em}
\hrule
\vspace{0.3em}
 \caption{Generic PiVoT at time step $n$
}
 \label{Algo: PiVoT}

\footnotesize\textbf{Remark:} This generic algorithm covers both positional-only PiVoT and Doppler-augmented PiVoT. In the $q_1(X_{n,k})$ update, \eqref{eq: general state update} reduces to \eqref{eq: state update} and \eqref{eq: Doppler state update}, respectively. Similarly, the $S_j^k$ terms (used in \eqref{eq: association update} and \eqref{eq: gD1 computation}--\eqref{eq: gD0 computation}) reduce to \eqref{eq: Sjk positional only l} and \eqref{eq: Sjk Doppler}, respectively.

\end{algorithm}

Algorithm~\ref{Algo: PiVoT} summarises the generic PiVoT routine for both positional-only and Doppler-augmented models. This section gives implementation details for the prediction step in Algorithm~\ref{Algo: PiVoT}; Stage 1 initialisation is detailed in Appendix~\ref{apx: implem details}. We then describe the extraction of state estimates and a Doppler specific post-processing step for reporting confidently moving objects in automotive radar data.

\vspace{-1em}
\subsection{Prediction} \label{sec: prediction}
After obtaining $q_1(X_{n-1})$, $q_1(\Lambda_{n-1})$, $q_1(P_{n-1})$ in \eqref{eq: all updates form} from Stage 1, and $q_2^*(E_{n-1,k})\!=\!\text{Ber}(p_{n-1,k}^e)$ in \eqref{eq: marginal existance calc} from Stage 2 at time step $n-1$, the prediction step for parameters in \eqref{eq: prior parameters} at time step $n$ for legacy object $k=1,...,K_{n-1}$ is:
\vspace{-0.5em}
\begin{align} \label{eq: prediction}
\begin{aligned}
    &\mu_{n,k}^\prime=F_n\mu_{n-1,k}, \qquad \Sigma_{n,k}^\prime=F_n\Sigma_{n-1,k}F_n^\top+Q_n,\\
    &\eta_{n,k}^\prime=\eta_{n-1,k}\gamma_{\Lambda,n}, \qquad \rho_{n,k}^\prime=\rho_{n-1,k}/\gamma_{\Lambda,n},\\
    &\phi_{n,k}^\prime=\max\{\phi_{n-1,k}\gamma_{P,n},\ d_Y-1\}, \\
    &\Phi_{n,k}^\prime = \Phi_{n-1,k}\phi_{n-1,k}/\phi_{n,k}^\prime, \qquad p_{n,k}^{e\prime}=p_{n-1,k}^e p_{n,k}^s, 
\end{aligned}
\end{align}
where $F_n$, $Q_n$ are from the state transition in \eqref{eq: linear Gaussian transition}, $\gamma_{\Lambda,n},\gamma_{P,n}\leq 1$ are forgetting factors for $\Lambda$ and $P$, and $p_{n,k}^s\in(0,1]$ is the survival probability defined in Section~\ref{sec: legacy object model}. The state and existence predictions follow from the linear Gaussian transition and the survival process defined in Section~\ref{sec: legacy object model}. The rate and shape predictions preserve the mean while inflating uncertainty via the forgetting factors. Alternative prediction models for $\Lambda$ and $P$ are also available, e.g. \cite{granstrom2019poisson,koch2008bayesian}. For a moving platform, such as automotive radar on an ego vehicle, measurements and object states may need to be transformed between time steps to account for ego motion.

\subsection{Estimation extraction} \label{sec: estimation extraction}
To prevent unbounded growth of $K_n$, a pruning step is applied at each time step, removing objects with existence probability $q_2^*(E_{n,k}) < P_{\text{pru}}$ from the posterior. This is necessary because, in principle, $K_n$ can grow to infinity over time, as every potential object born at previous time steps, including those with zero existence posterior, remains in the system and continues to be processed.

For reporting a reliable estimation output at the current time step, PiVoT introduces another two thresholds: when reporting an unreported object for the first time, we require $q_2^*(E_{n,k}) > P_{\text{rep}}$ to select only convincing estimates from heavy clutter; previously reported objects will not be reported in the current time step if $q_2^*(E_{n,k}) < P_{\text{stp}}$. It is important to note that the thresholds $P_{\text{rep}},P_{\text{stp}}$ affect only the final selection of estimates from the set of all processed objects. PiVoT still performs variational inference over the full set of potential objects with the same computational efficiency, regardless of the choice of reporting thresholds.

\begin{figure}[t!]
    \centering
    \includegraphics[width=1\linewidth]{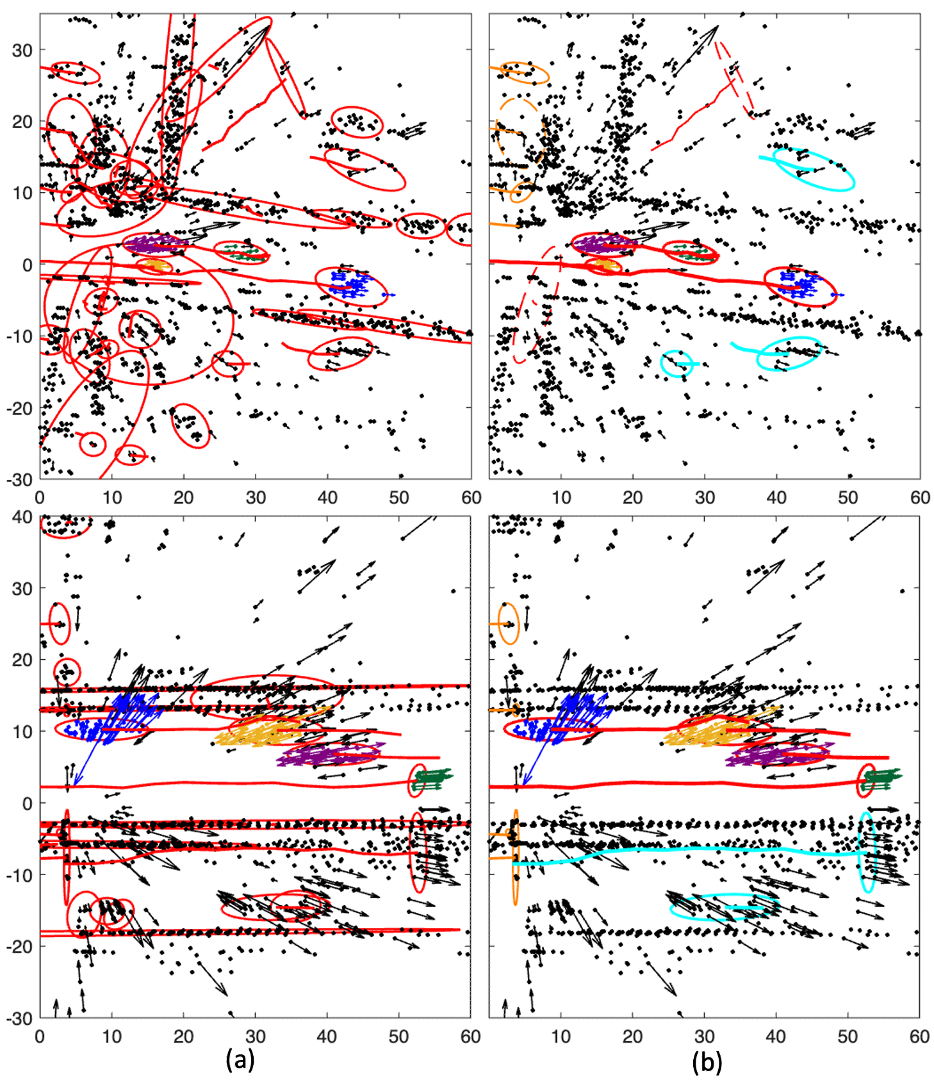}
    \caption{
    PiVoT tracking results on the RadarScenes \cite{schumann2021radarscenes} dataset (a) before and (b) after moving object post-processing. The ego vehicle is located at $(0,0)$ and moves to the right. Black and coloured dots/arrows denote measurements from ground truth stationary background and moving objects, respectively. Red trajectories and ellipses in (a) show all estimated tracks and shapes, including stationary and moving objects.
    All objects shown in (b) are regarded as moving. Among those,
    orange, cyan, and dashed ellipses indicate tracks rejected due to Doppler uninformativeness, ghost suppression, and the concrete object shape constraint, respectively.
    }
    \label{fig: Dop post processing}
\end{figure}
\begin{figure*}[t!]
\centerline{\includegraphics[width=18cm]{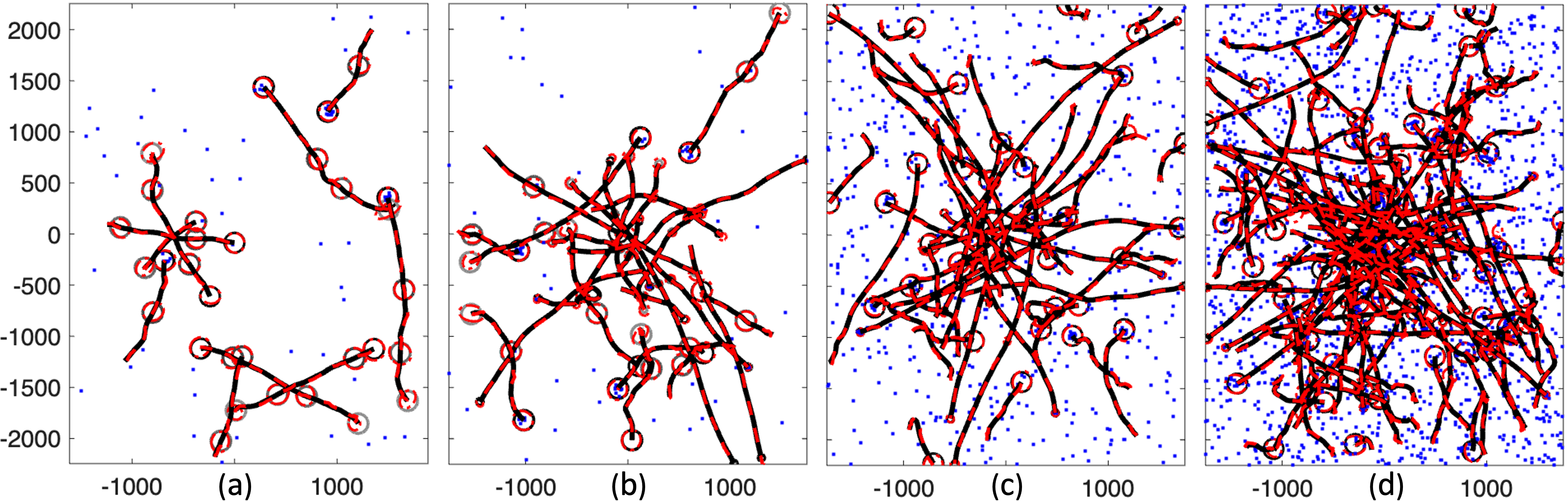}}
\caption{
True tracks (black), shapes, one-run PiVoT estimates (red dashed tracks and ellipses), and final-time-step measurements for trajectory/dataset pairs (a) T1/DS1, (b) T2/DS3, (c) T3/DS5, and (d) T4/DS6. True shapes (grey-to-black for earlier-to-later time steps) and shape estimates are shown at multiple steps in (a)--(b), and only at the final step in (c)--(d).
}
    \label{fig: gtdatset}
    \vspace{-0.2em}
\end{figure*}

\begin{table*}[tp!]
\centering
\caption{Performance comparison on DS1--DS6. Time and GOSPA are averaged over all time steps and 50 Monte Carlo runs\textsuperscript{\dag}. Each dataset header gives the trajectory set, detection probability $p_k^d$, and clutter rate $\Lambda_0$. Object numbers, tracks, shapes, and measurement rates are defined by T1--T4, see Section~\ref{sec: simulation setting}. DS2 additionally sets the T1 measurement rate to $\Lambda_k=5$.}
\vspace{-0.5em}
\setlength{\tabcolsep}{2.3pt}
\begin{tabular}{ c | c c c c c | c c c c c | c c c c c}
\toprule
Method 
& Time ($\mathrm{s}$) & GOSPA & Location & Missed & False
& Time ($\mathrm{s}$) & GOSPA & Location & Missed & False
& Time ($\mathrm{s}$) & GOSPA & Location & Missed & False \\
\midrule
& \multicolumn{5}{c|}{DS1: T1, $p_{k}^d=1$, $\Lambda_0=47.25$} 
& \multicolumn{5}{c|}{DS2: T1 with $\Lambda_k=5$, $p_{k}^d=0.9$, $\Lambda_0=47.25$}
& \multicolumn{5}{c}{DS3: T2, $p_{k}^d=0.9$, $\Lambda_0=47.25$} \\
\midrule
PiVoT 
& \textbf{0.0077} & \textbf{82.8745} &  \textbf{79.6145} & 1.3000 & 1.9600
& \textbf{0.0106} & \textbf{126.4433} & \textbf{105.1433} & \textbf{8.4400} & 12.860
& \textbf{0.0091} & \textbf{191.1176} & \textbf{159.5576} & \textbf{16.860} & \textbf{14.700} \\

$\mathrm{PMBM}_{\mathrm{F}}$ 
& 0.4747 & 108.0524 & 89.4524  & 17.200 & 1.4000
& 0.6243 & 141.2471 & 114.7471 & 12.500 & 14.000 
& 1.6159 & 253.0662 & 190.9662 & 17.900 & 44.200 \\
$\mathrm{PMBM}_{\mathrm{C}}$ 
& 0.1837 & 109.6324 & 90.4324  & 18.400 & 0.8000 
& 0.2927 & 141.5576 & 114.7576 & 13.000 & 13.800 
& 0.6221 & 261.1853 & 193.6853 & 19.900 & 47.600 \\
SPA$_{\mathrm{3000}}$ 
& 1.6411 & 133.4078 & 118.2878 & 8.1800 & 6.9400
& 1.3807 & 170.7621 & 140.6221 & 16.260 & 13.880
& 5.8976 & 463.1671 & 252.0071 & 160.14 & 51.020 \\

SPA$_{15000}$ 
& 7.3063 & 93.0783 & 91.5983 & \textbf{0.8800} & \textbf{0.6000}
& 6.0615 & 134.4843 & 118.8843 & 8.6000 &  \textbf{7.0000}
&  -- & -- & -- & -- & -- \\

\midrule
& \multicolumn{5}{c|}{DS4: T2, $p_{k}^d=1$, $\Lambda_0=1575$} 
& \multicolumn{5}{c|}{DS5: T3, $p_{k}^d=1$, $\Lambda_0=472.5$}
& \multicolumn{5}{c}{DS6: T4, $p_{k}^d=1$, $\Lambda_0=1575$} \\
\midrule
PiVoT 
& \textbf{0.2013} & \textbf{223.3302} & \textbf{168.1702} & \textbf{26.540} & 28.620
& \textbf{0.0550} & \textbf{395.5997} & \textbf{361.5197} & \textbf{24.460} & \textbf{9.6200}
& \textbf{0.3008} & \textbf{1000.147} & \textbf{817.8874} & \textbf{160.08} & \textbf{22.180} \\

$\mathrm{PMBM}_{\mathrm{C}}$ 
& 9.3477 & 355.4586 & 208.2586 & 120.90 & \textbf{26.300}
& 4.7251 & 621.1152 & 458.9152 & 115.80 & 46.400
& 39.983 & 1931.305 & 883.2056 & 925.00 & 123.10 \\
SPA$_{100}$ 
& 143.44 & 869.7601 & 325.1601 & 445.70 & 98.900
& 16.638 & 1930.081 & 774.9614 & 782.28 & 372.84
& 216.10 & 3537.380 & 1396.480 & 1655.7 & 485.20 \\

\bottomrule 
\end{tabular}
\label{tb:results}
\vspace{0.2em}
{\footnotesize
\textsuperscript{\dag} The only exceptions are that SPA$_{100}$ uses 10 runs on DS4 and DS6, and SPA$_{15000}$ is omitted on DS3 due to excessive runtime.
}
\vspace{-1em}
\end{table*}

\subsection{Post-processing for moving object selection} \label{sec: post processing}

PiVoT processes the full measurement stream to jointly detect and track all persistent structures, including genuinely moving objects, stationary infrastructure, and clutter induced artefacts. For the automotive radar experiments, we therefore apply a lightweight deterministic post-processing stage to report a reliable subset of moving objects. This stage operates only on inferred track statistics and does not modify the variational inference. In brief, tracks are retained only when they show 1) consistent kinematic motion over a short temporal window, 2) informative Doppler support along the estimated direction of motion when required to disambiguate motion, 3) plausible road user shape and sufficient dominance of their associated detections within the estimated extent, and 4) no evidence of being a multipath ghost induced by stationary structures. Fig.~\ref{fig: Dop post processing} illustrates how these checks refine the full PiVoT output into the final reported moving object set.

This automotive radar specific stage is used only to report a reliable and interpretable subset of moving road users from the richer set of PiVoT tracks, without changing inference; further procedural details are therefore given in Appendix~\ref{apx: doppler post processing}.

\section{Experiments } \label{sec: results}

This section first evaluates PiVoT against existing model-based trackers on various simulated tracking scenarios in Section \ref{sec: comparison simulation}, followed by a large-scale simulation demo tracking 1000 objects in Section~\ref{sec: large scale demo}. Finally, Section~\ref{sec: evaluation RadarScenes} compares PiVoT with a deep learning detection benchmark on a real-world automotive radar dataset.

\vspace{-0.5em}
\subsection{Evaluation on simulated scenes} \label{sec: comparison simulation}
This section compares PiVoT's detection and tracking performance with existing Bayesian NHPP-based trackers, including the Poisson multi-Bernoulli mixture (PMBM) filter~\cite{granstrom2019poisson} and the sum-product algorithm (SPA)~\cite{meyer2021scalable}. 
Since these methods do not yet accommodate Doppler measurements, we restrict this comparison to positional NHPP model in \eqref{measurement model}. The evaluation uses 6 simulated datasets of increasing complexity.

\subsubsection{Experimental setting} \label{sec: simulation setting} The six datasets are built from four trajectory sets, T1--T4, shown in Fig.~\ref{fig: gtdatset}. 
We use two object types: small objects with $(\Lambda_k,P_{n,k}^{-1})=(5,100I_2)$ and large objects with $(\Lambda_k,P_{n,k}^{-1})=(8,800I_2)$. 
T1 contains 4--10 randomly appearing and disappearing large objects. 
T2 contains 11--25 objects, with 60\% small and 40\% large. 
T3 and T4 contain 21--50 and 42--99 objects, respectively, both with 30\% small and 70\% large.
Each trajectory set has 50 time steps with sampling interval $\tau=1\,\mathrm{s}$. 
Objects follow a 2D constant-velocity model with positional NHPP measurements. In \eqref{measurement model} and \eqref{eq: linear Gaussian transition}, we set $F_n \!=\! \mathrm{diag}(F_n^1, F_n^2)$, $Q_n \!=\! \mathrm{diag}(Q_n^1, Q_n^2)$, and $H_n \!=\! \mathrm{diag}(H_n^1, H_n^2)$, where for each $d = 1, 2$,
\begin{equation}\label{eq:cv model}
    F_{n}^d=\begin{bmatrix} 1 & \tau \\ 0& 1
    \end{bmatrix}, \quad
    Q_{n}^d=25\begin{bmatrix} \tau^3/3 & \tau^2/2 \\ \tau^2/2& \tau\end{bmatrix}, \quad
    H^d=\begin{bmatrix} 1 & 0 \end{bmatrix}.
\end{equation}

For each dataset DS1--DS6, the true tracks are fixed from one of T1--T4, while 50 independent measurement sets are generated from the model, each over 50 time steps. The detailed configurations are given in Table~\ref{tb:results}. DS1--DS2 contain only large objects, and DS3--DS6 contain both object types.

The detailed parameter settings for all compared methods are provided in Appendix~\ref{apx: parameterasations for exp}. In summary, PiVoT uses uninformative parameterisations across DS1--DS6 by design, to assess robustness across diverse scenes. PMBM and SPA are configured more informatively to improve performance; 
for example, in DS1--DS2, their shape priors concentrate more on the single object type, while PiVoT uses uninformative priors spanning both types.  
SPA also uses the ground truth measurement rate, since it is not estimated in \cite{meyer2021scalable}.
In Table \ref{tb:results}, PMBM$_\mathrm{C}$ and PMBM$_\mathrm{F}$ use 16 and 100 DBSCAN distance values, respectively, evenly spaced over $[1,100]$, a range found to give strong performance.
The subscript in SPA denotes the number of particles, e.g., SPA$_{3000}$ uses 3000 particles. 

All methods are implemented in MATLAB, with runtime measured as the average execution time, in seconds, per time step (System: Apple M1, 16GB RAM).
Tracking performance is evaluated using the GOSPA metric~\cite{rahmathullah2017generalized}, a standard metric in multi-object tracking literature, with cut-off distance $100$, penalty parameter $1$, and $\alpha=2$. The total GOSPA error decomposes into localisation error for matched objects, missed-object, and false-object errors, as reported in Table \ref{tb:results}.

\subsubsection{Results} The results in Table~\ref{tb:results} establish PiVoT as a promising advancement in multi-object tracking, in both accuracy and efficiency. Across all datasets, PiVoT consistently achieves the lowest GOSPA scores and the fastest runtime. Its advantages are especially clear in DS4--DS6, which involve heavy clutter and/or large object numbers. In these challenging cases, PiVoT achieves substantially higher tracking accuracy while being around 50–100 times faster than the second-best method. 
In the simpler DS1--DS2 scenarios, SPA achieves lower missed object and false object components, but remains second to PiVoT in overall GOSPA and requires more than 500 times the computation time.
As the tracking challenge intensifies in DS4--DS6, SPA deteriorates substantially. PMBM then becomes the second-best method, 
although its accuracy also degrades in highly cluttered scenes with many objects.
In contrast, PiVoT maintains robust tracking performance, with GOSPA values increasing roughly proportionally to the object number. 
SPA's deterioration in DS4--DS6 may reflect its sensitivity to uncertainty in the shape prior as two object types are present: it is observed that SPA improves clearly when the shape prior is sufficiently concentrated around the true shape.

PiVoT's tracking and shape estimates are visualised in Fig.~\ref{fig: gtdatset} for representative datasets. The estimated tracks and shapes closely follow the ground truth, with few missed and false detections, even under the heavy clutter and frequent coalescence in scenes (c) and (d). The initially imperfect shape estimates in Fig.~\ref{fig: gtdatset}(a) and (b) are due to the uninformative prior and limited early observations, but are soon corrected to reasonable accuracy. In Fig.~\ref{fig: gtdatset}(d), and more generally in DS6, severe clutter makes true objects with sparse measurements difficult to distinguish by human inspection. A video demonstrating PiVoT’s accurate detection and tracking in this challenging setting is provided in \cite{pivotweb}.

\vspace{-0.5em}
\subsection{Demo of large-scale multi-object tracking in clutter} \label{sec: large scale demo}

This section demonstrates PiVoT's capability to detect and track 1000 objects in clutter. Since running PMBM and SPA in this large-scale setting would be computationally impractical, we report only PiVoT results as a scalability demonstration.

We consider a densely populated tracking environment with up to 1036 objects appearing and disappearing over 50 time steps. Objects follow a constant-velocity motion model with parameters in \eqref{eq:cv model} and the true tracks are shown in Fig. \ref{fig:estimate track}. Object elliptical shape covariance is $P_{n,k}^{-1}=50\,I$, and the measurement rates for all objects are $5$. Detection probability is $p_{n,k}^d=1$. The clutter density is $10^{-6}$ per unit area, resulting in a clutter rate of $15.75$. The object number varies as in Fig.~\ref{fig:objec number}. The challenge of this case is efficiently handling a large volume of data while managing many birth and death events. Further, the close proximity of objects leads to frequent coalescence, making data association particularly difficult. 

\begin{figure}[tp!]
    \centering
    \includegraphics[width=8cm]{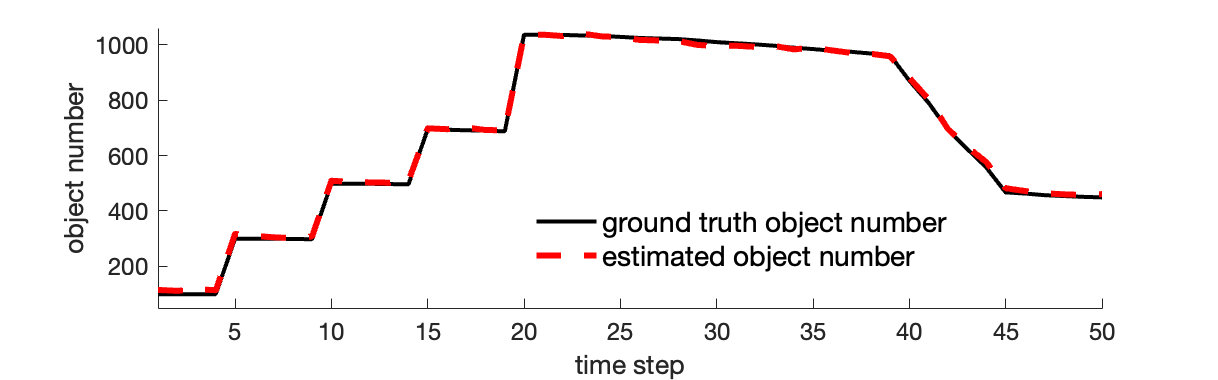}
    \caption{True and estimated object number.}
    \label{fig:objec number}
    \vspace{-1em}
\end{figure}

\begin{figure}[tp!]
    \centering
\includegraphics[width=8.5cm]{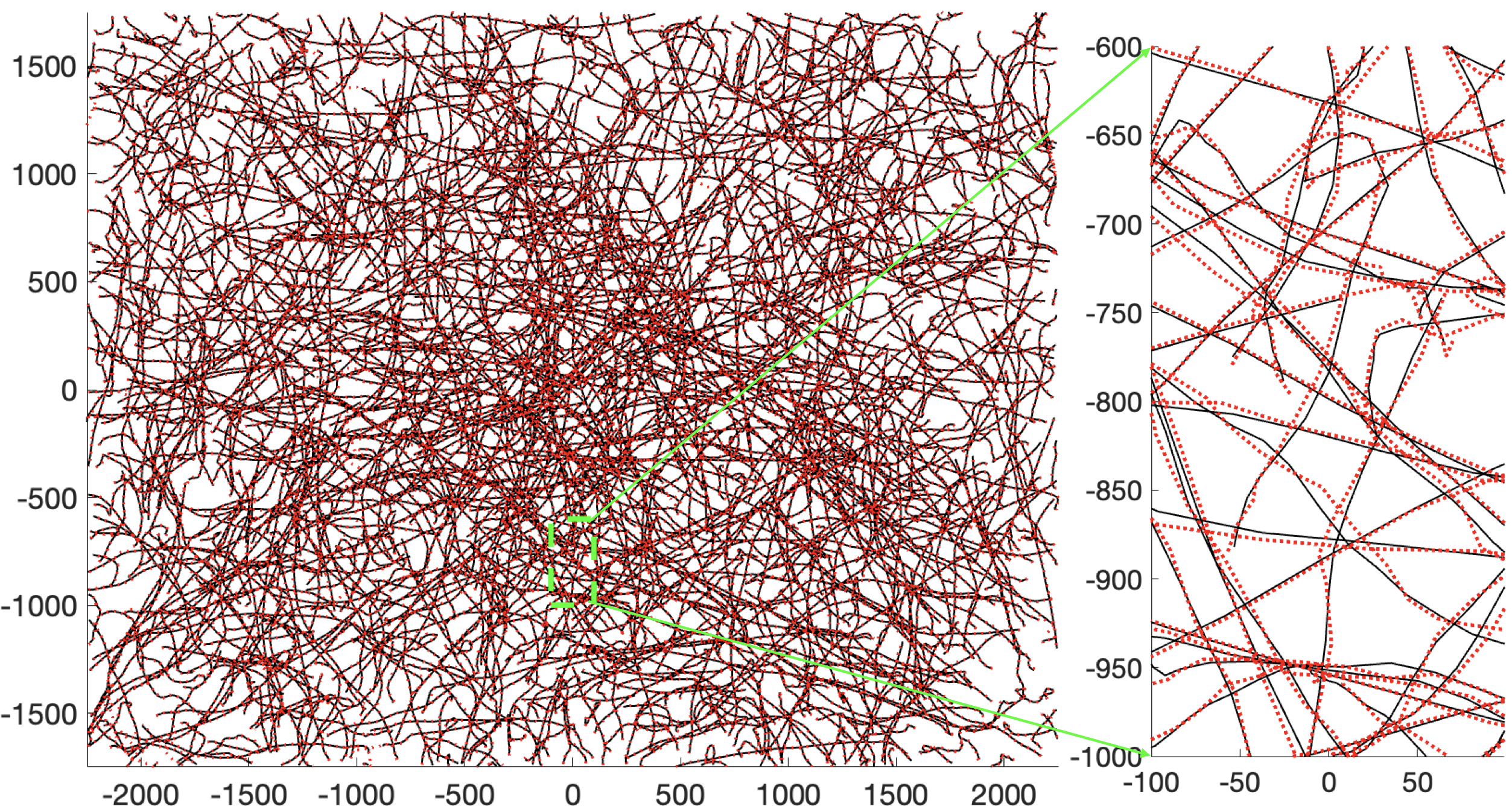}
    \caption{Tracking results for up to 1036 objects. Black curves are true tracks and red dotted curves are PiVoT estimates. The right panel enlarges the region marked by the green rectangle.}
    \label{fig:estimate track}
    \vspace{-0.5em}
\end{figure}

PiVoT uses the same parameter setting as in Section \ref{sec: comparison simulation}, except that the ground-truth object rates and shapes are assumed known. Fig. \ref{fig:objec number} and Fig.~\ref{fig:estimate track} present the estimated trajectories and estimated object counts over time alongside the ground truth, demonstrating PiVoT’s ability to track a thousand objects simultaneously, with timely detection of births and deaths and no excessive track loss. Furthermore, PiVoT achieves this with remarkable computational efficiency, requiring only $0.67\,\mathrm{s}$ per time step on average to track over 1000 objects in MATLAB 2024b (System: Apple M4 Pro, 24GB RAM). This implementation uses only vectorised operations, without explicit parallelisation or any gating technique. 

A video demonstration of this large-scale scene is provided in \cite{pivotweb}. Given its performance, further speed improvements could be expected through dedicated CPU/GPU parallelisation and gating over a large surveillance area, making PiVoT a highly promising solution for real-time multi-object tracking in large-scale, cluttered environments.

\subsection{Evaluation on RadarScenes dataset} \label{sec: evaluation RadarScenes}

\begin{table*}[tp!]
\centering
\caption{RadarScenes results on the validation and test splits. Time is measured per $1\,\mathrm{s}$ of radar data. All object matching and GOSPA metrics are macro-averaged over frames, except AP$_{\mathrm{global}}$, which pools object estimates over all frames.}
\vspace{-0.5em}
\footnotesize
\setlength{\tabcolsep}{2.1pt}
\renewcommand{\arraystretch}{1}
\begin{tabular}{c|c|c|ccccc|ccccc|cccc}
\toprule
\multirow{2}{*}{Split} 
& \multirow{2}{*}{Method}
& \multirow{2}{*}{\begin{tabular}{@{}c@{}}Time\\($\mathrm{s}$)\end{tabular}}
& \multicolumn{5}{c|}{Object matching, IoU $0.3$}
& \multicolumn{5}{c|}{Object matching, IoU $0.5$}
& \multicolumn{4}{c}{GOSPA error} \\[1pt]
\cline{4-8}\cline{9-13}\cline{14-17}
& &
& \rule{0pt}{2.4ex} F1 & Precision & Recall & AP$_{\mathrm{macro}}$ & AP$_{\mathrm{global}}$
& F1 & Precision & Recall & AP$_{\mathrm{macro}}$ & AP$_{\mathrm{global}}$
& GOSPA & Location & Missed & False \\
\midrule
\multirow{4}{*}{Val.}
& RadarGNN$_{\mathrm{O}}$ & 4.895
& 0.6943 & 0.6383 & 0.8259 & 0.7729 & \textbf{0.7482}
& 0.6017 & 0.5551 & 0.7096 & 0.6458 & 0.5901
& 4.9084 & 1.5588 & \textbf{0.7112} & 2.6384 \\

& RadarGNN$_{\mathrm{B}}$ & 4.899
& 0.7407 & 0.7085 & \textbf{0.8279} & 0.7797 & 0.7415
& 0.6447 & 0.6190 & 0.7144 & 0.6550 & 0.5862
& 4.2197 & 1.5153 & 0.8098 & 1.8946 \\

& RadarGNN$_{\mathrm{P}}$ & 4.896
& 0.7886 & 0.8291 & 0.7897 & 0.7585 & 0.6916
& 0.6931 & 0.7309 & 0.6892 & 0.6456 & 0.5535
& 3.4500 & 1.3384 & 1.2850 & 0.8266 \\

& PiVoT & \textbf{0.830}
& \textbf{0.8016} & \textbf{0.8547} & 0.7993 & \textbf{0.7825} & 0.7106
& \textbf{0.7829} & \textbf{0.8313} & \textbf{0.7815} & \textbf{0.7597} & \textbf{0.6739}
& \textbf{2.8114} & \textbf{0.8704} & 1.2400 & \textbf{0.7009} \\

\midrule
\multirow{4}{*}{Test}
& RadarGNN$_{\mathrm{O}}$ & 5.135
& 0.6643 & 0.6157 & \textbf{0.7826} & 0.7211 & \textbf{0.6921}
& 0.5688 & 0.5282 & 0.6669 & 0.5964 & 0.5316
& 5.6744 & 1.7170 & \textbf{0.9971} & 2.9603 \\

& RadarGNN$_{\mathrm{B}}$ & 5.139
& 0.7109 & 0.6884 & \textbf{0.7826} & 0.7302 & 0.6827
& 0.6117 & 0.5933 & 0.6706 & 0.6074 & 0.5255
& 4.8692 & 1.6411 & 1.1585 & 2.0697 \\

& RadarGNN$_{\mathrm{P}}$ & 5.133
& 0.7538 & 0.8103 & 0.7396 & 0.7072 & 0.6263
& 0.6571 & 0.7050 & 0.6438 & 0.5992 & 0.4907
& 4.0263 & 1.3731 & 1.8314 & \textbf{0.8218} \\

& PiVoT & \textbf{0.861}
& \textbf{0.7723} & \textbf{0.8372} & 0.7573 & \textbf{0.7357} & 0.6748
& \textbf{0.7554} & \textbf{0.8186} & \textbf{0.7407} & \textbf{0.7131} & \textbf{0.6469}
& \textbf{3.6252} & \textbf{0.9633} & 1.7185 & 0.9434 \\
\bottomrule
\end{tabular}
\vspace{-0.0em}
\label{tb:radarscenes_all_metrics}
\end{table*}

\begin{table}[tp!]
\centering
\caption{Point-wise moving foreground segmentation results on RadarScenes. Macro metrics are averaged over frames, while global metrics are computed by pooling all points.}
\vspace{-0.5em}
\footnotesize
\setlength{\tabcolsep}{2.3pt}
\renewcommand{\arraystretch}{1.0}
\begin{tabular}{c | c | c c c | c c c}
\toprule
\multirow{2}{*}{Split}
& \multirow{2}{*}{Method}
& \multicolumn{3}{c|}{Macro}
& \multicolumn{3}{c}{Global} \\[1pt]
\cline{3-8}
& & \rule{0pt}{2.2ex} F1 & Precision & Recall & F1 & Precision & Recall \\
\midrule
\multirow{2}{*}{Val.}
& RadarGNN
& 0.8258 & 0.7710 & \textbf{0.9108}
& \textbf{0.8887} & 0.8250 & \textbf{0.9631} \\
& PiVoT
& \textbf{0.8551} & \textbf{0.8824} & 0.8541
& 0.8785 & \textbf{0.8880} & 0.8692 \\
\midrule
\multirow{2}{*}{Test}
& RadarGNN
& 0.7861 & 0.7311 & \textbf{0.8801}
& \textbf{0.8607} & 0.7980 & \textbf{0.9342} \\
& PiVoT
& \textbf{0.8119} & \textbf{0.8500} & 0.8076
& 0.8410 & \textbf{0.8516} & 0.8307 \\
\bottomrule
\end{tabular}
\vspace{-0.0em}
\label{tb:radarscenes_pointwise}
\end{table}

To demonstrate PiVoT’s capability to operate on real-world automotive radar data with Doppler information, we evaluate it on RadarScenes \cite{schumann2021radarscenes}, a dataset with typical radar clutter and diverse road users.
RadarScenes is widely used for assessing machine learning methods for automotive radar-based multi-object detection.
We therefore compare with RadarGNN \cite{fent2023radargnn}, a strong deep learning baseline on this dataset with a publicly available implementation and released trained checkpoints.

We note that the two methods differ in their design goals, temporal aggregation windows, and input features.
PiVoT is evaluated in its full tracking configuration with moving object selection (Section \ref{sec: post processing}), where each frame accumulates radar data over $0.25\,\mathrm{s}$. This shorter frame length reduces latency and enables more timely responses in real-time tracking. By contrast, RadarGNN uses frames accumulated over $0.5\,\mathrm{s}$ and performs frame-wise detection without temporal state propagation, consistent with its original design and training in \cite{fent2023radargnn}. The input information also differs: PiVoT currently uses only positional and Doppler measurements, whereas RadarGNN additionally uses radar cross-section information. RadarGNN’s object classification capability is not considered here, since this experiment focuses on the object detection and tracking performance across all moving objects.

We emphasise that this comparison is not intended to claim state-of-the-art performance on RadarScenes, but to demonstrate that a fully training-free, model-based joint detector and tracker can achieve efficient and competitive performance on a full-scale modern automotive radar benchmark. To our knowledge, PiVoT is the first such method to report quantitative results comparable to a deep learning detection benchmark.

\begin{figure*}[tp!] 
\centerline{\includegraphics[width=18cm]{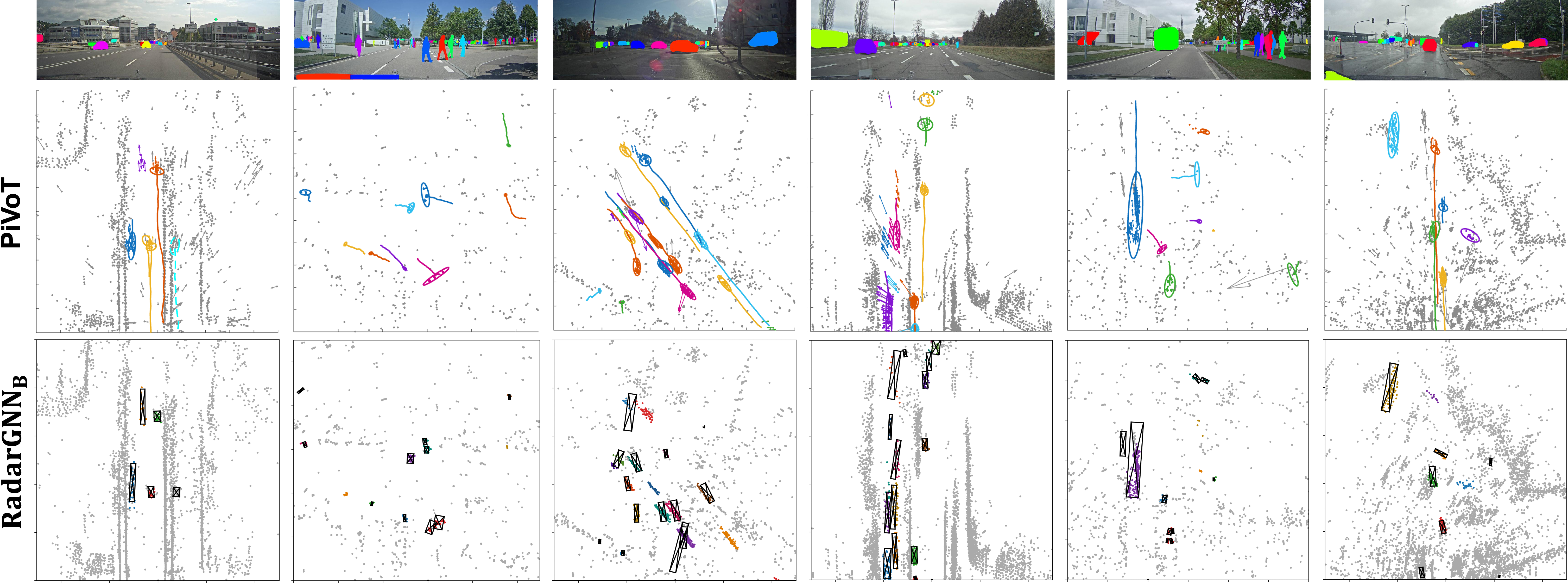}}
    \caption{Visualisation on representative RadarScenes frames. The results are matched to the camera images. Coloured points denote ground-truth moving-object measurements, with colours indicating object identities. RadarGNN$_{\mathrm{B}}$ contains more radar points from longer accumulation. PiVoT shows estimated shapes and tracks using ellipses and curves, while RadarGNN$_{\mathrm{B}}$ shows detections using bounding boxes. Dashed cyan ellipses and curves in PiVoT's first column indicate a ghost target.
    }
    \vspace{-0.9em}
    \label{fig: AutoRadar Compr}
\end{figure*}

\subsubsection{Experimental setting}
We evaluate all methods on the RadarScenes validation and test splits, following \cite{fent2023radargnn}. For RadarGNN, we use the best-performing translation-invariant model reported in \cite{fent2023radargnn} with the released checkpoints trained on the RadarScenes training split. Besides the original post-processing configuration, denoted by RadarGNN$_{\mathrm{O}}$, we consider two more configurations. Since RadarGNN$_{\mathrm{O}}$ tends to produce many false alarms, we tune the confidence thresholds used to remove background points and objects to obtain a balanced precision-recall configuration, denoted by RadarGNN$_{\mathrm{B}}$, and a precision-oriented configuration, denoted by RadarGNN$_{\mathrm{P}}$. These configurations are selected based on detection metrics evaluated on the validation split. For PiVoT, the parameterisation is selected only through visual inspection of tracking videos from a few sequences in the validation split. All runtime evaluations are conducted on a standard laptop with a 2.4 GHz 8-core Intel Core i9 CPU and 32GB RAM. PiVoT is implemented in MATLAB, while RadarGNN is run in Python through Docker with access to 7 of the 8 CPU cores.

Object-level detection and tracking results are reported in Table \ref{tb:radarscenes_all_metrics}. We use object-matching metrics, including precision, recall, F1 score, and average precision (AP), under point-wise IoU thresholds of 0.3 and 0.5 as in \cite{scheiner2021object}, together with GOSPA metrics \cite{rahmathullah2017generalized} using cut-off distance $3\,\mathrm{m}$, penalty parameter $1$, and $\alpha=2$. For GOSPA, unavailable ground-truth object positions are computed by averaging the corresponding ground-truth points. Table \ref{tb:radarscenes_pointwise} further reports point-wise moving-foreground segmentation performance, where all RadarGNN configurations share the same point-wise results.

\subsubsection{Computational time}
Table \ref{tb:radarscenes_all_metrics} reports the average processing time per $1\,\mathrm{s}$ radar data interval. 
For RadarGNN, the full inference chain takes $4.897\,\mathrm{s}$ on the Val. split and $5.136\,\mathrm{s}$ on the Test split on average.
In contrast, PiVoT takes $0.830\,\mathrm{s}$ and $0.861\,\mathrm{s}$, respectively, both below the $1\,\mathrm{s}$ data interval, satisfying the real-time processing requirement on average and running much faster than RadarGNN under the evaluated CPU setting. We note that RadarGNN is expected to run much faster on a GPU, while PiVoT may also be further accelerated.

\subsubsection{Detection and Tracking Results}

From Table \ref{tb:radarscenes_all_metrics}, PiVoT achieves the best frame-wise macro-averaged overall performance across all RadarGNN configurations, including F1 score, AP$_{\mathrm{macro}}$, and GOSPA. This improvement is mainly due to its higher precision and, similarly, fewer false alarms, which reflects PiVoT's ability to recover true moving objects from dense clutter. This precision advantage remains even compared with RadarGNN$_{\mathrm{P}}$, whose thresholds are deliberately tuned towards higher precision at the cost of recall, with recall already lower than that of PiVoT.

By contrast, RadarGNN often achieves higher AP$_{\mathrm{global}}$ and recall, and similarly fewer missed detections in Table \ref{tb:radarscenes_all_metrics}. This is mainly because PiVoT uses a conservative moving-object selection rule, which requires moving evidence over two to three frames before declaring a moving object. As shown in columns $1, 3$, and $4$ of Fig. \ref{fig: AutoRadar Compr}, this can cause PiVoT to miss objects that have only recently appeared in the field of view.

The same trend is also reflected in Table \ref{tb:radarscenes_pointwise}, which reports point-wise moving-foreground segmentation between moving-object and background measurements. PiVoT is stronger in precision and macro-averaged F1, whereas RadarGNN has an advantage in recall and global F1. An exception appears in Table \ref{tb:radarscenes_all_metrics} under the stricter IoU threshold of $0.5$, where RadarGNN performs worse than PiVoT across all object-level metrics. This is mainly because RadarGNN bounding boxes can miss many ground-truth points and sometimes split the same object into multiple boxes. Both effects are visible in Fig. \ref{fig: AutoRadar Compr}. The third column particularly illustrates poor box alignment when the object is diagonally oriented with respect to the coordinate axes, possibly due to limited representation of such orientations in the training data.

Fig. \ref{fig: AutoRadar Compr} further supports the preceding analysis through representative visual comparisons between RadarGNN${_\mathrm{B}}$ and PiVoT. In particular, 
PiVoT produces fewer false alarms and better shape estimates in terms of measurement inclusion. Its missed detections mainly occur for newly appearing objects or objects with abnormal measurements. For example, in column $4$, the missed object has Doppler arrows inconsistent with its motion, possibly due to upstream radar signal-processing artefacts. In column $1$, PiVoT also correctly suppresses a ghost target that RadarGNN$_{\mathrm{B}}$ falsely detects as a moving object, highlighting the benefit of retaining stationary structures as evidence for ghost-target rejection (see also Fig. \ref{fig: Dop post processing}).

We note that PiVoT also provides multi-object tracks directly, as shown in Fig.~\ref{fig: AutoRadar Compr}, along with estimates such as object kinematics and existence probabilities. In contrast, RadarGNN outputs require additional post-processing to form tracks.

Finally, \cite{pivotweb} provides frame-by-frame comparison videos between PiVoT and RadarGNN$_{\mathrm{B}}$ on three complete radar sequences, together with additional PiVoT demonstrations on varied RadarScenes clips.

\section{Theoretical analysis} 

\label{sec: theoretical analysis} 
This section presents Theorem~\ref{main theorem}, which justifies PiVoT's early identification of ineffective births in Section~\ref{sec: ineffective birth}. Further details and the proof are provided in Appendices~\ref{sec: theoretical lemma and proof} and~\ref{sec: selection of Ls}.

We first state the modelling assumption and introduce the admissible bound used in the theorem.
\begin{assumption} \label{assump}
The single-object measurement likelihood $\ell$ in \eqref{likelihood poisson existence} follows \eqref{measurement model}. The clutter Poisson intensity $\lambda_c=\Lambda_0\ell_0(Y_j)$ is uniform over the surveillance area. The birth prior $p(X_{n,k})$ in \eqref{eq: birth prior} is uninformative in the observation space, i.e. $H\Sigma_k^bH^\top=c I_{d_Y}$ with $c\to\infty$, where $I_{d_Y}$ is the identity matrix, and $H$ is in \eqref{measurement model}. This induces a uniform prior over $\mathbb{R}^{d_Y}$ for $HX_k$.
\end{assumption}

\begin{remark} \label{birth prior remark}
    A convenient consequence of this assumption is that the updated state $q_1(X_k)$ in \eqref{eq: all updates form}, \eqref{eq: state update} satisfies $H\Sigma_k H^\top \!=\! \overbar{R}_k$ (see Lemma~\ref{lemma: covariance update} in Appendix~\ref{sec: analysis ca update}) and $H\mu_k \!=\! \overbar{y}_k$ with $\overbar{R}_k,\overbar{y}_k$ given in \eqref{eq: state update}. This greatly simplifies the state update for newly-born objects, as used in Section \ref{sec: ineffective birth}'s Model Scenario.
\end{remark}

\begin{definition}[Admissible bound] \label{def:admissible-bound}
For any $s>0$ and any newly born object $k$, define the birth-common constants
\begin{align} 
\notag c(s):=&\exp(0.5[\psi_{d_Y}((\phi_k^\prime+s)
/2)+\log|\Phi_k^\prime|-d_Y\log\pi])\\ \label{eq: c(s)}
&\quad \times p_k^{e\prime}p_k^d\rho_k \exp(\psi(\eta_k^\prime+s))\textstyle\frac{1}{\lambda_c},\\ \label{eq: t(s)}
t(s):=&c(s)\exp(-\textstyle \frac{d_Y}{2s}),
\end{align}
where $\phi_k^\prime, \Phi_k^\prime, p_k^{e\prime}, \eta_k^\prime$ are the (birth) prior parameters in \eqref{eq: prior parameters}; $d_Y,\psi_{d_Y},\psi$ are as in \eqref{eq: Tjk definition}; $\rho_k$ is given in \eqref{eq: rate update} and fixed across CAVI iterations; and $\lambda_c$ is the clutter intensity in Assumption~\ref{assump}.

For a newly born object $k$, a scalar $B_k(s)$ is called an admissible bound if, for any current association distribution satisfying $\sum_{j=1}^{M}q_1(\theta_j=k)<s$,
one round of CAVI updates \eqref{eq: shape update}, \eqref{eq: rate update}, and \eqref{eq: state update} produces parameters $\mu_k,\phi_k,\Phi_k$ such that
\begin{align}
\label{eq: Bk(s)}
\Delta_k(s)\leq B_k(s)
\end{align}
where $\Delta_k(s)$ is defined as
\begin{align} \label{eq: Delta(s)}
    & \Delta_k(s):= \sum\nolimits_{j=1}^M \!\frac{1+t(s)}{d_k(y_j)+t(s)},  \\
    \label{eq: dk(s)}
    &d_k(y_j):= \exp(0.5(y_j\!-H\mu_k)^{\!\top} \!\phi_k\Phi_k (y_j\!-H\mu_k)),
\end{align}
\end{definition}
\begin{remark} \label{rmk: admis bound def}
The universal choice $B_k(s)=M$ is always admissible, since $d_k(y_j)\geq 1$ for all $j$.  Tighter admissible bounds are possible; Appendix~\ref{sec: construction U(s)} gives one such $B_k(s)$ for the known-shape case.
\end{remark}

We now present the main theorem.
\begin{theorem} \label{main theorem}
Let Assumption~\ref{assump} hold. For any $s>0$ and any newly born object $k\!=\!K_{n-1}\!+\!1,\ldots,K_n$, let $B_k(s)$ be an admissible bound in Definition~\ref{def:admissible-bound}, with $c(s)$ defined in \eqref{eq: c(s)}. Denote by $W_{-1}:[-1/e,0)\rightarrow(-\infty,-1]$ the $-1$ branch of the Lambert W function, and define
    \begin{align} \label{eq: Ls}
        L_k(s)\!:=\!\begin{cases}  
     s, &\hspace{-1.5cm}  \textup{if} \ \ \frac{ c(s) B_k(s)}{d_Y/2}\exp(-\frac{d_Y/2}{B_k(s)})< e \\ 
     \min\{s, V_k(s)\}, & \textup{otherwise}
    \end{cases} 
    \end{align}
    \vspace{-0.5em}
    \begin{align} \label{eq: Vs}
        \!V_k(s):=\frac{1}{\frac{1}{B_k(s)}-\frac{1}{d_Y/2}W_{\!-1}\!\left(-\frac{d_Y/2}{c(s)B_k(s)}\exp\left(\frac{d_Y/2}{B_k(s)}\right) \right)}
    \end{align}
    
Suppose CAVI iteratively applies the updates \eqref{eq: association update}--\eqref{eq: state update}. If, at any CAVI iteration,
$\sum_{j=1}^M q_1(\theta_j=k)<L_k(s)$, then $\sum_{j=1}^M q_1(\theta_j=k)$ decreases monotonically to $0$ in subsequent iterations, with consecutive updates satisfying
\begin{align} \label{eq: association ratio bound}
\frac{
\sum_{j=1}^{M} q_1^{\textup{new}}(\theta_j=k)
}{
\sum_{j=1}^{M} q_1^{\textup{old}}(\theta_j=k)
}
<
\frac{\Delta_k(s)}{B_k(s)}
\leq 1,
\end{align}
where $\Delta_k(s)$ is evaluated from \eqref{eq: Delta(s)} after the corresponding CAVI round starting from $q_1^{\textup{old}}(\theta)$.
\end{theorem}

\begin{remark}
A larger $L_k(s)$ enables earlier identification of guaranteed ineffective births, improving inference efficiency. Appendix~\ref{sec: function behaviour} shows that a smaller admissible bound $B_k(s)$ only increases the threshold $L_k(s)$, or leaves it unchanged. Hence a tighter $B_k(s)$ is preferred when available; see Remark~\ref{rmk: admis bound def} under Definition \ref{def:admissible-bound} for admissible choices of $B_k(s)$.
\end{remark}
\begin{remark}
$L_k(s)$ and $V_k(s)$ remain fixed across CAVI iterations for fixed $s$ and a chosen admissible bound $B_k(s)$. If $B_k(s)$ is also independent of $k$, the same threshold $L_k(s)$ can be used for all birth components sharing the same birth prior and for all iterations. More refined control is possible by recomputing $L_k(s)$ with different choices of $s$ or $B_k(s)$.
\end{remark}
\begin{remark}
The threshold $L_k(s)$ in Theorem~\ref{main theorem} provides only a sufficient condition for monotonic decrease. In practice, we use a simple heuristic threshold $L$, such as $0.5$ for higher inference efficiency or $0.05$ for more reliable identification of ineffective births. This heuristic is convenient, although an overly large $L$ may risk discarding valid births.
\end{remark}
Further discussion on how the threshold $L_k(s)$ and related functions depend on $B_k(s)$ and $s$, together with a tighter construction of $B_k(s)$, is provided in Appendix~\ref{sec: selection of Ls}.

The proof of Theorem~\ref{main theorem}, supporting lemmas, and detailed CAVI update analysis are provided in Appendix~\ref{sec: theoretical lemma and proof}.

\section{Conclusion}
\label{sec: conclusion}
We present PiVoT, a fast, training-free framework for large-scale multi-object detection and tracking in heavy clutter from point clouds. Experiments show that PiVoT significantly outperforms existing Bayesian trackers in cluttered or large-scale scenes and scales to a thousand objects without gating. PiVoT runs efficiently on a full-scale automotive radar dataset, with performance comparable to a deep-learning detection benchmark, while requiring no training and additionally providing tracking, shape, and existence probability estimation.         

For automotive radar applications, future work includes a formal Bayesian treatment of moving-object indication, 4D Doppler radar modelling, the incorporation of radar cross-section information, and object classification. More broadly, PiVoT could be extended towards decentralised sensor fusion~\cite{li2025decentralised} and nonlinear Gaussian measurement models.

\bibliographystyle{IEEEtran}
\bibliography{./IEEEabrv, references}

@book{bar1995multitarget,
  title={Multitarget-multisensor tracking: principles and techniques},
  author={Bar-Shalom, Yaakov and Li, Xiao-Rong},
  volume={19},
  year={1995},
  publisher={YBs Storrs, CT}
}

@inproceedings{rahmathullah2017generalized,
  title={Generalized optimal sub-pattern assignment metric},
  author={Rahmathullah, Abu Sajana and Garc{\'\i}a-Fern{\'a}ndez, {\'A}ngel F and Svensson, Lennart},
  booktitle={2017 20th International Conference on Information Fusion (FUSION)},
  pages={1--8},
  year={2017},
  organization={IEEE}
}

@article{blei2017variational,
  title={Variational inference: A review for statisticians},
  author={Blei, D. M. and Kucukelbir, A. and McAuliffe, J. D.},
  journal={Journal of the American statistical Association},
  volume={112},
  number={518},
  pages={859--877},
  year={2017},
  publisher={Taylor \& Francis}
}

@Book{           bishop:2006:PRML,
  author = 	 "C. Bishop",
  title = 	 "Pattern Recognition and Machine Learning",
  publisher =  "Springer",
  year = 	 "2006",
}

@inproceedings{gilholm2005poisson,
  title={Poisson models for extended target and group tracking},
  author={Gilholm, Kevin and Godsill, Simon and Maskell, Simon and Salmond, David},
  booktitle={Signal and Data Processing of Small Targets 2005},
  volume={5913},
  pages={230--241},
  year={2005},
  organization={SPIE}
}

@inproceedings{Gan2022,
  title={Conditionally factorized Variational {B}ayes with importance sampling},
  author={Gan, R. and Godsill, S. J.},
  booktitle={IEEE International Conference on Acoustics, Speech and Signal Processing (ICASSP)},
  year={2022},
  organization={IEEE},
}

@article{meyer2021scalable,
  title={Scalable detection and tracking of geometric extended objects},
  author={Meyer, Florian and Williams, Jason L},
  journal={IEEE Transactions on Signal Processing},
  volume={69},
  pages={6283--6298},
  year={2021},
  publisher={IEEE}
}

@inproceedings{lau2016structured,
  title={A structured mean field approach for existence-based multiple target tracking},
  author={Lau, Roslyn A and Williams, Jason L},
  booktitle={2016 19th International Conference on Information Fusion (FUSION)},
  year={2016},
  organization={IEEE}
}

@article{granstrom2019poisson,
  title={Poisson {multi-Bernoulli} mixture conjugate prior for multiple extended target filtering},
  author={Granstr{\"o}m, Karl and Fatemi, Maryam and Svensson, Lennart},
  journal={IEEE Transactions on Aerospace and Electronic Systems},
  volume={56},
  number={1},
  pages={208--225},
  year={2019},
  publisher={IEEE}
}

@article{gan2024variational,
  title={Variational tracking and redetection for closely-spaced objects in heavy clutter},
  author={Gan, Runze and Li, Qing and Godsill, Simon J},
  journal={IEEE Transactions on Aerospace and Electronic Systems},
    volume={60},
  number={4},
  pages={5286--5311},
  year={2024},
  publisher={IEEE}
}

@inproceedings{gan2022variational,
  title={A variational {Bayes} association-based multi-object tracker under the non-homogeneous {Poisson} measurement process},
  author={Gan, Runze and Li, Qing and Godsill, Simon},
  booktitle={2022 25th International Conference on Information Fusion (FUSION)},
  pages={1--8},
  year={2022},
  organization={IEEE}
}

@article{turner2014complete,
  title={A complete variational tracker},
  author={Turner, Ryan D and Bottone, Steven and Avasarala, Bhargav},
  journal={Advances in Neural Information Processing Systems},
  year={2014}
}

@book{papoulis2002probability,
  author    = {A. Papoulis and S. U. Pillai},
  title     = {Probability, Random Variables, and Stochastic Processes},
  edition   = {4},
  publisher = {McGraw-Hill},
  year      = {2002},
  note      = {{Section 5-4}}
}

@article{li2023adaptive,
  title={An adaptive and scalable multi-object tracker based on the non-homogeneous {Poisson} process},
  author={Li, Qing and Gan, Runze and Liang, Jiaming and Godsill, Simon J},
  journal={IEEE Transactions on Signal Processing},
  volume={71},
  pages={105--120},
  year={2023},
  publisher={IEEE}
}

@inproceedings{li2023scalable,
  title={A scalable {Rao-Blackwellised} sequential {MCMC} sampler for joint detection and tracking in clutter},
  author={Li, Qing and Gan, Runze and Godsill, Simon},
  booktitle={2023 26th International Conference on Information Fusion (FUSION)},
  organization={IEEE}
}

@article{xia2023trajectory,
  title={Trajectory {PMB} filters for extended object tracking using belief propagation},
  author={Xia, Yuxuan and Garc{\'\i}a-Fern{\'a}ndez, {\'A}ngel F and Meyer, Florian and Williams, Jason L and Granstr{\"o}m, Karl and Svensson, Lennart},
  journal={IEEE Transactions on Aerospace and Electronic Systems},
  volume={59},
  number={6},
  pages={9312--9331},
  year={2023},
  publisher={IEEE}
}

@article{li2025decentralised,
  title={Decentralised Variational Inference Frameworks for Multi-object Tracking on Sensor Networks},
  author={Li, Qing and Gan, Runze and Godsill, Simon J},
  journal={IEEE Transactions on Signal Processing},
  year={2025},
    volume={73},
  number={},
  pages={2753-2767},
}

@article{davey2007integrated,
  title={Integrated track maintenance for the {PMHT} via the hysteresis model},
  author={Davey, Samuel J and Gray, Douglas A},
  journal={IEEE transactions on Aerospace and Electronic Systems},
  volume={43},
  number={1},
  pages={93--111},
  year={2007},
  publisher={IEEE}
}

@article{granstrom2022tutorial,
  title={A tutorial on multiple extended object tracking},
  author={Granstr{\"o}m, Karl and Baum, Marcus},
  journal={Authorea Preprints},
  year={2022},
  publisher={Authorea}
}

@article{koch2008bayesian,
  title={Bayesian approach to extended object and cluster tracking using random matrices},
  author={Koch, Johann Wolfgang},
  journal={IEEE Transactions on Aerospace and Electronic Systems},
  volume={44},
  number={3},
  pages={1042--1059},
  year={2008},
  publisher={IEEE}
}

@misc{pivotweb,
  note         = {PiVoT project page and demos: \url{https://runzegan.github.io/projects/pivot/}}
}

@inproceedings{schumann2021radarscenes,
  title={{RadarScenes}: A real-world radar point cloud data set for automotive applications},
  author={Schumann, Ole and Hahn, Markus and Scheiner, Nicolas and Weishaupt, Fabio and Tilly, Julius F and Dickmann, J{\"u}rgen and W{\"o}hler, Christian},
  booktitle={2021 IEEE 24th International Conference on Information Fusion (FUSION)},
  pages={1--8},
  year={2021},
  organization={IEEE}
}

@article{malzer2021constraint,
  title={Constraint-based hierarchical cluster selection in automotive radar data},
  author={Malzer, Claudia and Baum, Marcus},
  journal={Sensors},
  volume={21},
  number={10},
  pages={3410},
  year={2021},
  publisher={MDPI}
}

@inproceedings{scheiner2019multi,
  title={A multi-stage clustering framework for automotive radar data},
  author={Scheiner, Nicolas and Appenrodt, Nils and Dickmann, J{\"u}rgen and Sick, Bernhard},
  booktitle={2019 IEEE Intelligent Transportation Systems Conference (ITSC)},
  pages={2060--2067},
  year={2019},
  organization={IEEE}
}

@inproceedings{knill2016direct,
  title={A direct scattering model for tracking vehicles with high-resolution radars},
  author={Knill, Christina and Scheel, Alexander and Dietmayer, Klaus},
  booktitle={2016 IEEE Intelligent Vehicles Symposium (IV)},
  pages={298--303},
  year={2016},
  organization={IEEE}
}

@inproceedings{thormann2021incorporating,
  title={Incorporating range rate measurements in EKF-based elliptical extended object tracking},
  author={Thormann, Kolja and Baum, Marcus},
  booktitle={2021 IEEE International Conference on Multisensor Fusion and Integration for Intelligent Systems (MFI)},
  pages={1--6},
  year={2021},
  organization={IEEE}
}

@inproceedings{fent2023radargnn,
  title={{RadarGNN}: Transformation invariant graph neural network for radar-based perception},
  author={Fent, Felix and Bauerschmidt, Philipp and Lienkamp, Markus},
  booktitle={Proceedings of the IEEE/CVF Conference on Computer Vision and Pattern Recognition},
  pages={182--191},
  year={2023}
}

@article{ahmad2023review,
  title={A review of automatic classification of drones using radar: key considerations, performance evaluation, and prospects},
  author={Ahmad, Bashar I and Rogers, Colin and Harman, Stephen and Dale, Holly and Jahangir, Mohammed and Antoniou, Michael and Baker, Chris and Newman, Mike and Fioranelli, Francesco},
  journal={IEEE Aerospace and Electronic Systems Magazine},
  volume={39},
  number={2},
  pages={18--33},
  year={2023},
  publisher={IEEE}
}

@article{scheiner2021new,
  title={New Challenges for Deep Neural Networks in Automotive Radar Perception: An Overview of Current Research Trends},
  author={Scheiner, Nicolas and Weishaupt, Fabio and Tilly, Julius F and Dickmann, Jurgen},
  journal={Automatisiertes Fahren 2020: Von der Fahrerassistenz zum autonomen Fahren 6. Internationale ATZ-Fachtagung},
  pages={165--182},
  year={2021},
  publisher={Springer}
}

@article{scheiner2021object,
  title={Object detection for automotive radar point clouds--a comparison},
  author={Scheiner, Nicolas and Kraus, Florian and Appenrodt, Nils and Dickmann, J{\"u}rgen and Sick, Bernhard},
  journal={AI Perspectives},
  volume={3},
  number={1},
  pages={6},
  year={2021},
  publisher={Springer}
}

@article{cavagna2019sparta,
  title={{SpaRTA} tracking across occlusions via partitioning of {3D} clouds of points},
  author={Cavagna, Andrea and Melillo, Stefania and Parisi, Leonardo and Ricci-Tersenghi, Federico},
  journal={IEEE Transactions on Pattern Analysis and Machine Intelligence},
  volume={43},
  number={4},
  pages={1394--1403},
  year={2019},
  publisher={IEEE}
}

@article{cox1996efficient,
  title={An efficient implementation of {Reid's} multiple hypothesis tracking algorithm and its evaluation for the purpose of visual tracking},
  author={Cox, Ingemar J. and Hingorani, Sunita L.},
  journal={IEEE Transactions on pattern analysis and machine intelligence},
  volume={18},
  number={2},
  pages={138--150},
  year={1996},
  publisher={IEEE}
}

@article{kari2023evolutionary,
  title={Evolutionary developments of today’s remote sensing radar technology—Right from the telemobiloscope: A review},
  author={Kari, Samedh Sachin and Raj, A Arockia Bazil and K, Balasubramanian.},
  journal={IEEE Geoscience and Remote Sensing Magazine},
  volume={12},
  number={1},
  pages={67--107},
  year={2023},
  publisher={IEEE}
}

@article{wang2024sequential,
  title={Sequential point clouds: A survey},
  author={Wang, Haiyan and Tian, Yingli},
  journal={IEEE Transactions on Pattern Analysis and Machine Intelligence},
  volume={46},
  number={8},
  pages={5504--5523},
  year={2024},
  publisher={IEEE}
}

@incollection{BOLE2014139,
title = {Chapter 3 - Target Detection},
booktitle = {Radar and ARPA Manual},
publisher = {\hspace{-0.5em} Butterworth-Heinemann},
edition = {3rd},
pages = {139-213},
year = {2014},
isbn = {978-0-08-097752-2},
doi = {https://doi.org/10.1016/B978-0-08-097752-2.00003-9},
author = {Alan Bole and Alan Wall and Andy Norris},
}

@inproceedings{tilly2020detection,
  title={Detection and tracking on automotive radar data with deep learning},
  author={Tilly, Julius F and Haag, Stefan and Schumann, Ole and Weishaupt, Fabio and Duraisamy, Bharanidhar and Dickmann, J{\"u}rgen and Fritzsche, Martin},
  booktitle={2020 IEEE 23rd International Conference on Information Fusion (FUSION)},
  pages={1--7},
  year={2020},
  organization={IEEE}
}

@article{scheel2018tracking,
  title={Tracking multiple vehicles using a variational radar model},
  author={Scheel, Alexander and Dietmayer, Klaus},
  journal={IEEE Transactions on Intelligent Transportation Systems},
  volume={20},
  number={10},
  pages={3721--3736},
  year={2018},
  publisher={IEEE}
}

@inproceedings{gan2025pivot,
  title={{PiVoT: Poisson} measurements-based variational multi-object detection and tracking},
  author={Gan, Runze and Li, Qing and Hopgood, James R and Davies, Mike E and Godsill, Simon},
  booktitle={2025 28th International Conference on Information Fusion (FUSION)},
  pages={1--8},
  year={2025},
  organization={IEEE}
}

\clearpage

\vspace{5em}
\setcounter{page}{1}
\onecolumn
\begin{center} { \LARGE Supplementary Materials for PiVoT: A Variational Solution for Real-time Large-scale Multi-object Detection and Tracking under Heavy Clutter } \end{center}
\appendices
\section{Notation} \label{apx: notation}
Tables~\ref{tb: variable notation}--\ref{tab: notation_functions} summarise the notation for variables, parameters, probability laws, and functions used throughout the paper.

\vspace{1em}
\begin{table}[!htbp]
\centering
\caption{Main variable and parameter notation. Note that the time index subscript $n$ is omitted from Section~\ref{sec: PiVoT challenges solution} onwards to simplify notation, unless needed to avoid ambiguity.}
\label{tab: notation_variables}
\footnotesize
\begin{tabular}{p{0.24\textwidth} p{0.68\textwidth}}
\hline
Notation & Description \\
\hline
$n$ & Discrete time index. \\
$k$, $j$ & Object index and measurement index, respectively. \\
$K_n$ & Number of potentially existing objects at time $n$. \\
$K_n^b$ & Number of birth objects at time $n$. \\
$M_n$ & Number of received measurements at time $n$. \\
$Y_{n,j}$, $Y_n=[Y_{n,1},\ldots,Y_{n,M_n}]$ & The $j$-th measurement and the full measurement set at time $n$. \\
$y_{n,j}$ & Positional component of $Y_{n,j}$ in the positional-only and Doppler-augmented models. \\
$d_Y$ & Dimension of the positional measurement space. \\
$\mathbb{R}^{d_Y}$ & $d_Y$-dimensional real vector space\\
$I_{d_Y}$ & Identity matrix of dimension $d_Y$. \\
$X_{n,k}$, $X_n=[X_{n,1},\ldots,X_{n,K_n}]$ & Kinematic state of object $k$ and the collection of all object states. \\
$\Lambda_{n,k}$, $\Lambda_n=[\Lambda_{n,1},\ldots,\Lambda_{n,K_n}]$ & Poisson rate of object $k$ and the collection of all object rates. \\
$P_{n,k}$, $P_n=[P_{n,1},\ldots,P_{n,K_n}]$ & Shape precision matrix of object $k$ and the collection of all object shape precision matrices. \\
$E_{n,k}$, $E_n=[E_{n,1},\ldots,E_{n,K_n}]$ & Existence indicator of object $k$ and the collection of all existence indicators. \\
$D_{n,k}$, $D_n=[D_{n,1},\ldots,D_{n,K_n}]$ & Detectability indicator of object $k$ and the collection of all detectability indicators. \\
$\Lambda_{n,0}$ & Clutter Poisson rate. \\
$\theta_{n,j}$, $\theta_n=[\theta_{n,1},\ldots,\theta_{n,M_n}]$ & Association variable for measurement $Y_{n,j}$ and the collection of all measurement oriented associations. \\
$X_{k-}$, $\Lambda_{k-}$, $P_{k-}$ & Collections of all object states, rates, and shape precision matrices except those of object $k$. \\
$E_{k-}$, $D_{k-}$ & Collections of all existence and detectability indicators except those of object $k$. \\
$F_n$, $Q_n$ & State transition matrix and covariance for legacy objects at time $n$; see \eqref{eq: linear Gaussian transition}.\\
$p_{n,k}^d$ & Detection probabilities of object $k$ at time $n$; see \eqref{eq: detectability model}.\\
$p_{n,k}^s$, $p_{n,k}^b$ & Survival and birth probabilities of object $k$ at time $n$; see Section \ref{sec: legacy object model} and \eqref{eq: birth prior}. \\
$\mu_{n,k}^b$, $\Sigma_{n,k}^b$ & Mean and covariance of the Gaussian birth prior for $X_{n,k}$; see \eqref{eq: birth prior}. \\
$\eta_{n,k}^b$, $\rho_{n,k}^b$ & Shape and scale parameters of the Gamma birth prior for $\Lambda_{n,k}$; see \eqref{eq: birth prior}. \\
$\Phi_{n,k}^b$, $\phi_{n,k}^b$ & Scale matrix and degrees of freedom of the Wishart birth prior for $P_{n,k}$; see \eqref{eq: birth prior}. \\
$p_{n,k}^{e\prime}$ & Predictive prior existence probability of object $k$; see \eqref{eq: prior parameters}. \\
$\mu'_{n,k}$, $\Sigma'_{n,k}$ & Mean and covariance of the Gaussian predictive prior for $X_{n,k}$; see \eqref{eq: prior parameters}. \\
$\eta'_{n,k}$, $\rho'_{n,k}$ & Shape and scale parameters of the Gamma predictive prior for $\Lambda_{n,k}$; see \eqref{eq: prior parameters}. \\
$\Phi'_{n,k}$, $\phi'_{n,k}$ & Scale matrix and degrees of freedom of the Wishart predictive prior for $P_{n,k}$; see \eqref{eq: prior parameters}. \\
$\gamma_{\Lambda,n},\gamma_{P,n}$ & Forgetting factors for the rate and shape prediction step; see \eqref{eq: prediction}.\\
$X_k^p$, $X_k^v$ & Positional and remaining components of the kinematic state $X_k$, respectively; used in Section \ref{sec: initialisation, clustering demo}. \\
$L$ & Practical threshold for early removal of ineffective births. \\
$H$ & Observation matrix mapping the object state to the positional measurement space, as exemplified in \eqref{eq: G and H}. \\
$G$ & Matrix that extracts the velocity component from the object state, as exemplified in \eqref{eq: G and H}. \\
$v_{n,j}$ & Scalar Doppler velocity component of $Y_{n,j}$ in the Doppler-augmented model, where $Y_{n,j}=(y_{n,j},v_{n,j})$. \\
$\sigma_v^2$ & Variance of the scalar Doppler velocity measurement noise in the Doppler-augmented model; see \eqref{eq: Doppler likelihood}. \\
$\alpha_{n,j}$, $e_{n,j}$ & Bearing angle of measurement $Y_{n,j}$ and the corresponding line of sight unit vector; see \eqref{eq: Doppler likelihood} and Fig. \ref{fig: Dop Illust Combined}(b). \\
$\mathrm{v}_{n,k}^{j}$ & True radial velocity of object $k$ at the bearing of measurement $j$; see Fig. \ref{fig: Dop Illust Combined}(b). \\
$\mathrm{v}_{n,k}^{\mathrm{x}},\mathrm{v}_{n,k}^{\mathrm{y}}$ & Cartesian velocity components of object $k$ along the $x$ and $y$ axes in the Doppler-augmented model; see Fig. \ref{fig: Dop Illust Combined}(b). \\
$\bar{y}_{k}$ & Pseudo measurement for object $k$ constructed from all positional measurements for state update; see \eqref{eq: state update}. \\
$\bar{R}_{k}$ & Covariance matrix associated with the positional pseudo measurement $\bar{y}_{k}$; see \eqref{eq: state update}. \\
$\bar{u}_{k}$ & Pseudo measurement for object $k$ constructed from all Doppler measurements for state update; see \eqref{eq: pseudo measurements Dopp}. \\
$\bar{U}_{k}$ & Covariance matrix associated with the Doppler pseudo measurement $\bar{u}_{k}$; see \eqref{eq: pseudo measurements Dopp}. \\
$P_{\mathrm{pru}}$, $P_{\mathrm{rep}}$, $P_{\mathrm{stp}}$ & Existence probability thresholds for pruning, reporting, and stopping reporting an object, respectively; see Section \ref{sec: estimation extraction}.\\
$e$ & Euler's number, the base of the natural exponential, used in Theorem~\ref{main theorem}. \\
\hline
\end{tabular}
\label{tb: variable notation}
\end{table}

\clearpage

\begin{table}[!htbp]
\centering
\caption{Probability law and variational notation. The time index subscript $n$ is omitted from Section~\ref{sec: PiVoT challenges solution} onwards for clarity.}
\label{tab: notation_laws}
\footnotesize
\begin{tabular}{p{0.24\textwidth} p{0.68\textwidth}}
\hline
Notation & Description \\
\hline
$p$ & Exact probability law specified by the generative model; see Sec.~\ref{sec: Prob Formulation}. \\
$\hat{p}_n$ & Filtering probability law at time $n$, formed by combining the predictive prior constructed from the filtered result at time $n-1$ with the exact current measurement likelihood; see \eqref{eq: filtering posterior definition}. \\
$\bar{p}$ & Approximate filtering probability law obtained after applying the approximation in \eqref{eq: final approx for Xi}; defined separately for Stage~1 and Stage~2 in \eqref{eq: stage 1 unnormalised} and \eqref{eq: stage 2 unnormalised}. \\
$q_1$ & Stage 1 variational law used for detection, tracking, rate estimation, shape estimation, and association. \\
$q_2$ & Stage 2 object-wise variational law used for existence and detectability evaluation. \\
$q_2^*$ & Optimal Stage 2 variational law under the corresponding Stage 2 objective. \\
\hline
\end{tabular}
\label{tb: law notation}
\end{table}

\begin{table}[!htbp]
\centering
\caption{Main named functions.}
\label{tab: notation_functions}
\footnotesize
\begin{tabular}{p{0.24\textwidth} p{0.50\textwidth} p{0.18\textwidth}}
\hline
Notation & Description & Defined in \\
\hline
$\mathcal{N}(\cdot)$ & Gaussian distribution, parameterised by a mean and covariance. & Standard notation \\
$\mathcal{G}(\cdot)$ & Gamma distribution, parameterised by shape and scale parameters. & Standard notation \\
$\mathcal{W}(\cdot)$ & Wishart distribution, parameterised by a scale matrix and degrees of freedom. & Standard notation \\
$\textup{Ber}(\cdot)$ & Bernoulli distribution, parameterised by its success probability. & Standard notation \\
$\delta[\cdot]$ & Kronecker delta indicator, equal to $1$ when the condition is true and $0$ otherwise. & Standard notation \\
$\psi(\cdot)$ & Digamma function. & Standard notation \\
$\psi_{d_Y}\!(\cdot)$ & Multivariate digamma function of dimension $d_Y$. & Standard notation \\
$\exp(\cdot)$ & Exponential function. & Standard notation \\
$|\cdot|$ & Determinant of a square matrix. & Standard notation \\
$\Tr(\cdot)$ & Trace of a square matrix. & Standard notation \\
$W_{-1}(\cdot)$ & The $-1$ branch of the Lambert W function. & Standard notation \\
$\ell_0(Y_j)$ & Clutter measurement likelihood. & Sections \ref{sec: measuremnt model} and \ref{sec: Doppler model} \\
$\ell(Y_j|X_k,P_k)$ & Single-object measurement likelihood, defined separately for the positional-only and Doppler-augmented models. & \eqref{measurement model} and \eqref{eq: Doppler likelihood}, respectively. \\
$h(\Lambda_n,D_n,\theta_n)$ & Product term in the association prior. & \eqref{eq: association prior form 1}, \eqref{eq: association prior form 2} \\
$\Xi_k(\Lambda_k,\theta)$ & Object specific term obtained after marginalising $D_k$ in the association prior. & \eqref{eq: original expectation} \\
$S_j^k$ & Expected contribution of object $k$ to measurement $Y_j$ in the association update. & \eqref{eq: association update} \\
$g(D_k)$ & Detectability dependent term used in the Stage~2 update of $q_2(E_k,D_k)$. & \eqref{eq: simplified g} \\
$c(s)$, $t(s)$ & Birth common constants used to construct the admissible threshold. & \eqref{eq: c(s)} and \eqref{eq: t(s)}\\
$d_k(y_j)$ & Exponential distance score between measurement $y_j$ and object $k$. & \eqref{eq: dk(s)} \\
$\Delta_k(s)$ & Aggregate score controlling contraction of the association mass of object $k$. & \eqref{eq: Delta(s)} \\
$B_k(s)$ & Admissible upper bound on $\Delta_k(s)$ for birth object $k$. & \eqref{eq: Bk(s)} \\
$L_k(s)$ & Sufficient threshold of birth object $k$ for association mass to decay to zero.& \eqref{eq: Ls} \\
$V_k(s)$ & Auxiliary threshold function used in the construction of $L_k(s)$. & \eqref{eq: Vs} \\
\hline
\end{tabular}
\label{tb: function notation}
\end{table}

\twocolumn
% \clearpage
\section{Limitations of the Naive Mean-Field Assumption} 
\label{apx:limitation MF}
This appendix examines the inevitable failure of using the naive mean-field structure in \eqref{eq: naive mean-field} for the variational distribution $q(\theta,X,\Lambda,P,E,D)$ to approximate the true posterior $\hat{p}(\theta,X,\Lambda,P,E,D|Y)$ in \eqref{eq: filtering posterior definition} by minimising KL$(q\|\hat{p})$. Applying the standard CAVI update formula to this variational inference problem, we can verify that the mean-field structure $q$ in \eqref{eq: naive mean-field} needs to satisfy the following finer mean-field structure to achieve a lower KL$(q\|\hat{p})$:
\begin{align} \label{eq: finer mean-field}
    q(\theta,X,\Lambda,P,E,D)\!=\!\prod_{j=1}^M\!q(\theta_j)\!\prod_{k=1}^K \! q(X_k)q(\Lambda_k)q(P_k)q(E_k,D_k\!)
\end{align}
A fundamental issue arises because the true posterior $\hat{p}(\theta,X,\Lambda,P,E,D|Y)$ in \eqref{eq: filtering posterior definition} involves $p(\theta,M|\Lambda,D)$ in \eqref{eq: association prior} and $p(D_k|E_k)$ in \eqref{eq: detectability model} as multiplication factors. Specifically, for any $k\!=\!1,...,K$ and $j\!=\!1,...,M$, the joint factor $p(\theta_j\!=\!k, \theta_{j-},M|\Lambda,D)p(D_k|E_k\!=\!0)$ is always zero. Consequently, $\hat{p}(\theta_j=k,E_k=0,\theta_{j-},E_{k-},X,\Lambda,P,D)$ is always zero, i.e. $\hat{p}$ is zero whenever a measurement $Y_j$ is associated with a non-existent object $k$. Since a variational distribution cannot assign probability mass where $\hat{p}=0$ (as otherwise KL$(q\|\hat{p})$ would be infinite), we must have: 
\begin{align} \notag
    q(\theta_j=k,E_k=0,\theta_{j-},E_{k-},X,\Lambda,P,D)=0,
\end{align}
for all $j,k$, and for all values of $\theta_{j-},E_{k-},X,\Lambda,P,D$. Since $q$ follows the structure in \eqref{eq: finer mean-field}, this constraint implies 
\begin{align} \label{eq: naive constraint}
    \!q(\theta_j\!=\!k)q(E_k\!=\!0,D_k)q(\theta_{j-},E_{k-},X,\Lambda,P,D_{k-})=0.
\end{align}
To allow an object $k$ even a small probability of associating with at least one measurement, i.e. $q(\theta_j=k)\neq 0$, the only way to ensure the constraint \eqref{eq: naive constraint} for all possible values of $\theta_{j-},E_{k-},X,\Lambda,P,D_{k-}$ and $D_k$ is to enforce $q(E_k=0,D_k)=0$ for all values of $D_k$, which in turn implies $q(E_k=0)=0$. This leads to an inevitable failure: 
\begin{enumerate}
  \item Forcing $q(E_k=0)=0$ means the object $k$ must always exist, eliminating any meaningful uncertainty quantification of existence.
  \item Otherwise, the object $k$ cannot associate with any measurement, preventing it from ever being updated.
\end{enumerate}
Either outcome is problematic, making the naive mean-field assumption in \eqref{eq: naive mean-field} fundamentally flawed.

\section{Optimal NHPP Approximation View of \eqref{eq: final approx for Xi}} \label{apx: alternative approx interpret}
This appendix provides a detailed interpretation of the approximation in \eqref{eq: final approx for Xi} from the perspective of point process approximation. We first show that approximating $\Xi_k(\Lambda_k,\theta)$ in \eqref{eq: original expectation} by \eqref{eq: final approx for Xi} is equivalent to replacing a marginalised point process with an NHPP. We then justify this replacement by showing that the approximating NHPP minimises the KLD from the original marginalised process.

\subsubsection{Point process interpretation of the approximation} \label{apx: PP interpretation}
Recall from Section \ref{sec: measuremnt model} that the overall measurement process is constructed by superposing a clutter process with intensity $\Lambda_0\ell_0(Y_j)$ and $K$ independent NHPPs, where the $k$-th NHPP has intensity $D_k\Lambda_k\ell(Y_j|X_k,P_k)$, dependent on the detectability variable $D_k$. From \eqref{eq: likelihood joint}-\eqref{eq: association prior form 1}, the joint association and likelihood density for the overall measurement process is
\begin{align} \label{eq: joint likeasoc example}
    &p(Y,M,\theta|D,X,\Lambda,P)=\textstyle\frac{1}{M!}p(Y|\theta,P,X)\ \Omega_0\prod\nolimits_{s=1}^K\Omega_s\\ \notag
    &\Omega_0=e^{-\Lambda_0}\Lambda_0^{\sum_{j=1}^{M} \delta[\theta_{j}=0]}, \ \ \Omega_s=e^{-D_s\Lambda_s}(D_s\Lambda_s)^{\sum_{j=1}^{M} \delta[\theta_{j}=s]},
\end{align}
where each $\Omega_s$ is contributed solely by the $s$-th NHPP and is independent of the other superposed components.

Now consider marginalising $D_k$ under prior $\hat{p}(D_k)$, similarly to \eqref{eq: stage 1 unnormalised goal} and \eqref{eq: stage 2 unnormalised goal}. The global joint density $\sum_{D_k}\hat{p}(D_k)p(Y,M,\theta|D,X,\Lambda,P)$ retains the same form as \eqref{eq: joint likeasoc example}, except that $\Omega_k$ is replaced by its expectation under $\hat{p}(D_k)$, i.e. $\E_{\hat{p}(D_k)}\Omega_k=\Xi_k(\Lambda_k,\theta)$, which is exactly the quantity defined in \eqref{eq: original expectation}. From a modelling perspective, this corresponds to replacing the $k$-th component of the superposed measurement process, which was originally an NHPP conditioned on $D_k$, with a marginalised point process. The other components remain unchanged due to independence.

To examine this $D_k$-marginalised point process more closely, let $Z=[Z_1,\ldots,Z_m]$ denote the measurement set generated by object $k$ alone, with cardinality $m$. Its likelihood density $p_k(Z,m|X_k,\Lambda_k,P_k)$ is obtained by integrating out $D_k$ from the likelihood of the original $k$-th NHPP, i.e.
\begin{align} \notag
    &p_k(Z,m|X_k,\Lambda_k,P_k) =  p_k(m|\Lambda_k)\prod\nolimits_{j=1}^m \ell(Z_j|X_k,P_k),\\ \label{eq: margi k NHPP}
    &\qquad\qquad p_k(m|\Lambda_k):=\E_{\hat{p}(D_k)}\text{Pois}(m;D_k\Lambda_k).
\end{align}
where $\text{Pois}(m;D_k\Lambda_k)$ denotes the Poisson distribution with rate $D_k\Lambda_k$, and $\hat{p}_k(D_k)=\text{Ber}(p_k^{e \prime}p_k^d)$ as defined in \eqref{eq: predictive Dk}. This marginalised process is generally not an NHPP. Rather, it is a mixture between an empty process and an NHPP, with the original spatial density $\ell(\cdot|X_k,P_k)$ retained.

The approximation replaces $\Omega_k$ in \eqref{eq: joint likeasoc example} by 
$$e^{-p_k^{e \prime}p_k^d\Lambda_k} (p_k^{e \prime}p_k^d\Lambda_k)^{\sum_{j=1}^{M_n} \delta[\theta_{j}=k]}.$$
Comparing this expression with the original form of $\Omega_k$ in \eqref{eq: joint likeasoc example}, we see that this corresponds to replacing the $k$-th NHPP with rate $D_k\Lambda_k$ by a new NHPP with rate $p_k^{e \prime}p_k^d\Lambda_k$. Since this substitution does not alter the spatial distribution encoded in the likelihood term $p(Y|\theta,P,X)$ in \eqref{eq: joint likeasoc example}, this approximating NHPP inherits the same spatial distribution $\ell(\cdot|X_k,P_k)$. Hence, the intensity function of the approximating NHPP is: $p_k^{e \prime}p_k^d\Lambda_k\ell(\cdot|X_k,P_k)$.

Consequently, replacing $\Xi_k(\Lambda_k,\theta)$ in \eqref{eq: original expectation} with the approximation in \eqref{eq: final approx for Xi} is equivalent to replacing the $D_k$-marginalised point process (i.e. the $k$-th component superposed into the overall measurement process) with an NHPP of intensity $p_k^{e \prime}p_k^d\Lambda_k\ell(\cdot|X_k,P_k)$.

\subsubsection{Justification as KLD minimisation}
\label{apx: KLD minimisation}
We justify the employed NHPP approximation by showing that it arises naturally as the solution to a KLD minimisation problem. Specifically, we consider the following result:
\begin{lemma} \label{lemma: KLD minimisation}
    Let $q(Z,m)$ denote the likelihood density of an NHPP. Among all NHPP likelihoods, the $q(Z,m)$ that minimises the KLD from the $D_k$-marginalised point process, i.e. 
    \begin{align} \label{eq: KLD point process}
        \textup{KL}(p_k(Z,m|X_k,\Lambda_k,P_k)||q(Z,m)),
    \end{align}
    where $p_k(Z,m|X_k,\Lambda_k,P_k)$ is defined in \eqref{eq: margi k NHPP}, is given by
    \begin{align} \label{eq: KLD minimisation result}
        q(Z,m)=\textup{Pois}(m;\Lambda_k\E_{\hat{p}(\!D_k\!)}D_k)\prod\nolimits_{j=1}^m \!\ell(Z_j|X_k,P_k).
    \end{align}
\end{lemma}
\begin{proof}
    As a general form, the likelihood of an NHPP can be expressed as
    \begin{align} \label{eq: NHPP likelihood form}
        q(Z,m)=\text{Pois}(m;\lambda)\prod\nolimits_{j=1}^m q_\ell (Z_j),
    \end{align}
    where $\lambda>0$ is the rate, and $q_\ell (Z_j)$ is the spatial distribution. Subsequently, the KLD $\text{KL}(p_k||q)$ in \eqref{eq: KLD point process} can be expressed as 
    \begin{align} \label{eq: KL factorisation}
        \text{KL}(p_k||q)=& \ \text{KL}(p_k(m|\Lambda_k)||\text{Pois}(m;\lambda) )\\ \notag
        &+\E_{p_k(m|\Lambda_k)}\sum\nolimits_{j=1}^m\text{KL}(\ell(Z_j|X_k,P_k)|| q_\ell (Z_j) ).
    \end{align}
The second term is minimised to zero by setting $q_\ell (\cdot)=\ell(\cdot|X_k,P_k)$, which does not affect the first term. Therefore, the optimal spatial distribution is $q_\ell (Z_j)=\ell(Z_j|X_k,P_k)$. 

The remaining KLD minimisation task is to find the $\lambda$ that minimises $\text{KL}(p_k(m|\Lambda_k)||\text{Pois}(m;\Lambda) )$. Since the Poisson distribution belongs to the exponential family, its KLD is minimised by matching the expectation of its sufficient statistic, which is $m$. That is, the optimal $\Lambda$ satisfies
$$\E_{p_k(m|\Lambda_k)}T(m)= \E_{\text{Pois}(m;\Lambda)}T(m)=\lambda.$$
Using the definition of $p_k(m|\Lambda_k)$ in \eqref{eq: margi k NHPP}, we compute
\begin{align} \notag
    \lambda= \sum\nolimits_m \E_{\hat{p}(D_k)} \text{Pois}(m;D_k\Lambda_k) m = \Lambda_k\E_{\hat{p}(\!D_k\!)}D_k.
\end{align}
Substituting this $\lambda$ and $q_\ell (Z_j)=\ell(Z_j|X_k,P_k)$ into \eqref{eq: NHPP likelihood form} yields the \eqref{eq: KLD minimisation result} stated in the lemma, completing the proof.
\end{proof}

Using Lemma \ref{lemma: KLD minimisation} and the definition of $\hat{p}(D_k)$ in \eqref{eq: predictive Dk}, we find that the NHPP that minimises the KLD in \eqref{eq: KLD point process} from the $D_k$-marginalised point process has intensity $p_k^{e \prime}p_k^d\Lambda_k\ell(\cdot|X_k,P_k)$, which exactly matches the employed approximate NHPP. The resulting approximation in \eqref{eq: final approx for Xi} for $\Xi_k(\Lambda_k,\theta)$ is hence justified.

\section{Approximation Accuracy Analysis and Refinement for Isolated Objects}
\label{apx: justification of approx for far away}
This appendix provides justification for why the approximation introduced in \eqref{eq: final approx for Xi} has minimal impact on the inference accuracy of an object $k$ that is well-separated from others, and also explains how to remove this approximation error through an additional refinement step in Stage 2. 

\subsubsection{Justification for well-separated object $k$}
Specifically, we show that the approximation error does not affect the Stage 2 inference for object $k$ where $q_1(X,\Lambda,P)$ is fixed. The argument is to show that, under the same gating reduction that is safe for a sufficiently isolated object, the Stage~2 update without the approximation in \eqref{eq: final approx for Xi} and the employed Stage~2 update with the approximation in \eqref{eq: final approx for Xi} yield the same expression for $g(D_k)$.

It is well known that for an object $k$ sufficiently isolated from others, gating techniques can be employed to restrict consideration to only a subset of measurements that lie within a sufficiently large gate around this object. Denote the set of labels for these measurements as $\mathcal{G}_k$. In this scenario, associations can be safely assumed to involve either object $k$ or clutter, as the remaining objects are too distant to contribute meaningfully. The Stage 2 updates derived in \eqref{eq: existance update} can then be applied for this single-object case, and the optimal $\log q_2^*(E_k,D_k)$ is given by \eqref{eq: existance update}, where $g(D_k)$ takes the form:
\begin{align} \label{gDk exact gating}
    g(D_k)\!=\!-\eta_k\rho_kD_k\!+\!\sum\nolimits_{j\in\mathcal{G}_k}\log\! \Big(\!D_kS_j^k\!+\!\Lambda_0\ell_0(Y_j)\!\Big).
\end{align}
Note that this expression involves no approximation via $\Xi_k(\Lambda_k,\theta)$ from \eqref{eq: final approx for Xi}, as only object $k$ is considered and the $\prod\nolimits_{k-}\Xi_k(\Lambda_k,\theta)$ in \eqref{eq: partial marginal assoc}, to which the approximation applies, is absent.

For comparison, consider the derived expression for $\log q_2^*(E_k,D_k)$ in the multi-object case with $K_n$ objects, where the approximation in \eqref{eq: final approx for Xi} is applied. The corresponding $g(D_k)$ becomes:
\begin{align} \notag
    g(D_k)\!=\!-\eta_k\rho_kD_k\!+\!\sum_{j=1}^M\log\! \Big(\!D_kS_j^k\!+\!\Lambda_0\ell_0(Y_j)\!+\!\!\sum_{\substack{s=1 \\ s \neq k}}^{K_n} p_s^{e \prime}p_s^d S_j^s\Big).\\[-2.0em]\notag
\end{align}
For $j\notin \mathcal{G}_k$, the log-summands can be ignored as constants independent of $D_k$ since $S_j^k \approx 0$. This is due to the very large value of $(Y_j\!-\!H\mu_k)^{\!\top} \!\phi_k\Phi_k (Y_j\!-\!H\mu_k)$ in the definition of $S_j^k$ in \eqref{eq: association update}, as measurements outside the gate are far from object $k$.
% This is due to the large $(Y_j\!-\!H\mu_k)^{\!\top} \!\phi_k\Phi_k (Y_j\!-\!H\mu_k)$ in the definition of $S_j^k$ in \eqref{eq: association update}. 
Similarly, for $j\in \mathcal{G}_k$, $S_j^s \approx 0$ for all $s \neq k$, since those objects $s$ are far from the measurements within the gate. Consequently, $\sum_{s\neq k} p_s^{e \prime}p_s^d S_j^s$ can be neglected. Neglecting these terms, the resulting expression for $g(D_k)$ under the approximation coincides with the exact form in \eqref{gDk exact gating}, showing that the approximation in \eqref{eq: final approx for Xi} does not impact the Stage 2 inference for object $k$.

Hence, the remaining influence of the approximation comes from the $q_1(X,\Lambda,P)$ computed from the Stage 1 inference.

\subsubsection{Approximation mitigation and Stage 2 refinement} \label{apx: refinement}
The approximation error from \eqref{eq: final approx for Xi} introduced in Stage 1 for $q_1(X,\Lambda,P)$ can be mitigated or removed by extending the Stage 2 objective to include a refinement step. Specifically, we may incorporate $X_k,\Lambda_k,P_k$ in $q_2$ and exclude them from the fixed $q_1$. In doing so, we continue to approximate the same target distribution $\hat{p}$ as in \eqref{eq: stage 2 marginal posterior}, but with a refined and richer variational form:
\begin{align}\notag
    q_2(E_k,D_k)&q_2(\theta|E_k,D_k)q_2(X_k)q_2(\Lambda_k)q_2(P_k)\\ \notag
    &\times q_1(X_{k-})q_1(\Lambda_{k-})q_1(P_{k-}).
\end{align}
In this refined Stage~2 objective, the detectability variable $D_k$ of the object being evaluated is retained explicitly. Hence, unlike in Stage~1 where the approximation in \eqref{eq: final approx for Xi} is applied to every object, this approximation is not applied to the object-specific contribution of $k$. The approximation can only enter through the marginalised components of the other objects. For a sufficiently well-separated object, these other-object contributions are negligible within the local gate of object $k$, following the same reasoning as above. Thus, the refinement can reduce, and in the isolated gated case remove, the residual impact of the Stage~1 approximation on the inferred object features.

Importantly, this refinement procedure is flexible. One may choose to refine only a subset of variables or a subset of objects, rather than updating all of $X_k$, $\Lambda_k$, and $P_k$ for every $k$. Furthermore, structured dependencies can be incorporated, such as $q_2(X_k | E_k)$, where $q_2(X_k | E_k = 1)$ can approximate a more accurate conditional posterior $\hat{p}(X_k \mid E_k = 1, Y)$, as advocated in works such as \cite{lau2016structured,meyer2021scalable}. Nevertheless, all such refinement steps require additional computation and iterative updates, and are therefore omitted in this paper.

\section{Derivation of efficient evaluation of $q_2^*(E_k,D_k)$} \label{apx: Stage2 derivation}
Here we derive the optimal $q_2^*(E_k,D_k)$ in \eqref{eq: existance update} and \eqref{eq: simplified g}, which enables an efficient Stage 2 implementation.

First, observing from \eqref{eq: conditional association}, the normalised $q_2^*(\theta|E_k,D_k)$ is
\vspace{-0.1em}
\begin{align} \notag
    \frac{\exp(\E_{q_1(X,\Lambda,P)}\log \bar{p}(X, \Lambda, P, \theta , E_k, D_k,  Y))}{\sum_{\theta} \exp(\E_{q_1(X,\Lambda,P)}\log \bar{p}(X, \Lambda, P, \theta , E_k, D_k,  Y)) }.\\[-1.5em]\notag
\end{align}
Substituting this into the $\log q_2^*(\theta| E_k,D_k)$ term in \eqref{eq: direct conditional update}, the last line of \eqref{eq: direct conditional update} cancels, yielding $\log q_2^*(E_k,D_k)$ up to an additive constant as
\vspace{-0.1em}
\begin{align} \notag
    &\log \textstyle\sum_{\theta} \exp(\E_{q_1(X,\Lambda,P)}\log \bar{p}(X, \Lambda, P, \theta , E_k, D_k,  Y)).
\end{align}

Next, substituting \eqref{eq: stage 2 unnormalised}, \eqref{eq: individual likelihood}, and discarding terms constant with respect to $E_k,D_k$, the optimal $\log q_2^*(E_k,D_k)$ is
\begin{align}\notag
    &\log q_2^*(E_k,D_k)\eqc\log\hat{p}(E_k)p(D_k|E_k)\\ \notag
    & \quad \ + \log \textstyle\sum_{\theta} \exp(\E_{q_1(X,\Lambda,P)}\log p(Y|\theta,P,X)\bar{p}(\theta,M|\Lambda,D_k) )    \\ \notag
    &= \log \hat{p}(E_k)p(D_k|E_k) \!+\! \log \sum_{\theta}\! \Big[ \!\exp(\E_{q_1(\Lambda)} \!\log \bar{p}(\theta,M|\Lambda,D_k))\\[-0.5em] \notag
    &\qquad\qquad\qquad\quad \times\prod\nolimits_{j=1}^M \exp\big(\E_{q_1(X,P)}\log p(Y_j|\theta_j,P,X)\big)\Big].
    \\ \notag
    & \eqc \log \hat{p}(E_k)p(D_k|E_k)\! - \! D_k\E_{q_1(\Lambda)}\Lambda_k \!+\! \log \sum\nolimits_{\theta}\!\!\Big[\! \prod\nolimits_{j=1}^M  \\ \notag
    & \hspace{1em}\big[\!\exp\E_{q_1(\Lambda)}\log\big(\sum\nolimits_{\substack{s=1 \\ s \neq k}}^{K_n}p_s^{e \prime}p_s^d \Lambda_s \delta[\theta_j=s]
    +\! \Lambda_kD_k\delta[\theta_j=k]\\ \notag
    &\hspace{1.5em}+\Lambda_0\delta[\theta_j=0]\big)\big]\exp\big(\E_{q_1(X,P)}\log p(Y_j|\theta_j,P,X)\big)\Big],
\end{align}
where $\eqc$ denotes equality up to additive constant terms with respect to $E_k$ and $D_k$. The last line follows from \eqref{eq: stage 2 marginal assoc}, \eqref{eq: association prior form 2} and the identity $\exp\E\log(\prod_{j}a_j)=\prod_{j} \exp\E\log a_j$. Furthermore, since the delta indicators select exactly one intensity term for each fixed $\theta_{n,j}$, we have $f(\sum_{k=0}^{K_n}b_k\delta[\theta_{n,j}\!=\!k])=\sum_{k=0}^{K_n} f(b_k)\delta[\theta_{n,j}\!=\!k]$ for $f$ equal to $\exp$ or $\log$. Therefore, we arrive at
\begin{align} \notag 
    &\log q_2^*(E_k,D_k)\eqc \log \hat{p}(E_k)p(D_k|E_k) -D_k\E_{q_1(\Lambda)}\Lambda_k\\ \notag
    & \ + \log \prod\nolimits_{j=1}^M \sum\nolimits_{\theta_j=0}^{K_n}\Big[\exp(\E_{q_1(X,P)}\log p(Y_j|\theta_j,P,X)) \\ \notag
    &\times \!\big(\Lambda_0\delta[\theta_j=0]\!+\! \tilde{\Lambda}_kD_k\delta[\theta_j=k]\!+\!\sum\nolimits_{\substack{s=1 \\ s \neq k}}^{K_n}p_s^{e \prime}p_s^d \tilde{\Lambda}_s \delta[\theta_j=s]\big)\!\Big]
\end{align}
where $\tilde{\Lambda}_k:=\exp(
\E_{q_1(\Lambda_k)}\log\Lambda_k)$ for all $k=1,...,K_n$. Finally, performing the summation over $\theta_j$, substituting $p(Y_j|\theta_j,P,X)$ in \eqref{eq: individual likelihood} and the expectation $\E_{q_1(\Lambda)}\Lambda_k$ using \eqref{eq: all updates form}, we obtain the final expression for the optimal $q_2^*(E_k,D_k)$:

\vspace{-0.5em}
\begin{align} \notag
    \log q_2^*(E_k,D_k)\eqc\log\hat{p}(E_k)p(D_k|E_k)+g(D_k),\\[-2.0em]\notag
\end{align}
\begin{align} \notag
    g(D_k)\!=\!-\eta_k\rho_kD_k\!+\!\sum_{j=1}^M\log\! \Big(\!D_kS_j^k\!+\!\Lambda_0\ell_0(Y_j)\!+\!\!\sum_{\substack{s=1 \\ s \neq k}}^{K_n} p_s^{e \prime}p_s^d S_j^s\Big)\\[-2.0em]\notag
\end{align}
where $S_j^k$ is defined in \eqref{eq: association update}. This yields the $\log q_2^*(E_k,D_k)$ expression in \eqref{eq: existance update}. Since each logarithm differs by only one summand for different $k$, we simplify $g(D_k)$ using a shared total sum:
\vspace{-0.5em}
\begin{align} \notag
g(D_k)&=\!-\eta_k\rho_kD_k\!+\!\sum_{j=1}^M\log\! \Big(\!S_j^\text{sum}+(D_k-p_k^{e \prime}p_k^d)S_j^k\Big)\\[-2.9em]\notag
\end{align}
\begin{align}\notag
S_j^\text{sum}&=\Lambda_0\ell_0(Y_j)+\sum_{k=1}^{K_n} p_k^{e \prime}p_k^d S_j^k.\\[-2.0em]\notag
\end{align}
This matches \eqref{eq: simplified g} and completes the derivation.

\section{Numerically stable computation of $g(D_k)$} \label{apx: numerical stable gD}
While \eqref{eq: simplified g} and \eqref{eq: Sj sum} offer an efficient way to compute $g(D_k)$, they may encounter numerical issues in practice.
This appendix presents a numerically stable yet efficient evaluation of $g(D_k\!=\!1)$ and $g(D_k\!=\!0)$ that mitigates overflow/underflow, and quantifies its computational cost. For all $j \!=\! 1,...,M$ and $k \!=\! 1,...,K_n$, the following quantities are assumed to be precomputed and stored with sufficient precision. These values are also required in the log-sum-exp computation of $S_j^\text{sum}$ in \eqref{eq: Sj sum}, and thus should already be available:
\begin{itemize}
    \item Log-domain terms $\log S_j^{\text{sum}} $, along with $\log S_j^k$, 
    \begin{align} \notag
        h_j^k=\log (p_k^{e \prime}p_k^dS_j^k)
    \end{align} for $k=1,...,K_n$, and $h_j^0=\log(\Lambda_0\ell_0(Y_j))$.

    \item The maximum value $m_j=\max_{k=0}^{K_n} h_j^k$ and its index $I_j\in$ $\{0,1,...,K_n\}$ such that $m_j = h_j^{I_j}$.
    
    \item The stabilised terms $\exp(h_j^k - m_j)$, and their sum $U_j^{\text{sum}} = \sum_{k=0}^{K_n} \exp(h_j^k - m_j)$.
\end{itemize}

\subsubsection{$g(D_k = 1)$}
For $D_k = 1$, the two summands inside each logarithm in \eqref{eq: simplified g} are positive, so we can apply the log-sum-exp trick to obtain: 
\begin{align}\notag
    &g(D_k\!=\!1)\!=\!-\eta_k\rho_k\!+\!\! \sum_{j=1}^{M} \!\max \{a_j, b_j^k \} \!+\! \log\!\big(1\!+\!\exp(\!-|a_j\!-\!b_j^k |)\!\big)\\ \label{eq: gD1 computation}
    &\qquad  a_j= \log S_j^\text{sum}, \qquad  b_j^k = \log(1\!-\!p_k^{e \prime}p_k^d)\!+\!\log S_j^k.
\end{align}
The quantities $a_j$ and $b_j^k$ are either precomputed or can be efficiently evaluated from stored quantities. The main computational cost lies in evaluating $\log(1+\exp(-|a_j-b_j^k |))$ for all $j$ and $k=1,...,K_n$, which requires $K_nM$ exponential and logarithmic operations in total to compute all $g(D_k=1)$.

\subsubsection{$g(D_k = 0)$} 
For $D_k = 0$ the computation is slightly more complex. From \eqref{eq: simplified g} we have
\begin{align} \label{eq: gD0 computation}
\begin{aligned}
g(D_k\!=\!0)=\!\sum\nolimits_{j=1}^M \!g_j^k, \ \quad g_j^k=\log(S_j^\text{sum}\!-p_k^{e \prime}p_k^dS_j^k).
\end{aligned}
\end{align}
To evaluate $g_j^k$ numerically stably for all $k = 1,\dots,K_n$ and $j = 1,\dots,M$, thereby ensuring sufficient precision for all $g(D_k=0)$, we consider two scenarios:
\paragraph{For $k \neq I_j$}
We use
\begin{align}\label{eq: numerical stable g0 k neq Ij}
    g_j^k=m_j+\log(U_j^{\text{sum}}-\exp(h_j^k - m_j)),
\end{align}
where $m_j,U_j^{\text{sum}},\exp(h_j^k - m_j)$ are all precomputed and stored. Since $k\neq I_j$, the largest stabilised term remains in the sum after subtraction, so $U_j^{\text{sum}}-\exp(h_j^k - m_j)\in[1,K_n]$. This keeps the logarithm numerically stable.
\paragraph{For $k= I_j$} The subtraction in the expression above may be close to zero, risking underflow before the logarithm is taken. In this case, we apply the log-sum-exp trick using the second-highest value $c_j$ from $\{h_j^0,h_j^1,...,h_j^K\}$: 
\begin{align} \label{eq: numerical stable g0 k = Ij}
    g_j^k=&c_j+\log\Big( \exp(h_j^0-c_j)+\sum\nolimits_{\substack{s=1 \\ s \neq k}}^{K_n} \exp(h_j^s-c_j)\Big).
\end{align}
Since $c_j$ is the largest value after excluding $h_j^k$, the largest exponentiated term in the sum is $1$, while all remaining terms lie in $(0,1]$. This gives a stable log-sum-exp evaluation of $g_j^k$.

\subsubsection{Cost analysis} We now assess the cost in terms of exponential and logarithmic operations. Computing all $g(D_k = 0)$ for $k=1,...,K_n$ requires evaluating all $g_j^k$ via either \eqref{eq: numerical stable g0 k neq Ij} or \eqref{eq: numerical stable g0 k = Ij}. Specifically: 1) Each use of \eqref{eq: numerical stable g0 k neq Ij} involves one exponential and one logarithm, and is applied up to $K_n M$ times. 2) Each use of \eqref{eq: numerical stable g0 k = Ij} involves $K_n$ exponentials and one logarithm, and is applied at most $M$ times. Hence, the total number of exponentials and logarithms required is bounded by $2 K_n M$.

Recall that computing all $g(D_k=1)$ for $k=1,...,K_n$ via \eqref{eq: gD1 computation} mainly requires $K_nM$ exponential and logarithmic operations in total. The numerically stable evaluation of all $g(D_k)$ therefore retains $\mathcal{O}(K_n M)$ complexity. In practice, the total number of exponential/logarithmic evaluations is bounded by $3 K_n M$, making it comparable to performing two to three updates of $q_1(\theta)$ in \eqref{eq: association update}. This remains substantially more efficient than the direct computation via \eqref{eq: direct conditional update}, which requires at least $\mathcal{O}(K_n^2 M)$ operations.

\section{Additional implementation details} \label{apx: implem details}
This appendix supplements Section~\ref{sec: Algorithm structure} with additional implementation details on variational initialisation and moving-object selection.
\subsubsection{Variational initialisation and closely spaced birth pruning} \label{sec: CAVI initialisation}
Points~(i)--(iii) of Section~\ref{sec: initialisation, clustering demo} suggest using a flat factorised birth prior, initialising newly born components over distinct small positional regions, and tiling these regions over the surveillance area. Building on this birth initialisation, we recommend first running two to three CAVI iterations for existing objects $k=1,\ldots,K_{n-1}$ with births disabled, so that existing tracks are prioritised in associating with measurements. The resulting $q_1(X_{n,k})$ for existing objects is then combined with the birth initialisation above to start the subsequent CAVI inference with new births included. This reduces the risk that newborn objects explain measurements belonging to existing tracks.

In addition, when updating $q_1(\theta_n)$ in the first CAVI iteration via \eqref{eq: association update}, we recommend using a clutter likelihood $\ell_0(Y_{n,j})$ lower than the expected true level. This biases the early association update towards explaining measurements by objects, which helps $q_1(X_{n,k})$ localise within its prescribed region.

After CAVI converges, multiple variational components may represent a single physical object. This is most likely under non elliptical measurement patterns or gradual target entry, where an existing component explains the already observed portion while a birth component captures the newly observed part. Such duplication inflates the estimated object count. A simple remedy is to prune newly born components whose converged positional means lie very close to other objects, and then rerun CAVI to restore convergence.

\subsubsection{Post-processing details for Doppler PiVoT moving object selection} \label{apx: doppler post processing}

This appendix details the post-processing stage for moving-object selection used in the Doppler PiVoT automotive radar experiments. The procedure is applied after inference, operates only on inferred track statistics, and does not modify the variational updates. Its purpose is to extract a reliable subset of moving road users from the broader set of scene elements tracked by PiVoT.

Tracks are first nominated as moving candidates using minimum speed gates and velocity Mahalanobis distance over short time windows, together with basic directional and trajectory smoothness checks. This stage removes stationary tracks and artefacts induced by flickering clutter and ego motion. An example of surviving candidates are shown in Fig. \ref{fig: Dop post processing}(b).

Second, candidates are examined for Doppler informativeness. Tracks whose associated Doppler velocities are largely perpendicular to the estimated motion direction, or have weak projected components along that direction, are treated as Doppler uninformative and flagged as unreliable. This suppresses objects for which stationary and moving hypotheses are unidentifiable under the adopted model. For example, the orange ellipses in Fig. \ref{fig: Dop post processing}(b) correspond to stationary road barrier segments for which the radar line of sight is nearly orthogonal to the ego motion. These segments consistently appear at the side of the ego vehicle with near-zero Doppler and can therefore be misinterpreted as moving consistently with the ego vehicle, rather than being stationary. Such tracks are rejected at this stage unless supported by smooth trajectories, kinematic consistency, and evidence of independent motion.

Third, a concrete object shape constraint is enforced. Objects are rejected if (i) estimated shapes are implausible for road participants (e.g. are excessively large or have unrealistic aspect ratios), or (ii) their associated detections fail to dominate the measurements within the estimated 95 percent shape ellipse (e.g. excessive stationary measurements inside). Tracks removed at this stage are shown as dashed ellipses in Fig. \ref{fig: Dop post processing}(b).

Finally, a ghost suppression step removes objects arising from multipath reflections, leveraging PiVoT’s ability to track stationary structures. Potential ghost and true target pairs are first identified based on symmetric motion directions about the perpendicular bisector of their positions, with ghosts characterised by longer sensor range, shorter track duration, and fewer associated measurements. These candidates are then validated by checking whether the inferred reflection points lie within stationary objects. Confirmed ghost tracks are removed. Tracks rejected by this step are shown as cyan ellipses in Fig. \ref{fig: Dop post processing}(b), whose reflection points fall within the elongated stationary ellipses corresponding to road barriers in Fig. \ref{fig: Dop post processing}(a).

Only tracks that satisfy all consistency checks are reported as moving objects (the remaining red ellipses in Fig. \ref{fig: Dop post processing} (b)). All criteria and window lengths are fixed across experiments. This post-processing stage is computationally negligible compared to inference and serves solely to improve the reliability and interpretability of the reported outputs.

\section{Theoretical foundations and proof of Theorem~\ref{main theorem}} \label{sec: theoretical lemma and proof}
This appendix presents the proof of Theorem~\ref{main theorem}, supported by auxiliary lemmas analysing the CAVI updates.
\subsubsection{Preliminary algebraic lemmas}
We begin by presenting algebraic lemmas that will be used in the subsequent analysis.
\begin{lemma} \label{lemma sugar water ieq}
    Let $0<a\leq\alpha$, and $0<x\leq 1$. Then,
    \begin{align}
        1-\frac{1}{1+ax}\leq \Big(1-\frac{1}{1+a}\Big)\frac{\alpha+1}{\alpha+1/x}.
    \end{align}
\end{lemma}
\begin{proof}
Since $a+1\leq a+1/x$, adding $\alpha-a\geq 0$ to both the numerator and denominator of $\frac{a+1}{a+1/x}$ does not decrease the ratio; hence, $\frac{a+1}{a+1/x}\leq\frac{\alpha+1}{\alpha+1/x}$. Therefore, we have
$$ \Big(1-\frac{1}{1+a}\Big)\frac{\alpha+1}{\alpha+1/x} \geq \Big(1-\frac{1}{1+a}\Big)\frac{a+1}{a+1/x} = 1-\frac{1}{1+ax}, $$
hence completing the proof.
\end{proof}

\begin{lemma} \label{lemma: lambert W}
    Let $p,q,x>0$, and let $W_{-1}:[-\frac{1}{e},0)\rightarrow(-\infty,-1]$, $W_{0}:[-\frac{1}{e},\infty)\rightarrow[-1,\infty)$ denote the $-1$ and $0$ branches of the Lambert W function, respectively. Then,
    \begin{align} \label{eq: ieq in lemma}
        x> p\big[1-\frac{1}{1+q\exp(\frac{1}{p}-\frac{1}{x})}\big]
    \end{align}
    if and only if one of the following conditions holds: 1) $pq < e$, or 2) $pq \geq e$ and either $x<\big[\frac{1}{p}-W_{-1}(-\frac{1}{pq}\big)\big]^{-1}$ or $x>\big[\frac{1}{p}-W_{0}(-\frac{1}{pq}\big)\big]^{-1}$.
\end{lemma}
\begin{proof}
    Let $t\!=\!\frac{1}{p}\!-\!\frac{1}{x}$, such that $x\!=\!\frac{p}{1-pt}$. Then, \eqref{eq: ieq in lemma} becomes 
    \begin{align} \label{eq: lemma chaneg of var}
        \frac{p}{1-pt}>p(1-\frac{1}{1+qe^t})=\frac{pqe^t}{1+qe^t}
    \end{align}
    Since $1-pt=1-(1-p/x)>0$, multiplying both sides of \eqref{eq: lemma chaneg of var} by $(1-pt)(1+qe^t)/p$ yields
    \begin{align}\notag
        1+qe^t>qe^t-pqte^t \iff te^t>-(pq)^{-1}.
    \end{align}
    This inequality always holds if $-\frac{1}{pq}<-\frac{1}{e}$, establishing condition 1 of Lemma \ref{lemma: lambert W}. Otherwise, when $-\frac{1}{pq}\geq-\frac{1}{e}$, the inequality holds if and only if $t<W_{-1}(-\frac{1}{pq})$ or $t>W_0(-\frac{1}{pq})$. Substituting back $t=\frac{1}{p}-\frac{1}{x}$, this becomes $\frac{1}{x}>\frac{1}{p}-W_{-1}(-\frac{1}{pq})$ or $\frac{1}{x}<\frac{1}{p}-W_{0}(-\frac{1}{pq})$. Since both $W_{-1}(-\frac{1}{pq})$ and $W_{0}(-\frac{1}{pq})$ are negative, rearranging yields condition 2 of Lemma \ref{lemma: lambert W}, completing the proof.
\end{proof}
\begin{corollary} \label{lemma: lambert W corr}
    Let $p,q,r,x>0$. Then,
    \begin{align} \label{eq: another ieq in lemma}
        x> p\big[1-\frac{1}{1+q\exp(-\frac{r}{x})}\big]
    \end{align}
    if and only if one of the following conditions holds: 1) $\frac{pq}{r}\exp(-\frac{r}{p}) < e$, or 2) $\frac{pq}{r}\exp(-\frac{r}{p}) \geq e$ and either $x<\big[\frac{1}{p}-\frac{1}{r}W_{-1}(-\frac{r}{pq}\exp(\frac{r}{p}))\big]^{-1}$ or $x>\big[\frac{1}{p}-\frac{1}{r}W_{0}(-\frac{r}{pq}\exp(\frac{r}{p}))\big]^{-1}$.
\end{corollary}
\begin{proof}
    Let $x'=\frac{x}{r}$, $p'=\frac{p}{r}$, $q'=q\exp(-\frac{r}{p})$. Then, inequality \eqref{eq: another ieq in lemma} is equivalent to $x'> p'\big[1-\frac{1}{1+q'\exp(\frac{1}{p'}-\frac{1}{x'})}\big]$. Applying Lemma \ref{lemma: lambert W} completes the proof.
\end{proof}

\subsubsection{Analysis of CAVI update} \label{sec: analysis ca update}
Next, we present two lemmas and a proposition concerning the employed CAVI updates and their bounding behaviour.

\begin{lemma} \label{lemma: covariance update}
    Suppose the prior covariance $\Sigma_k^\prime$ in \eqref{eq: prior parameters}, used in the state update \eqref{eq: state update}, satisfies $H \Sigma_k^\prime H^\top = c I_{d_Y}$, where $c > 0$ is a scalar and $I_{d_Y}$ is the identity matrix. Then the updated covariance $\Sigma_k$ in \eqref{eq: state update} satisfies $H\Sigma_k H^\top=(\overbar{R}_k^{-1}+\frac{1}{c}I_{d_Y})^{-1}$.
\end{lemma}
\begin{proof}
    Using the Gaussian update formula and Woodbury identity, the updated covariance $\Sigma_k$ in \eqref{eq: state update} is
    \begin{align} \notag 
        \Sigma_k&=({\Sigma_k^\prime}^{-1} + H^\top\overbar{R}_k^{-1} H )^{-1}\\ \notag
        &=\Sigma_k^\prime-\Sigma_k^\prime H^\top (\overbar{R}_k+H\Sigma_k^\prime H^\top)^{-1} H\Sigma_k^\prime.
    \end{align}
    Then, applying $H$ and $H^\top$ to both sides, we obtain
    \begin{align} \notag
        H &\Sigma_k H^\top\!= H \Sigma_k^\prime H^\top -H \Sigma_k^\prime H^\top (\overbar{R}_k+H\Sigma_k^\prime H^\top)^{-1} H \Sigma_k^\prime H^\top \\ \notag
        &=c I_{d_Y}-c^2 (\overbar{R}_k +cI_{d_Y})^{-1}=c(I_{d_Y}-(\textstyle\frac{1}{c}\overbar{R}_k+I_{d_Y})^{-1})\\ \notag
        &=c[I_{d_Y}-(I_{d_Y}-(c\overbar{R}_k^{-1}+I_{d_Y})^{-1})]=(\overbar{R}_k^{-1}+\textstyle\frac{1}{c}I_{d_Y})^{-1},
    \end{align}
    where the second last equality follows from the Woodbury identity. This completes the proof.
\end{proof}

\begin{lemma} \label{lemma: CAVI bounded by c}
    Let $\sum_{j=1}^M q_1(\theta_j=k)<s$. Then the $\eta_k,\phi_k,\Phi_k$ obtained from updates \eqref{eq: rate update}, \eqref{eq: shape update} using such a $q_1(\theta)$ satisfy:
    \begin{align} \label{eq: lemma digamma bounding}
        \psi(\eta_k)<\psi(\eta_k^\prime+s)&, \quad  \psi_{d_Y}(\phi_k/2)\!<\!\psi_{d_Y}((\phi_k^\prime+s)/2),\\ \label{eq: lemma PSD bounding}
        &\log |\Phi_k|\leq\log |\Phi_k^\prime|, 
    \end{align}
    where $d_Y$ is the dimension of $y_j$.
\end{lemma}
\begin{proof}
    The inequalities in \eqref{eq: lemma digamma bounding} follow from the strict monotonicity of the digamma function $\psi (\eta)$ for $\eta>0$ and the multivariate digamma function $\psi_{d_Y}(\phi/2)$ for $\phi>d_Y-1$. The inequalities in \eqref{eq: lemma PSD bounding} follow from the fact that $\Phi_k'-\Phi_k$ is a positive semi definite matrix, as implied by \eqref{eq: shape update}.
\end{proof}

\begin{proposition} \label{CAVI bound proposition}
    Define $c(s)$ as in \eqref{eq: c(s)}. Let Assumption~\ref{assump} hold, and suppose
    CAVI applies one round of updates sequentially via \eqref{eq: shape update}-\eqref{eq: state update} with $q_1(\theta)=q_1^{-}(\theta)$, followed by updating $q_1(\theta)$ to $q_1^{+}(\theta)$ using \eqref{eq: Sjk positional only l}-\eqref{eq: association update}.
    Then for any newly-born object $k\!=\!K_{n-1}\!+\!1,...,K_n$ and any $s>0$ such that $\sum_{j=1}^M q_1^-(\theta_j\!=\!k)\!<\!s$, $q_1^{+}(\theta_j\!=\!k)$ satisfies
    \begin{align} \label{eq: q+ bound in proposition}
        q_1^{+}(\theta_j\!=\!k)&< 1-\frac{1}{1\!+\!\frac{c(s)}{ d_k(y_j)}\exp(-\frac{d_Y/2}{\sum_{j=1}^M q_1^{-}(\theta_j=k)})},
    \end{align}
    where $d_k(y_j)$ is evaluated from \eqref{eq: dk(s)} using the parameters $\phi_k,\Phi_k,\mu_k$ obtained from that round of updates \eqref{eq: shape update} and \eqref{eq: state update}.
\end{proposition}
\begin{proof}
    From \eqref{eq: association update} and Assumption \ref{assump}, we have 
    \begin{align} \notag
        q_1^{+}(\theta_j=k)&=\frac{p_k^{e \prime}p_k^d S_j^k}{\lambda_c+\sum_{l=1}^{K_n} p_l^{e \prime}p_l^d S_j^l}\leq \frac{p_k^{e \prime}p_k^d S_j^k}{\lambda_c+p_k^{e \prime}p_k^d S_j^k}\\ \label{eq: q+ bound}
        &\qquad =1-\frac{1}{1+\frac{1}{\lambda_c}p_k^{e \prime}p_k^d S_j^k}
    \end{align}
    where $S_j^k>0$ is defined in \eqref{eq: association update}. Using \eqref{eq: Sjk positional only l}, \eqref{eq: c(s)}, and Lemma \ref{lemma: CAVI bounded by c}, the term $\frac{1}{\lambda_c}p_k^{e \prime}p_k^d S_j^k$ can be upper bounded as
    \begin{align} \notag
        \frac{1}{\lambda_c}p_k^{e \prime}p_k^d& S_j^k< \frac{c(s)\exp(-0.5\Tr(\phi_k\Phi_kH\Sigma_kH^\top\!))}{\exp(0.5(y_j\!-\!H\mu_k)^{\!\top} \!\phi_k\Phi_k (y_j\!-\!H\mu_k))}\\ \label{eq: ppS bounded}
        &=\frac{c(s)}{d_k(y_j)}\exp(-0.5\Tr(\phi_k\Phi_kH\Sigma_kH^\top\!)),
    \end{align}
    where $d_k(y_j)$ is defined in \eqref{eq: dk(s)}, and $\mu_k,\Sigma_k$ are obtained from state update \eqref{eq: state update}. From Lemma \ref{lemma: covariance update} and Assumption \ref{assump}, $H\Sigma_kH^\top$ equals the $\overbar{R}_k$ in \eqref{eq: state update}, and thus $\Tr(\phi_k\Phi_kH\Sigma_kH^\top\!)$ equals $d_Y/\sum_{j=1}^{M}q_1^-(\theta_{j}\!=\!k) $. Subsequently, \eqref{eq: ppS bounded} becomes
    $$\frac{1}{\lambda_c}p_k^{e \prime}p_k^d S_j^k< \frac{c(s)}{d_k(y_j)}\exp\Big(-\frac{d_Y/2}{\sum_{j=1}^M q_1^-(\theta_{j}=k)}\Big). $$
    Substituting this inequality into \eqref{eq: q+ bound} for a further upper bound yields \eqref{eq: q+ bound in proposition}, completing the proof.
\end{proof}

\subsubsection{Proof of Theorem~\ref{main theorem}}  \label{sec: proof for main theorem} 
With Proposition \ref{CAVI bound proposition}, Corollary \ref{lemma: lambert W corr}, and Lemma \ref{lemma sugar water ieq}, we are now ready to prove Theorem~\ref{main theorem}. 
\begin{proof}
    Suppose $\sum_{j=1}^M q_1^-(\theta_j = k) < s$. Proposition~\ref{CAVI bound proposition} first gives the one-step upper bound \eqref{eq: q+ bound in proposition} on $q_1^+(\theta_j=k)$. We further bound its right-hand side using Lemma~\ref{lemma sugar water ieq} with $x=\frac{1}{d_k(y_j)}$, $a=c(s)\exp(-\frac{d_Y/2}{\sum_{j=1}^M q_1^{-}(\theta_j=k)})$, and $\alpha=c(s)\exp(-\frac{d_Y}{2s})=t(s)$, where $t(s)$ is as defined in \eqref{eq: t(s)}. This gives    
    \begin{align} \notag
        q_1^{+}(\theta_j=k)<& \bigg(1-\frac{1}{1+c(s)\exp(-\frac{d_Y/2}{\sum_{j=1}^M q_1^{-}(\theta_j=k)})}\bigg)\\
        &\times \frac{1+t(s)}{d_k(y_j)+t(s)}.
    \end{align}
    Summing over $j$ and using $\Delta_k(s)$ in \eqref{eq: Delta(s)} yields
    \begin{align} \notag
        &\!\!\!\!\sum_{j=1}^{M} \!q_1^{+}(\theta_j\!=\!k) \!< \!\bigg(\!1-\frac{1}{1+c(s)\exp(-\frac{d_Y/2}{\sum_{j=1}^M q_1^{-}(\theta_j=k)})}\bigg) \Delta_k(s)\\ \label{eq: q+ upper bound med}
        & = \!B_k(s)\bigg(\!1\!-\!\frac{1}{1+c(s)\exp(-\frac{d_Y/2}{\sum_{j=1}^M q_1^{-}(\theta_j=k)})}\!\bigg) \frac{\Delta_k(s)}{B_k(s)},
    \end{align}
    where $B_k(s)$ is an admissible bound from Definition~\ref{def:admissible-bound}. We now seek a sufficient condition under which
    \begin{align} \notag
        B_k(s)\bigg(\!1\!-\!\frac{1}{1+c(s)\exp(-\frac{d_Y/2}{\sum_{j=1}^M q_1^{-}(\theta_j=k)})}\!\bigg)<\sum_{j=1}^M q_1^{-}(\theta_j=k), 
    \end{align}
    so that applying \eqref{eq: q+ upper bound med} yields the desired bound \eqref{eq: association ratio bound} in Theorem \ref{main theorem}: $\sum_{j=1}^{M} q_1^{+}(\theta_j\!=\!k) <\frac{\Delta_k(s)}{B_k(s)} \sum_{j=1}^M q_1^{-}(\theta_j=k)$. This sufficient condition can be established via Corollary~\ref{lemma: lambert W corr}, by setting $p=B_k(s)$, $q=c(s)$, $r=d_Y/2$, $x=\sum_{j=1}^M q_1^{-}(\theta_j=k)$. Corollary~\ref{lemma: lambert W corr} then implies that the inequality holds if either of the following two conditions is satisfied:\\
    1) $\frac{c(s)B_k(s)}{d_Y/2}\exp(-\frac{d_Y/2}{B_k(s)})\geq e$ and ${\sum_{j=1}^M q_1^{-}(\theta_j\!=\!k) \!<\!  V_k(s)}$, where $V_k(s)$ is defined in \eqref{eq: Vs}, or \\
    2) $\frac{c(s)B_k(s)}{ d_Y/2}\exp(-\frac{d_Y/2}{B_k(s)})<e$.
    
    Recall that we have also assumed $\sum_{j=1}^M q_1^-(\theta_j = k) < s$. Therefore, combining the above conditions gives the threshold $L_k(s)$ in \eqref{eq: Ls}. Hence, if $\sum_{j=1}^M q_1^{-}(\theta_j=k)<L_k(s)$, then $\sum_{j=1}^{M} q_1^{+}(\theta_j\!=\!k) < \frac{\Delta_k(s)}{B_k(s)}\sum_{j=1}^M q_1^{-}(\theta_j=k)$. 
    
    Since $\Delta_k(s)\leq B_k(s)$ by the admissibility defined in \eqref{eq: Bk(s)}, it follows that $\sum_{j=1}^M q_1^{+}(\theta_j = k) < \sum_{j=1}^M q_1^-(\theta_j = k) < L_k(s)$, so the sufficient condition remains satisfied at the next iteration. By induction, the same bound continues to hold, and $\sum_{j=1}^M q_1(\theta_j = k)$ decreases monotonically across iterations, with the ratio of successive values bounded by $\frac{\Delta_k(s)}{B_k(s)} \leq 1$. 

It remains to show that $\sum_{j=1}^M q_1(\theta_j=k)$ converges to zero. Let $\beta_r:=\sum_{j=1}^M q_1^{(r)}(\theta_j=k)$ denote the association mass after the $r$-th subsequent update from any iteration satisfying $\beta_r<L_k(s)$, and define $$f(x):=B_k(s)\big(1-\frac{1}{1+c(s)\exp(-d_Y/(2x))}\big).$$
The preceding analysis shows that ${\beta_r}$ is monotonically decreasing and bounded below by zero, and therefore $\beta_r\to\beta_\infty\geq0$. Suppose $\beta_\infty>0$. Then since $\beta_\infty<L_k(s)$, the construction of $L_k(s)$ gives $f(\beta_\infty)<\beta_\infty$. However, \eqref{eq: q+ upper bound med} and $\Delta_k(s)\leq B_k(s)$ give $\beta_{r+1}<f(\beta_r)$. Taking limits in this inequality, using $\beta_{r+1}\to\beta_\infty$ and continuity of $f$, yields $\beta_\infty\leq f(\beta_\infty)$, contradicting $f(\beta_\infty)<\beta_\infty$. Hence, $\beta_\infty=0$, completing the proof.

\end{proof}

\section{Additional Details on the Threshold $L_k(s)$ in Theorem \ref{main theorem}} \label{sec: selection of Ls}
This appendix provides additional details on the threshold in Theorem~\ref{main theorem}. We first analyse how the threshold $L_k(s)$ depends on the admissible bound $B_k(s)$ and the parameter $s$, showing that a tighter $B_k(s)$ gives a no smaller threshold for fixed $s$. We then give a practical construction of such a bound for the known-shape case.
\subsubsection{Threshold $L_k(s)$ behaviour with respect to $s$ and $B_k(s)$} \label{sec: function behaviour}
Here we examine how the threshold $L_k(s)$ in Theorem~\ref{main theorem} changes with $s$ and the admissible bound $B_k(s)$. The main message is twofold. First, for fixed $s$, a tighter/lower admissible bound $B_k(s)$ gives a no smaller threshold $L_k(s)$. Second, the dependence on $s$ is more subtle: decreasing $s$ improves the internal bound terms, in particular increasing $V_k(s)$ and making the first-case condition in \eqref{eq: Ls} easier to satisfy. However, $L_k(s)$ is also capped by $s$, so $L_k(s)$ is not monotone in $s$ in general.

We first consider the dependence on $B_k(s)$ for fixed $s$. Since $B_k(s)\exp(-\frac{d_Y/2}{B_k(s)})$ decreases with smaller $B_k(s)$, decreasing $B_k(s)$ makes the first case in \eqref{eq: Ls}, namely $\frac{ c(s) B_k(s)}{d_Y/2}\exp(-\frac{d_Y/2}{B_k(s)})< e$, easier to satisfy. Moreover, in the second case of \eqref{eq: Ls}, it will be shown below that $V_k(s)$ increases as $B_k(s)$ decreases. This gives the claimed monotonicity of $L_k(s)$ with respect to $B_k(s)$. Indeed, fix $s$ and compare two admissible bounds $B_1<B_2$. If $B_2$ satisfies the first case in \eqref{eq: Ls}, then so does $B_1$, and both thresholds equal $s$. If $B_2$ is in the second case but $B_1$ moves to the first case, then the new threshold is $s$, which is no smaller than $\min\{s,V_k(s)\}$. Finally, if both remain in the second case, then $V_k(s)$ is larger for $B_1$, and hence $\min\{s,V_k(s)\}$ is no smaller. Therefore, for fixed $s$, tightening/lowering $B_k(s)$ can only increase, or leave unchanged, the sufficient threshold $L_k(s)$.

We next consider the role of $s$. Decreasing $s>0$ decreases $c(s)$ in \eqref{eq: c(s)} and hence also decreases $t(s)$ in \eqref{eq: t(s)}. Since $d_k(y_j)\geq 1$, each term in \eqref{eq: Delta(s)} is nondecreasing in $t(s)$, so $\Delta_k(s)$ also decreases. Therefore a smaller admissible bound $B_k(s)$ may be obtained. In addition, the condition term $\frac{c(s)B_k(s)}{ d_Y/2}\exp(-\frac{d_Y/2}{B_k(s)})$ in \eqref{eq: Ls} decreases with a smaller $s$, so the first case is easier to satisfy. The analysis below also shows that $V_k(s)$ increases as $s$ decreases. However, since $L_k(s)$ is either $s$ or $\min\{s,V_k(s)\}$, these favourable changes do not imply monotonicity of $L_k(s)$ with respect to $s$.

Finally, we show that $V_k(s)$ increases when either $s$ or $B_k(s)$ decreases. Recall from Section~\ref{sec: proof for main theorem} that $V_k(s)$ is constructed using Corollary~\ref{lemma: lambert W corr}. Specifically, define the function
\begin{align} \label{eq: definition fl}
    f_l(p,q,r)=\left[\frac{1}{p} - \frac{1}{r} \, W_{-1}\left(-\frac{r}{pq} \exp\left(\frac{r}{p}\right)\right)\right]^{-1}.
\end{align}
Then $V(s)=f_l(p,q,r)$ with $p=B_k(s)$, $q=c(s)$, $r=d_Y/2$. Smaller $B_k(s)$ implies smaller $p$, and smaller $s$ implies smaller $c(s)$ and hence smaller $q$. Since $f_l(p, q, r)$ increases as $p$ or $q$ decreases, as formalised in Lemma~\ref{lemma: function behaviour} presented below, this establishes the desired monotonicity: $V_k(s)$ increases as either $s$ or $B_k(s)$ decreases.

\begin{lemma} \label{lemma: function behaviour}
Let $p, q, r > 0$. The function $f_l(p,q,r)$ defined in~\eqref{eq: definition fl}, when well-defined (i.e., $ \frac{pq}{r} \exp(-\frac{r}{p}) \geq e $), increases as either $p$ or $q$ decreases, or as $r$ increases.
\end{lemma}
\begin{proof}
     Define $g_x(p,q,r)=p\big[1-\frac{1}{1+q\exp(-\frac{r}{x})}\big]$. When $f_l(p,q,r)$ is well defined, Corollary~\ref{lemma: lambert W corr} gives that the solution set of the inequality $x > g_x(p, q, r)$ is $$\textstyle x\in (-\infty,f_l(p,q,r))\bigcup(\big[\frac{1}{p}-\frac{1}{r}W_{0}(-\frac{r}{pq}\exp(\frac{r}{p}))\big]^{-1},\infty).$$
     Observe that $g_x(p, q, r)$ decreases when either $p$ or $q$ decreases, or when $r$ increases. Hence, any $x$ satisfying $x>g_x(p,q,r)$ for a given $(p,q,r)$ still satisfies the inequality after any such parameter change. Therefore, the original solution set is contained in the updated solution set. Since the two intervals above are separated by a nonempty gap except at the boundary case, the lower-branch interval $(-\infty,f_l(p,q,r))$ cannot shrink under such a parameter change. Consequently, $f_l(p,q,r)$ increases as either $p$ or $q$ decreases, or as $r$ increases.
\end{proof}

\subsubsection{Construction of $B_k(s)$ for the known-shape case} \label{sec: construction U(s)}
We consider the case where the object shape is known or the shape prior is sufficiently strong that the shape update is negligible. In this case, we assume that the expected precision of $q_1(P_k)$ remains fixed at its prior value, namely $\phi_k\Phi_k=\phi_k^\prime\Phi_k^\prime$. We now construct a tighter admissible bound $B_k(s)$ in \eqref{eq: Bk(s)} for this setting. From \eqref{eq: Delta(s)}, we have 
\begin{align} \notag
    \Delta_k(s)=\sum_{j=1}^M \!\frac{1+t(s)}{\exp(0.5(y_j\!-\!H\mu_k)^{\!\top} \!\phi_k^\prime \Phi_k^\prime (y_j\!-\!H\mu_k))+t(s)},
\end{align}
where $t(s)$ is defined in \eqref{eq: t(s)}. By Remark~\ref{birth prior remark} under Assumption~\ref{assump}, $H\mu_k=\overbar{y}_k$, where $\overbar{y}_k$ in \eqref{eq: state update} is a weighted average of the measurements. Hence $H\mu_k$ lies in the surveillance region $\mathcal{S}$ containing all measurements, given that $\mathcal{S}$ is convex, e.g. rectangular. Therefore,
\begin{align} \notag
    \Delta_k(s)\leq \max_{y\in S} \sum_{j=1}^M \!\frac{1+t(s)}{\exp(0.5(y_j\!-\!y)^{\!\top} \!\phi_k^\prime\Phi_k^\prime (y_j\!-\!y))+t(s)}.
\end{align}
This maximum can be bounded by a grid based construction. Let $G_1,\ldots,G_{N_G}$ be a partition of $\mathcal{S}$. Then a practical and relatively tight admissible bound $B_k(s)$ can be obtained as
\begin{align} \notag
 &\Delta_k(s)\!\leq \!\!\!\!\!\max_{i=1,...,N_G } \max_{y\in G_i }\!\sum_{j=1}^M \!\frac{1+t(s)}{\exp(0.5(y_j\!-\!y)^{\!\top} \!\phi_k^\prime\Phi_k^\prime (y_j\!-\!y))+t(s)}\\ \notag
 &\leq\!\!\!\max_{i=1,...,N_G } \!\sum_{j=1}^M \!\frac{1+t(s)}{\min_{y\in G_i } \exp(0.5(y_j\!-\!y)^{\!\top} \!\phi_k^\prime\Phi_k^\prime (y_j\!-\!y))+t(s)}\\ \notag
 &:=B_k(s),
\end{align}
where the last inequality becomes exact in the limiting case where the grid is arbitrarily fine, i.e. each $G_i$ consists of individual points in $\mathcal{S}$. The inner minimisation $\min_{y\in G_i } \exp(0.5(y_j-y)^{\!\top} \phi_k^\prime\Phi_k^\prime (y_j-y))$ for each $j$ and $G_i$ can be efficiently computed for diagonal $\Phi_k^\prime$ and rectangular grid cells. For non diagonal $\Phi_k^\prime$, one can instead build the grid in transformed coordinates where the quadratic form becomes diagonal. This gives the same type of computation and remains cheap for $d_Y=2$ or $3$.
%\vspace{-1em}

\section{Parameterisation of compared methods in Section~\ref{sec: comparison simulation}} \label{apx: parameterasations for exp}

Here we list the parameterisations of PMBM, SPA, and PiVoT used in the simulation experiments in Section~\ref{sec: comparison simulation}. For all methods, the kinematic transition model uses the ground truth parameterisation in \eqref{eq:cv model}, and the positional birth prior covers the entire surveillance region.

For PMBM, the parameterisation follows \cite{granstrom2019poisson}. PMBM reports objects with existence probability above $0.5$ under the highest weight global hypothesis. The gate probability is set to $0.999$, the maximum number of global hypotheses is set to $100$, and Murty's algorithm is capped at $20$ assignments. The Poisson birth rate is set to $0.5$, and survival probability is set to $0.99$. The clutter rate and detection probability are given as the ground truth for all DS1--DS6. The rate prior follows a Gamma distribution with shape parameter $2$ and mean set to $8$ for DS3--DS6, giving a prior that accommodates both measurement rates well. For DS1--DS2, its mean is instead set to the ground truth. The shape prior follows an inverse Wishart distribution with degrees of freedom $10$. Its mean is set to the ground truth $800I_2$ for DS1--DS2, and to $400I_2$ for DS3--DS6 to better accommodate both object types. The forgetting factors for shape and rate prediction are set to $\tau=10$ and $\eta=1.2$, respectively, see Table III of \cite{granstrom2019poisson}. These factors are defined differently from those in PiVoT, but the chosen values inflate the uncertainty to a comparable level. Finally, PMBM$\mathrm{C}$ and PMBM$\mathrm{F}$ use $16$ and $100$ DBSCAN distance values, respectively, evenly spaced over $[1,100]$, a range found to give strong overall performance. 

For SPA, the object declaration threshold is set to $0.5$, the pruning threshold to $10^{-3}$, and the number of iterations to $3$. The mean number of births is set to $0.5$, the survival probability to $0.99$, and the clutter rate to the ground truth. Since SPA does not handle rate estimation or detection probability, we first take the measurement rate as the ground truth rate for DS1--DS2 and the larger rate $8$ for DS3--DS6, then multiply it by the ground truth detection probability (not equal to $1$ only for DS2 and DS3) to account for missed detections. The shape prior is set as an inverse Wishart distribution with degrees of freedom $10$. Its mean is set to the ground truth $800I_2$ for DS1--DS2, and to $400I_2$ for DS3--DS6 to better accommodate both object types. The degrees of freedom for shape prediction is set to $100$. SPA is implemented with different particle sizes, indicated by the subscripts in Table~\ref{tb:results}.

For PiVoT, identical uninformative birth priors are used as suggested in Section~\ref{sec: initialisation, clustering demo}, with birth probability $p_{n,k}^b=0.2$ and ineffective birth removal threshold $L=0.5$. The thresholds $P_{\mathrm{pru}}$, $P_{\mathrm{rep}}$, and $P_{\mathrm{stp}}$ in Section~\ref{sec: estimation extraction} are set to $0.03$, $0.8$, and $0.3$, respectively. To assess PiVoT's robustness across different scenes, we assume an unknown clutter rate and use the same uninformative shape and birth priors for all DS1--DS6. The shape (precision matrix) prior is set as a Wishart distribution with mean $I_2/400$ and degree of freedom $3$, and the rate prior is set as a Gamma distribution with shape parameter $2$ and mean $8$. The unknown clutter rate is set to the current number of measurements at each time step, which is a crude upper bound but works well under the current parameterisation. The transition parameters $p_{n,k}^s$, $\gamma_{\Lambda,n}$, and $\gamma_{P,n}$ in \eqref{eq: prediction} are all set to $0.9$. The initial $q_1$ is shown in Fig.~\ref{fig:clustering}(a) in Section~\ref{sec: initialisation, clustering demo}.

\end{document}